\documentclass{article}

\PassOptionsToPackage{square, numbers, sort}{natbib}

\usepackage[preprint]{neurips_2022}

\usepackage[utf8]{inputenc} 
\usepackage[T1]{fontenc}    

\usepackage{amsmath}
\usepackage{amsxtra,amsfonts,amscd,amssymb,bm,bbm}
\usepackage{nicefrac}       
\usepackage{amsthm,thmtools,thm-restate}
\usepackage{mathtools}
\newcommand\labelAndRemember[2]
  {\expandafter\gdef\csname labeled:#1\endcsname{#2}%
   \label{#1}#2}
\newcommand\recallLabel[1]
   {\csname labeled:#1\endcsname\tag{\ref{#1}}}
\newcommand\labelr[2]
  {\expandafter\gdef\csname labeled:#1\endcsname{#2}%
   \label{#1}#2}
\newcommand\recall[1]
   {\csname labeled:#1\endcsname}

\usepackage[linesnumbered,vlined,ruled]{algorithm2e}
\SetKwInOut{Input}{input}
\SetKwInOut{Output}{output}

\usepackage{hyperref,url}       
\usepackage[capitalise]{cleveref}
\usepackage[titletoc,title]{appendix}
\usepackage{cancel}
\usepackage{enumerate}
\usepackage{enumitem}
\usepackage{comment}
\usepackage{times}

\usepackage{booktabs}       
\usepackage{multirow}
\usepackage{multicol}
\usepackage{makecell}
\usepackage{array}
\newcolumntype{H}{>{\setbox0=\hbox\bgroup}c<{\egroup}@{}}
\newcolumntype{Z}{>{\setbox0=\hbox\bgroup}c<{\egroup}@{\hspace*{-\tabcolsep}}}

\usepackage[normalem]{ulem}
\usepackage{exscale}
\usepackage{tikz}
\usepackage{float}
\usepackage{graphicx,graphics,subfig}
\graphicspath{ {./fig/} }
\allowdisplaybreaks

\usepackage[final]{showlabels}

\makeatletter
\def\SL@eqntext#1{\rlap{\fbox{\vbox{{\showlabelsetlabel{\SL@prlabelname{#1}}}}}}}
\makeatother

\makeatletter
\def\thm@space@setup{\thm@preskip=3pt
\thm@postskip=0pt}
\makeatother
\newtheorem{theorem}{Theorem}[section]
\newtheorem{lemma}[theorem]{Lemma}
\newtheorem{definition}[theorem]{Definition}
\newtheorem{corollary}[theorem]{Corollary}
\newtheorem{proposition}[theorem]{Proposition}
\newtheorem{remark}[theorem]{Remark}

\newtheorem*{problem}{Problem}
\newtheorem{assumption}[theorem]{Assumption}

\usepackage{pifont}

\newcommand{\id}{I}

\newcommand{\nS}{|\mathcal{S}|}
\newcommand{\nA}{|\mathcal{A}|}
\newcommand{\RR}{\mathbb{R}}

\newcommand{\Eps}{\mathcal{E}}

\newcommand{\II}{\mathbb{I}}
\newcommand{\EE}{\mathbb{E}}
\newcommand{\PP}{\mathbb{P}}
\newcommand{\VV}{\mathbb{V}}
\newcounter{cnt}
\foreach \num in {1,2,...,26}{%
  \stepcounter{cnt}%
  \expandafter\xdef \csname c\Alph{cnt}\endcsname {\noexpand\mathcal{\Alph{cnt}}}%
  \expandafter\xdef \csname b\Alph{cnt}\endcsname {\noexpand\mathbb{\Alph{cnt}}}%
}

\DeclareMathOperator*{\argmin}{arg\,min}
\DeclareMathOperator*{\argmax}{arg\,max}

\newcommand{\Proj}{\operatorname{Proj}}
\newcommand{\diag}{\operatorname{diag}}
\newcommand{\TV}{\operatorname{TV}}

\newcommand{\violation}{\mathrm{violation}}
\newcommand{\mix}{\mathrm{mix}}

\newcommand{\verify}{\text{V\small{ERIFY}}}
\newcommand{\true}{\text{T\small{RUE}}}
\newcommand{\false}{\text{F\small{ALSE}}}

\newcommand{\leqsim}{\lesssim}
\newcommand{\geqsim}{\gtrsim}

\newcommand{\st}{\mathop{\textrm{s.t.}\ }}  
\newcommand{\T}{\mathrm{T}}  
\newcommand{\iprod}[2]{\left\langle #1, #2 \right\rangle}
\newcommand{\nrm}[1]{\left\|#1\right\|}
\newcommand{\minop}[1]{\min\left\{#1\right\}}
\newcommand{\sqr}[1]{\left\|#1\right\|^2}
\newcommand{\prob}[1]{\mathbb{P}\left(#1\right)}
\newcommand{\epct}[1]{\mathbb{E}\left[#1\right]}
\newcommand{\cond}[2]{\mathbb{E}\left[\left.#1\right|#2\right]}
\newcommand{\condP}[2]{\mathbb{P}\left(\left.#1\right|#2\right)}

\newcommand{\bigO}[1]{\mathcal{O}\left(#1\right)}
\newcommand{\tbO}[1]{\tilde{\mathcal{O}}\left(#1\right)}
\newcommand{\tbbO}[1]{\tilde{\mathcal{O}}\big(#1\big)}
\newcommand{\tbOm}[1]{\tilde{\Omega}\left(#1\right)}
\newcommand{\Om}[1]{\Omega\left(#1\right)}
\newcommand{\ThO}[1]{\Theta\left(#1\right)}
\newcommand{\ceil}[1]{\left\lceil #1\right\rceil}
\newcommand{\floor}[1]{\left\lfloor #1\right\rfloor}
\DeclarePairedDelimiterX{\ddiv}[2]{(}{)}{%
  #1\;\delimsize\|\;#2%
}
\newcommand{\KL}{\operatorname{KL}\ddiv}
\newcommand{\pos}[1]{\left[#1\right]_{+}}
\renewcommand{\neg}[1]{\left[#1\right]_{-}}

\newcommand{\hmu}{\hat{\mu}}
\newcommand{\mub}{\mu_{\pi_b}}
\newcommand{\pib}{\pi_b}

\newcommand{\abs}[1]{\left|#1\right|}
\newcommand{\onu}{\overline{\nu}}

\newcommand{\highprobs}[1]{with probability at least $1-\nicefrac{\delta}{#1}$}

\newcommand{\highprobdemo}{with probability at least $1-\delta$}
\newcommand{\opi}{\overline{\pi}}
\newcommand{\hpi}{\hat{\pi}}
\newcommand{\tnu}{\tilde{\nu}}
\newcommand{\hnu}{\hat{\nu}}
\newcommand{\ox}{\overline{x}}
\newcommand{\tpi}{\tilde{\pi}}
\newcommand{\hdelta}{\widehat{\Delta}}
\newcommand{\oV}{\overline{V}}
\newcommand{\ol}{\overline{\lambda}}
\newcommand{\odelta}{\overline{\Delta}}
\newcommand{\pik}{\pi^{(K)}}
\newcommand{\piks}{\pi^{(K-1)}}
\newcommand{\xk}{x^{(K)}}
\newcommand{\xks}{x^{(K-1)}}

\newcommand{\tu}{\tilde{u}}
\newcommand{\ttheta}{\tilde{\theta}}
\newcommand{\tJ}{\tilde{J}}
\newcommand{\hJ}{\widehat{J}}

\newcommand{\hy}{\hat{y}}
\newcommand{\hg}{\widehat{g}}
\newcommand{\oy}{\overline{y}}
\newcommand{\Gr}{\mathcal{G}}

\newcommand{\gLv}{\nabla_V\cL_w}
\newcommand{\gLl}{\nabla_\lambda \cL_w}
\newcommand{\gLx}{\nabla_x\cL_w}
\newcommand{\hcL}{\widehat{\mathcal{L}}}

\newcommand{\clog}{\iota}
\newcommand{\dlog}{\log(1/\delta)}

\newcommand{\htmx}{\hat{t}_{\mix}}

\newcommand{\bu}{\mathbf{u}}

\newcommand{\hpia}{\hat{\pi}_a}

\newcommand{\nval}{\min\left\{|\mathcal{S}||\mathcal{A}|,|\mathcal{S}|+I\right\}}

\newcommand{\algB}{DPDL}
\newcommand{\algBname}{\underline{D}eviation-controlled \underline{P}rimal-\underline{D}ual \underline{L}earning algorithm}
\newcommand{\adaptiveAlg}{Adaptive-\algB}

\newcommand{\Gap}{\operatorname{Gap}}
\newcommand{\epsapp}{\epsilon_{\operatorname{approx}}}

\newcommand{\epsver}{\epsilon_{\operatorname{ver}}}

\newcommand{\hJuk}{\widehat{J}^{u^{\kappa}}}

\newcommand{\dgap}{\Delta}

\newcommand{\hphi}{\psi}
\newcommand{\outx}{\overline{x}}

\newcommand{\vms}{\mathfrak{V}} 
\newcommand{\safe}{\mathfrak{S}}
\newcommand{\dev}{D}
\newcommand{\devset}{D}

\renewcommand{\sp}{s_{\oplus }}
\newcommand{\sm}{s_{\ominus }}
\newcommand{\condPt}[2]{\mathbb{P}_{\theta}\left(\left.#1\right|#2\right)}
\newcommand{\condPb}[2]{\mathbb{P}_{\pi_b}\left(\left.#1\right|#2\right)}

\newcommand{\RV}{R_{\mathcal{V}}}
\newcommand{\RLam}{R_{\Lambda}}

\renewcommand{\paragraph}[1]{\textbf{#1}\,\,}

\title{A Near-Optimal Primal-Dual Method for Off-Policy Learning in CMDP}

\author{%
  Fan Chen\\
  School of Mathematics\\
  Peking University\\
  \texttt{chern@pku.edu.cn}
  \And
  Junyu Zhang\\
  Department of Industrial Systems Engineering and Management\\
  National University of Singapore\\
  \texttt{junyuz@nus.edu.sg}
  \And
  Zaiwen Wen\\
  Beijing International Center for Mathematical Research\\
  Peking University\\
  \texttt{wenzw@pku.edu.cn}
}

\begin{document}

\maketitle

\begin{abstract}
  As an important framework for safe Reinforcement Learning, the Constrained Markov Decision Process (CMDP) has been extensively studied in the recent literature. However, despite the rich results under various on-policy learning settings, there still lacks some essential understanding of the offline CMDP problems, in terms of both the algorithm design and the information theoretic sample complexity lower bound. In this paper, we focus on solving the CMDP problems where only offline data are available. By adopting the concept of the single-policy concentrability coefficient $C^*$, we establish an $\Omega\left(\frac{\min\left\{|\mathcal{S}||\mathcal{A}|,|\mathcal{S}|+I\right\} C^*}{(1-\gamma)^3\epsilon^2}\right)$ sample complexity lower bound for the offline CMDP problem, where $I$ stands for the number of constraints. By introducing a simple but novel deviation control mechanism, we propose a near-optimal primal-dual learning algorithm called DPDL. This algorithm provably guarantees zero constraint violation and its sample complexity matches the above lower bound except for an $\tilde{\mathcal{O}}((1-\gamma)^{-1})$ factor. Comprehensive discussion on how to deal with the unknown constant $C^*$ and the potential asynchronous structure on the offline dataset are also included. 
\end{abstract}

\section{Introduction}\label{sect:intro}
Reinforcement Learning (RL) is an important tool for modeling the real world tasks that involve sequential decision making. Such RL problems are often mathematically described as a Markov Decision Process (MDP) that maximizes a cumulative sum of rewards. The safe reinforcement learning, on the other hand, not only cares the reward maximization, but also attempts to ensure a reasonable system performance with respect to certain safety constraints. Such safety constrained RL problems are often formulated as  the Constrained Markov Decision Process (CMDP) $\cM=(\cS, \cA, \PP, r, u, \gamma, \rho_0)$, where $\cS$ is a finite state space, $\cA$ is a finite action space, $\gamma \in(0,1)$ is the discount factor, $\PP\left(s^{\prime} \mid s, a\right)$ stands for the transition probability from $s$ to $s^{\prime}$ under the action $a$ for $\forall(s,a,s^{\prime})\in\cS\times\cA\times\cS$, and $r:\cS \times \cA \rightarrow[-1,1]$ is the reward function, $(u_i: \cS \times \cA \rightarrow[-1,1])_{i\in[I]}$ is a set of $I$ utility functions,  $\rho_0$ is the initial state distribution over $\cS$. The goal of CMDP is to find an optimal policy $\pi$ to maximize the cumulative reward while satisfying a group of constraints:
\begin{eqnarray}
	\label{eqn:srl-obj-demo}
	 &\max_{\pi}& 
	  J(\pi)\, :=\mathbb{E}\bigg[\sum_{t=0}^{+\infty} \gamma^{t}\cdot r\left(s_{t}, a_{t}\right)\,\Big|\,s_0\sim\rho_0,\pi\bigg]\\
	 &\st & \!\!J^u_i(\pi)
	 :=\mathbb{E} \bigg[\sum_{t=0}^{+\infty} \gamma^{t}\cdot u_i\left(s_{t}, a_{t}\right)\bigg]\geq 0, \mbox{ for } i \in [I]=\left\{ 1,2,...,I\right\}.\nonumber
\end{eqnarray} 

For the CMDP problem, there has been plenty of on-policy algorithms, see \cite[etc.]{OPDOP, EECRL, UCBPD}. However, in real world applications such as training physical robots, where safety is an important measure of performance, the real time on-policy interaction with the environment may suffer from the potential damages to the robots. 
Besides, in many non-simulating environments, the on-policy data collection may also be time-consuming. Therefore, it is crucial to design an off-policy algorithm to solve the CDMP problems, where plenty of historical data are already accumulated while real time interactions are limited. To our best knowledge, offline CMDP algorithms are rare \cite{le2019batch, xu2021constraints, OPDVI},  and the sample complexity guarantees are limited. 
In particular, a strong uniform concentrability assumption is required in \cite{le2019batch}, and the model-based method \cite{OPDVI} mainly considers the case an empirical model is known. 
Thus it is still not clear how to efficiently solve offline CMDPs with model-free approaches, and there lacks essential understanding of the information theoretic lower bound on the sample complexity of the offline CMDP. 

In this paper, we propose a \underline{D}eviation-controlled \underline{P}rimal-\underline{D}ual \underline{L}earning (DPDL) method to solve problem \eqref{eqn:srl-obj-demo}. We adopt the primal-dual strategy developed in  \cite[etc.]{RandomizedLP2017,CautiousDualRL2021,bai2021achieving,SCAL} as the main algorithmic framework while several non-trivial contributions have been made beyond the existing results. Unlike the aforementioned literatures that exclusively rely on the accessibility of a generative model, \algB\ utilizes the offline data, where the distribution shift difficulties of the offline data is tackled by a novel and effective adaptive deviation control mechanism. If the considered CMDP instance has a finite (but potentially unknown) \textit{concentrability coefficient}, \algB\ provably finds a policy with $\mathcal{O}(\epsilon)$-optimal reward and zero constraint violation. An information theoretical lower bound on the sample complexity of offline CMDP is also derived in this paper, which indicates 
that our deviation control mechanism achieves a minimax optimal complexity dependence on $I,\nS,\nA,C^*$. 

\paragraph{Main Contribution.} We summarize the contributions in details as follows. 
\begin{itemize}[noitemsep,topsep=0pt,parsep=0pt,partopsep=0pt,leftmargin=*]
  \item  We propose the \algB\ algorithm to solve the CMDP problem \eqref{eqn:srl-obj-demo}. Suppose the CMDP instance satisfies the Slater's condition and certain prior knowledge on the concentrability coefficient $C^*$ is given, \algB\ provably finds an $\epsilon$-optimal policy with zero constraint violation using $\tbO{\frac{\nval C^*}{(1-\gamma)^4\epsilon^2}}$ offline samples.
  
  \item We establish an information theoretic sample complexity lower bound of $\Om{\frac{\nval C^*}{(1-\gamma)^3\epsilon^2}}$ for the offline CMDPs, indicating that DPDL is near optimal up to an $\tilde{\mathcal{O}}((1-\gamma)^{-1})$ factor. The necessity of the Slater's condition for achieving zero constraint violation is also established.
  
  \item In order to handle the practical situation where $C^*$ is unknown, an adaptive version of DPDL is designed with the same sample complexity as DPDL.   
  
  \item Our analysis of \algB\ also extends to the asynchronous case, where the offline dataset consists of a sample trajectory generated by certain behavior policy. In this situation, the sample complexity of DPDL is shown to be $\tbO{\frac{t_{\mix}^2\nval C^*}{(1-\gamma)^4\epsilon^2}}$.
\end{itemize} 

\paragraph{Related Work.}  
Recently, considerable efforts have been devoted to the online learning of CMDP.  Under the episodic and tabular setting, several works \cite{OPDOP, EECRL, UCBPD} have achieved the $\tbbO{\sqrt{\nS^2\nA T}}$ regret and cumulative constraint violation, with different dependence on the episode length $H$ omitted. Under proper assumptions, zero or bounded cumulative constraint violation can be achieved \cite{concave-util-zero, OptPess-PD}. In terms of the number of constraints $I$, MOMA proposed in \cite{MOMA} achieves an $\tbbO{\!\sqrt{\min\{\nS,\!I\}I\nS\nA/T}}$ convergence on both average reward gap and constraint violation. Nevertheless, all the above results adopt the model-based approaches. Except for \cite{MOMA}, they either consider the cases where $I\!=\!1$ or completely ignore the influence of $I$ in the sample complexity. Therefore, both deriving an efficient model-free method and obtaining the optimal dependence on $I$ remain open.

Another approach closely related to our paper is the primal-dual method in RL, see  \cite[etc.]{RandomizedLP2017,PiLearning2017,PD-LfD,CautiousDualRL2021,bai2021achieving}. Given the access to a generative model, the model-free primal-dual method developed in \cite{bai2021achieving} achieves an $\tilde{\mathcal{O}}\big(\frac{I|\cS||\cA|}{(1-\gamma)^4\epsilon^2}\big)$ sample complexity to find an $\epsilon$-optimal safe policy.
The deviation control mechanism we develop enables the primal-dual approach to extend beyond the generative model.

Finally, we mention a few related works in the offline RL and safe RL. Previous offline RL algorithms with sample efficiency guarantees typically assume the \textit{uniform concentrability} \cite[etc.]{munos2008finite, le2019batch} or lower bounded \textit{minimum visitation $\mu_{\min}$} \cite[etc.]{yin2020near, yin2021near}. Recently, under the less restrictive assumption of the \emph{single-policy concentrability coefficient} $C^*$, a minimax optimal sample complexity lower bound of $\Omega\big(\frac{\nS C^*}{(1-\gamma)^3\epsilon^2}\big)$ for discounted offline MDPs is derived in \cite{OfflineRL_Lower_Bound}. A similar $\Omega\big(\frac{H^3\nS C^*}{\epsilon^2}\big)$ lower bound is also derived for the episodic setting in \cite{policy-finetuning}.
Under both settings, offline algorithms with $\tilde{\mathcal{O}}(\nS C^*\epsilon^{-2})$ sample complexity (with  different $(1-\gamma)^{-1}$ or $H$ factors omitted) have been discovered with either model-based \cite{OfflineRL_Lower_Bound,policy-finetuning,yin2021towards,li2022settling} or model-free approaches \cite{shi2022pessimistic, yan2022efficacy}.  In terms of the offline CMDP problem, the only existing results are \cite{le2019batch, xu2021constraints, OPDVI}, where \cite{xu2021constraints} only provides asymptotic convergence, \cite{le2019batch} relies on a much stronger uniform concentrability assumption, and \cite{OPDVI} is a model based method that potentially suffers an $\mathcal{O}((C^*)^2)$ dependence. 
Compared to these works, our method is model-free and has an optimal $\mathcal{O}(C^*)$ dependence on the concentrability coefficient.

\section{Problem setup}\label{sect:preliminary}
\subsection{LP formulation of CMDP problem}
For any policy $\pi$, the (unnormalized) state-action occupancy measure is defined as 
\begin{equation}
	\label{eqn:pi2nu}
	\nu^{\pi}(s, a) :=\sum_{t=0}^{+\infty} \gamma^{t}\cdot \PP\left(s_{t}=s, a_{t}=a \mid s_0\sim\rho_0,\pi\right), \,\,\mbox{ for }\,\,\forall (s,a)\in\cS\times\cA.
\end{equation} 
Given any occupancy measure $\nu^{\pi}$, the policy $\pi$ that generates $\nu^\pi$ can be recovered as
\begin{eqnarray}
	\label{eqn:nu2pi}
	\pi(a|s) = \frac{\nu^\pi(s,a)}{\sum_{a'}\nu^\pi(s,a')}, \quad\forall(s,a)\in\cS\times\cA.
\end{eqnarray}  
According to \cite{altman1995CMDP}, it is well known that the set of all  state-action occupancy measures form a polyhedron $\big\{\nu\!\in\!\mathbb{R}^{\nS\times\nA}_{\geq0}\!:\! \sum_{a\in\cA} (\id-\gamma \bP_a)\nu_a\!=\!\rho_0\big\}$, where  $\nu_{a}\!:=\!\left(\nu(s,a)\right)_{s\in\cS}$ is an $\nS$-dimensional column vector, and $\bP_a\!:=\!\left(\PP(s'|s,a)\right)_{s',s}$ is an $\nS\!\times\!\nS$ transition matrix, see also \cite{RandomizedLP2017}. Therefore, combined with the fact that 
$J(\pi) = \langle\nu^\pi,r\rangle$, and $J^u_i(\pi) = \langle\nu^{\pi},u_i\rangle$, the  CMDP problem \eqref{eqn:srl-obj-demo} can be reformulated as an LP problem with $\nS\!+\!I$ constraints:
\begin{eqnarray}
	\label{prob:CMDP-LP}
	\max_{\nu\in\RR_{\geq0}^{\nS\times\nA}}
	\quad \langle\nu,r\rangle\quad \st \quad \sum_{a\in\cA} (\id-\gamma \bP_a)\nu_a=\rho_0,\quad \langle\nu,u_i\rangle\geq0,\, \forall i\in[I].
\end{eqnarray}
Due to the fundamental theorem of LP, see e.g. \cite{bertsimas1997introduction}, problem \eqref{prob:CMDP-LP} has an optimal basic feasible solution with at most $\nS+I$ positive entries, which indicates the following proposition. 
\begin{proposition}
	\label{proposition:sparsity}
	For the CMDP problem \eqref{eqn:srl-obj-demo} with $I$ constraints, there is an optimal policy $\pi^*$ such that $|\mathrm{supp}(\nu^{\pi^*})| \leq \cN\!:=\!\min\{\nS\!+\!I,\nS\nA\}$, where $\mathrm{supp}(\cdot)$ denotes the support of a vector. 
\end{proposition}
This result captures the potential sparse structure of the optimal policy when $I$ is not as large as $\nS\nA$, and is the key to deriving a tight complexity dependence on the number of constraints $I$.

\subsection{Off-policy learning from demonstration}
In this work, we consider the offline CMDP problems where the agent cannot interact with the environment. Instead, the optimization is conducted using a fixed offline dataset. To standardize the discussion, we make the following assumption on the offline dataset, see e.g. \cite{OfflineRL_Lower_Bound}.
\begin{assumption}[Independent batch dataset] 
\label{assumption:SyncData}
  The batch dataset $\cD$ consists of independent tuples $(s,a,s',r,\bu)$, such that $(s,a)\sim\mu$, $\cond{r}{s,a}=r(s,a), \cond{\bu_i}{s,a}=u_i(s,a)$, and $s'\sim\PP(\cdot|s,a)$, where $\mu$ is called the reference distribution. 
\end{assumption}

To characterize the distribution shift of an arbitrary occupancy measure $\nu^\pi$ from the reference distribution $\mu$, we introduce the following notion of the deviation:
$\dev^{\pi}:=\max_{s,a}\frac{(1-\gamma)\nu^\pi(s,a)}{\mu(s,a)}$, where the $(1\!-\!\gamma)$-factor normalizes $\nu^\pi$ to be a distribution. In offline RL, it is natural to assume that the deviation $\dev^{\pi^*}$ of the optimal policy is finite. That is, the reference distribution $\mu$ fully covers $\mathrm{supp}(\pi^*)$. Otherwise, no optimality can be guaranteed. Combining the sparse nature of the optimal solution of \eqref{eqn:srl-obj-demo}, we introduce the following finite concentrability assumption for our problem. 

\begin{assumption}
	\label{assumption:concentrability}
	For $\forall\hphi\geq 1$, denote the \textit{$\hphi$-deviated policy class} as $\Pi(\hphi)\!:=\!\left\{\pi\!:\!\nu^{\pi}\!\in\! D(\hphi)\right\}$ where   
	\begin{equation}
		\label{defn:Dphi-set}
		D(\hphi)\!:=\!\bigg\{\nu\in\RR_{\geq 0}^{\nS\nA}\!: \max_{s,a}\frac{(1-\gamma)\nu(s,a)}{\mu(s,a)}\leq \hphi,\, \sum_{s,a}\frac{(1-\gamma)\nu(s,a)}{\mu(s,a)}\leq \cN\hphi\bigg\}.
	\end{equation} 
	We assume there exists a finite $\hphi$ such that some optimal policy $\pi^*$ is contained in $\Pi(\hphi)$. Let $C^*$ be the minimum of such $\hphi$. We call this constant $C^*$ the (single-policy) concentrability coefficient.
\end{assumption} 
The above assumption includes a sparsity induced constraint as a result of Proposition \ref{proposition:sparsity}, its counterpart in the definition of single-policy concentrability of offline MDP \cite{OfflineRL_Lower_Bound} is the deterministic optimal policy. The explicit dependence on $\cN$ in $D(\hphi)$ facilitates the derivation of the information theoretic lower bound as well as a near-optimal algorithm.

A second remark is that if we know any upper bound $\hphi$ of the coefficient $C^*$, then it will be sufficient to only consider the policies in $\Pi(\hphi)$. When $C^*$ is unknown, $\hphi$ control the risk of distribution shift.  Consequently, in this paper, we propose to solve the LP formulation \eqref{prob:CMDP-LP} with a tighter feasible region introduced by $D(\hphi)$. This will allow us to properly control the variance of the off-policy sampling when some of $\mu(s,a)$ is extremely small or even zero. We call this strategy \emph{deviation control}.

\subsection{Conservatism toward constraints} \label{sec:conservative}
We say policy $\pi$ is safe if it satisfies all constraints in \eqref{eqn:srl-obj-demo}, and we say $\pi$ is $\epsilon$-safe if $J^u_i(\pi)\geq -\epsilon$, for $\forall i\in[I]$. Most of the existing online CMDP algorithms guarantee $\mathcal{O}\big(1/\sqrt{T}\big)$ average safeness. To ensure the true safeness (zero constraint violation) in this work, we assume the Slater's condition to hold throughout this paper. In fact, in \cref{sec:LowerBound}, we will show that the Slater's condition is the necessary condition for any offline CMDP algorithm to obtain zero constraint violation.
\begin{assumption}\label{assumption:slater}
	There exists $\varphi>0$ and a policy $\pi$ such that $J^u_i(\pi)\geq \frac{\varphi}{1-\gamma}, \,\forall i\in[I]$. 
\end{assumption}
A prior knowledge of such a constant $\varphi$ is assumed throughout our discussion, and we also assume the Slater's condition holds for $\Pi':=\Pi(C^*)$. Given Assumption \ref{assumption:slater}, we leverage the idea of conservative constraints proposed in \cite{bai2021achieving}. Namely, instead of $J^u_i(\pi)\geq 0$, we consider the conservative constraints $J^u_i(\pi)\geq \kappa$ when solving the CMDP problem, where $\kappa>0$ is a properly chosen parameter that controls the level of conservatism in the constraints.
In order to keep the form of the constraints in problem \eqref{eqn:srl-obj-demo}, we adopt a shifted utility function $u_i^{\kappa}$ defined by $u_i^{\kappa}(s,a):=u^i(s,a)-(1-\gamma)\kappa$ for $\forall (s,a)\in\cS\times\cA$, $\forall i\in[I]$.
Therefore,  $J^u_i(\pi)\geq \kappa$ is then equivalent to $J^{u^\kappa}_i(\pi)\geq 0$. It can be shown that a properly selected $\kappa$ will facilitate a high probability of preserving zero constraint violation, while only introducing  an extra $\mathcal{O}\big(\frac{\kappa}{\varphi}\big)$ sub-optimality gap in the reward. 
  

\section{The Deviation-controlled Primal Dual Learning (DPDL) algorithm}\label{sect:algorithm}


To solve CMDP with offline samples, we  transform its  LP formulation \eqref{prob:CMDP-LP} to a saddle point form
\begin{equation}\label{prob:minimax}
    \max_{\nu\in \devset(\hphi)} \min_{\lambda\geq 0, V} \cL(V,\lambda,\nu)
    := \iprod{r}{\nu}+\bigg\langle V, \rho_0-\sum_{a} (\id-\gamma \bP_a)\nu_a\bigg\rangle+\sum_{i} \iprod{\lambda}{U_{\kappa}\nu},
\end{equation}
where $D(\hphi)$ is defined by \eqref{defn:Dphi-set}, $V\!\in\!\RR^{\nS},\lambda\!\in\!\RR^{I}$ are  Lagrangian multipliers, and the matrix $U_{\kappa}$ is defined as $U_{\kappa}:=\left[u_1^{\kappa},\cdots,u_I^{\kappa}\right]^\top\in\RR^{I\times\nS\nA}$ with $u^\kappa_i$ being the shifted utility defined in \cref{sec:conservative}. 
Given the reference distribution $\mu$, the objective function can be rewritten as an expectation:
\begin{equation*}\label{eqn:lagrangian-epct}
  \cL(V,\lambda,\nu)=
  \mathop{\mathbb{E}}_{s_0\sim\rho_0}\left[V(s_0)\right] \,+\!
  \mathop{\mathbb{E}}_{
    \substack{
      (s,a)\sim\mu \\
      \substack{s'\sim\mathbb{P}(\cdot|s,a)}
    }
  }\!\left[\frac{\nu(s,a)}{\mu(s,a)}\left(
    r(s,a) -\left(V(s)-\gamma V(s')\right)
    \!+\!\sum_i \lambda_iu^{\kappa}_i(s,a)
  \right)\!\right]\!.
\end{equation*}
If the reference distribution $\mu$ is known, we can directly sample a stochastic gradient of $\cL$. However, when the reference distribution $\mu$ is unknown in practice, then the importance sampling weight $\frac{\nu(s,a)}{\mu(s,a)}$ is also unknown. To tackle this issue, let $\hmu$ be a proper estimation of the reference distribution $\mu$, we introduce the weights $w(s,a)\!=\!\frac{\mu(s,a)}{\hmu(s,a)}$, and the diagonal matrix $W=\diag\left(w(s,a)\right)$. Then we apply a change of variables $x \!=\! W^{-1}\nu$, in other words, we set $\frac{x(s,a)}{\hmu(s,a)} \!=\!\frac{\nu(s,a)}{\mu(s,a)}$ for $\forall s,a$ to enable sampling. From now on, we will focus on the following reweighted problem 
\begin{equation}\label{prob:weighted-minimax}
	\begin{aligned}
		\min_{\lambda\in\Lambda, V\in\cV}\,\max_{x\in \cX} \,\,\cL_{w}(V,\lambda,x):=\cL(V,\lambda,Wx),
	\end{aligned}
\end{equation}
where the feasible regions are defined as 
\begin{equation}
\begin{split}
  \labelr{eqn:def-feasible}{
    \cX :=\bigg\{x\in\RR_{\geq 0}^{\nS\nA}: \max_{s,a}\frac{x(s,a)}{\hmu(s,a)}\leq \frac{\hphi}{1-\gamma}, \sum_{s,a}\frac{x(s,a)}{\hmu(s,a)}\leq \frac{\cN\hphi}{1-\gamma}, \sum_{s,a} x(s,a)\leq \frac{4}{1-\gamma}\bigg\},\\
    \cV:=\left\{V\!\in\!\RR^{\nS}: \nrm{V}_{\infty}\leq \frac{8}{1-\gamma}(1+\frac{2}{\varphi})\right\}\qquad\mbox{and}\qquad\Lambda =:\left\{\lambda\in\RR^I_{\geq0}: \nrm{\lambda}_1\leq \frac{8}{\varphi}\right\}.
  }
\end{split}
\end{equation}
The sets $\cX$, $\cV$ and $\Lambda$ are chosen to be large enough so that they contain the optimal solution of the problem \eqref{prob:minimax}, see detailed discussion in  \cref{subsect:duality-to-regret}.  Given a  sample $\zeta=(s_0,s,a,s',r,\bu)\sim\rho_0\times\cD$, and a point $Z:=(V,\lambda,x)$, we construct the unbiased gradient estimators for $\cL_w(\cdot)$ as
\begin{equation}
	\begin{aligned}
		\labelr{eqn:def-gradient}{
			\hg_V(Z;\zeta)&:=\II_{s_0}+\frac{x(s,a)}{\hmu(s,a)}\left(\gamma\II_{s'}-\II_{s}\right),\\
			\hg_\lambda(Z;\zeta)&:=\frac{x(s,a)}{\hmu(s,a)}\bu^{\kappa},\\
			\hg_x(Z;\zeta)&:=\frac{r+\gamma V(s)-V(s')+\iprod{\bu^{\kappa}}{\lambda}}{\hat\mu(s,a)}\II_{s,a}},
	\end{aligned}
\end{equation}
where $\II_{s}$ is the $\nS$-dimensional unit vector with the $s$-th element being one, $\II_{s,a}$ is the $\nS\nA$-dimensional unit vector with the $(s,a)$-th element being one, and $\bu^{\kappa}=\bu-\kappa(1-\gamma)\mathbf{1}\in\RR^I$ is the shifted utility vector.  Based on these estimators, we propose a stochastic mirror descent ascent approach to solve problem \eqref{prob:weighted-minimax}, as stated in \cref{algo:algB}. 
\begin{algorithm}
	\caption{\algBname\ (\algB)}\label{algo:algB}
	\Input{
		Tolerance $\epsilon>0$, confidential level $\delta>0$, conservatism level $\kappa>0$, stepsize $\eta_t>0$, constants $\alpha_V,\alpha_\lambda,\alpha_x,N_e,\varsigma>0$, and initial feasible solution $Z^1=[V^1;\lambda^1;x^1]$.
	} 
    Obtain $N_e$ samples from $\cD$, let $N(s,a)$ be the times that the pair $(s,a)$ appears. Compute
    \begin{equation}\label{eqn:def-empirical-dis}
    	\hmu(s,a)=\max\bigg(\frac{N(s,a)}{N_e},\varsigma\bigg),\quad \forall (s,a)\in\cS\times\cA.
    \end{equation}
	\For{$t=1,\cdots,T-1$}{
		Sample $\zeta_t=(s^0_t,s_t,a_t,s'_{t},r_t,\bu_t)$ from $\rho_0\times\cD$\;
		Compute stochastic gradients $g_V^t:=\hg_V(Z^t;\zeta^t), g_\lambda^t:=\hg_\lambda(Z^t;\zeta^t)$, and $g_x^t:=\hg_x(Z^t;\zeta^t)$\;
		Compute the stochastic mirror descent ascent update \begin{equation}\label{eqn:def-update}
			\begin{aligned}
				V^{t+1}&=\Proj_\cV\left(V^{t}-\eta_t\alpha^{-1}_Vg_V^{t}\right),\\
				\lambda^{t+1}&=\argmin_{\lambda\in\Lambda}\left(\iprod{g_\lambda^t}{\lambda-\lambda^t}+\frac{\alpha_\lambda}{\eta_t}\KL{\lambda}{\lambda^t}\right),\\
				x^{t+1}&=\argmin_{x\in\cX}\left(-\iprod{g_x^t}{x-x^t}+\frac{\alpha_x}{\eta_t}\KL{x}{x^t}\right),\\
			\end{aligned}
		\end{equation}
	}
	Compute the average iterate $\overline{x}=\frac1T\sum_{t=1}^{T} x^t, \overline{V}=\frac1T\sum_{t=1}^{T} V^t, \overline{\lambda}=\frac1T\sum_{t=1}^{T} \lambda^t$\;
	Compute $\overline\pi(a|s)=\frac{\overline{x}(s,a)}{\sum_{a'} \overline{x}(s,a')}$, for all $(s,a)$\;
	\Output{
		Policy $\overline\pi$ and the approximate solution $\overline{x}$.
	}
\end{algorithm}

The algorithm starts from a feasible solution $Z^1$, which, for example, can be easily chosen as $V^1 \!=\! \mathbf{0}$, $\lambda^1\!=\!\frac{\mathbf{1}}{\varphi I}$, $x^{1}\!=\!\frac{\cN}{\nS\nA}\frac{\hat{\mu}}{1-\gamma}$. In each iteration, an offline sample $\zeta^t$ is used to construct the unbiased gradient estimators $g^t_V, g^t_\lambda$ and $g^t_x$. A stochastic mirror descent ascent step \eqref{eqn:def-update} is then used to update the solution $Z^t$, where $\Proj_\cV(\cdot)$ denotes the Euclidean projection to the set $\cV$, and $\mathrm{KL}(Y\|Y'):= \sum_i Y_i\log\frac{Y_i}{Y_i'} -\sum_i Y_i + \sum_i Y_i'$ denotes the generalized KL divergence. Simple closed form solutions are available to the $V^{t+1}$ and $\lambda^{t+1}$ updates. By taking the advantage of the special structure of $g_x^t$ and the fact that $x^t\in\cX$ is feasible, the $x^{t+1}$ subproblem can be reduced to the root finding of a 1-dimensional monotone function, which can be solved efficiently, see details in Appendix \ref{subsect:compute}. 

Finally, it is worth noting that $\ox$ is the approximate optimal solution to the reweighted problem. And $W\ox$ will be the approximate solution to the original problem \eqref{prob:minimax} before the change of variable.  Therefore, ideally, we should have output the policy $\overline{\pi}_w(a|s) = \frac{w(s,a)\bar x(s,a)}{\sum_{a'}w(s,a')\bar x(s,a')}$, which is inaccessible in practice without knowing the reference distribution $\mu$. In order to overcome such dilemma, we show that by properly constructing the estimated distribution $\hmu$, the $\opi$ output by Algorithm \ref{algo:algB} will be close enough to the ideal output $\overline{\pi}_w$.




\section{The sample complexity of \algB}
\subsection{Main results of \algB}
For the \algB\ algorithm, the convergence and performance guarantee of the output policy $\bar{\pi}$ are summarized as the following theorem.
\begin{theorem}\label{thm:algB-thm-demo}
	Suppose that Algorithm \ref{algo:algB} runs with $\eta_t\equiv\frac{1}{\sqrt{T}}$,  $\kappa = 5\varphi\epsilon$, $\alpha_\lambda=\frac{1}{1-\gamma}\sqrt{\frac{\hphi}{\log I}}$, $\alpha_V=\varphi\sqrt{\frac{\hphi}{\nS}}$, $\alpha_x=\frac{1}{\varphi(1-\gamma)}\sqrt{\frac{\cN\hphi}{\log\hphi}}$, and $\hphi\geq C^*$. Then for any fixed $\epsilon\in\big(0,\frac{1}{10(1-\gamma)}\big]$, and $T\geq c_o\frac{\cN\hphi\clog}{\varphi^2(1-\gamma)^4\epsilon^2}$, where $\clog=\log\left(\frac{\hphi\nS\nA I}{\delta}\right)$ and $c_o$ is a universal constant, the output policy $\overline\pi$ of \algB\ satisfies the following \highprobdemo
  \begin{equation*}
  	 J(\pi^*)-J(\overline\pi)\leq \bigO{\epsilon},\quad\mbox{and}\quad
  	 J^u_i(\overline\pi)\geq 0,\forall i\in[I].
  \end{equation*} 
  When $\hphi=\cO(C^*)$, \algB\ needs at most $\tbO{\frac{\cN C^*}{\varphi^2(1-\gamma)^4\epsilon^2}}$ samples to find a safe $\bigO{\epsilon}$-optimal policy.
\end{theorem}

\begin{remark}
	\label{remark:unknownC}
  When the prior knowledge of $C^*$ is not available, and the selected parameter $\hphi<C^*$ but the Slater's condition for $\Pi(\hphi)$ still holds, the output policy $\overline\pi$ of \algB\ satisfies that
    $$J(\overline\pi)\geq \max_{\pi\in\Pi(\hphi)\cap\mathfrak{S}} J(\pi)-\bigO{\epsilon}\qquad\mbox{and}\qquad 
    J^u_i(\overline\pi)\geq -\epsapp,\forall i\in[I],$$
  where $\mathfrak{S}$ denotes the set of safe policies, and $\epsapp(\hphi):=J(\pi^*)-\max_{\pi\in\Pi(\hphi)\cap\mathfrak{S}} J(\pi)$ in some sense measures the ``sub-optimality'' of the policy class $\Pi(\hphi)$. In case a fixed sub-optimality gap $\epsilon$ is given, such difficulty of unknown $C^*$ also appears in the guarantees provided in previous works \cite{OfflineRL_Lower_Bound, policy-finetuning, yin2021towards, li2022settling, shi2022pessimistic, yan2022efficacy}.
\end{remark} 
A simple approach to resolve the difficulty of an unknown $C^*$ is discussed later in \cref{sect:adaptive-algB}.

\subsection{The analysis of \algB}
\label{sec:algB-analysis}
We break down the analysis of \cref{thm:algB-thm-demo} into the following steps.  First of all, we provide a proper choice of $N_e$ and $\varsigma$ so that $\hmu$ is close enough to $\mu$. See proof in Appendix \ref{subsect:empirical-dis}.
\begin{proposition}\label{prop:empirical-dis}
	Denote $\epsilon_e = \frac{\epsilon}{100}$, and let $\varsigma=\frac{\varphi(1-\gamma)^2\epsilon_e}{2\cN\hphi}$, and $N_e\geq \frac{512\cN\hphi}{\varphi^2(1-\gamma)^4\epsilon_e^2}\cdot\log\left(\frac{6\nS\nA}{\delta}\right)$. Then  \highprobs{3}, the estimated reference distribution $\hmu$ defined by \eqref{eqn:def-empirical-dis}  satisfies the following properties simultaneously: 
	\textbf{(1).} $\frac{\mu(s,a)}{\hmu(s,a)}\leq 2$, and $\hmu(s,a)\geq \varsigma$, for all $s,a$;
	\textbf{(2).} For any $\pi\in\Pi(\hphi)$, $W^{-1}\nu^{\pi}\in\cX$;
	\textbf{(3).} For any $x\in\cX$, $\nrm{Wx-x}_1\leq \varphi(1-\gamma)\epsilon_e$.
\end{proposition}
All the rest of our analyses are all conditioning on the success of \cref{prop:empirical-dis}.
It is worth noting that in \cref{prop:empirical-dis}, (3) clarifies the validity of constructing the output policy $\opi$ with $\ox$ instead of $W\ox$; (2) explains why the feasible region $\cX$ is defined as \eqref{eqn:def-feasible}; and (1), combined with the carefully specified feasible domains, provides the proper upper bounds on the magnitude and variance of the unbiased gradient estimators in \eqref{eqn:def-gradient}. A very detailed discussion is provided in Appendix \ref{appdx:grad-var}. In particular, for the $\hg_x(\cdot)$ estimator, an explicit $\mathcal{O}(\cN)$ dependence has been established for both the magnitude and variance, which plays a crucial role in deriving the optimal $\mathcal{O}(\min\{\nS\nA,\nS+I\})$ dependence on $\nS$, $\nA$ and $I$. 
Let us define the following gap to measure the performance of the output $\bar x$ w.r.t. problem \eqref{prob:weighted-minimax}:
\begin{equation}\label{eqn:def-duality-gap}
	\begin{aligned}
		\Gap(\ox):=\max_{x\in\cX}\min_{V\in\cV,\lambda\in\Lambda}\cL_w(V,\lambda,x)-\min_{V\in\cV,\lambda\in\Lambda}\cL_w(V,\lambda,\ox).
	\end{aligned}
\end{equation}
Based on the properly bounded gradient estimators, a high probability bound for $\Gap(\ox)$ is established in the following theorem. Its proof is detailed in \cref{appdx:ThmGap}.
\begin{theorem}\label{thm:algB-gap}
	Suppose the constants $\eta_t$, $\alpha_V$, $\alpha_\lambda$, $\alpha_x$ and $\kappa$ are chosen the same as Theorem \ref{thm:algB-thm-demo}. Then there is a universal constant $c_o$ such that, as long as $T\geq c_o\frac{\cN\hphi\clog}{\varphi^2(1-\gamma)^4\epsilon^2}$, the output $\ox$ satisfies $\Gap(\ox)\leq\frac{\epsilon}{2}$ \highprobs{3}.
\end{theorem}

Given Theorem \ref{thm:algB-gap}, we finalize the proof of Theorem \ref{thm:algB-thm-demo} by properly transforming the bound on $\Gap(\ox)$ to the expected reward gap and the constraint violation on the original CMDP problem \eqref{eqn:srl-obj-demo}, which is discussed in details in Appendix \ref{subsect:duality-to-regret}.

\subsection{Extension to asynchronous setting}\label{subsect:async}
In some situations, an independent dataset that satisfies Assumption \ref{assumption:SyncData} may not be available. Instead, the dataset may have the following asynchronous structure.
\begin{assumption}
\label{assumption:AsyncData}
  The \textit{asynchronous} dataset $\cD_{async}$ is a single sample trajectory generated by some behavior policy $\pi_b$. Namely, what we observe is a sequence 
  $\{s_t,a_t,r_t,\bu_t\}_{t\geq1}$
  generated under $\pib$. We assume the Markov Chain $\{(s_t,a_t)\}_{t\geq1}$ is  irreducible, aperiodic and uniformly ergodic, with the stationary distribution $\mu$ and the mixing time $t_\mix<+\infty$.
\end{assumption}
The asynchronous data structure introduced here is frequently considered in RL, for example, the asynchronous Q-learning \cite{Async-Q}. However, to our best knowledge, this type of offline data has yet been considered under the assumption of a finite single-policy concentrability. 
In this situation, we set $\zeta_t=(s_t^0,s_t,a_t,s_{t+1},r_t,\bu_t)$ in the \algB\ method (Algorithm \ref{algo:algB}), where $s_t^0\sim\rho_0$ and $(s_t,a_t,s_{t+1},r_t,\bu_t)$ is the tuple in the $t$-th time step of the asynchronous dataset. The sample complexity of the \algB\ Algorithm under \cref{assumption:AsyncData} is established as follows.


\begin{theorem}\label{thm:algB-thm-markov-demo}
Under \cref{assumption:AsyncData}, we follow the choice of constants in \cref{thm:algB-thm-demo}. Then given any fixed $\epsilon\in\left(0, \frac{1}{10(1-\gamma)}\right]$, $\hphi\geq C^*$, and $T\geq c'_o\frac{t_{\mix}^2\cN\hphi\clog^3}{\varphi^2(1-\gamma)^4\epsilon^2}$, the output policy $\overline\pi$ of \algB\ satisfies the following \highprobdemo
  \begin{equation*}
  	 J(\pi^*)-J(\overline\pi)\leq \epsilon\qquad\mbox{and}\qquad
  	 J^u_i(\overline\pi)\geq 0,\forall i\in[I].
  \end{equation*} 
  Here $\clog=\log\left(T\nS\nA I/\delta\right)$ and $c_o'$ is a universal constant.
  Therefore, when $\hphi=\cO(C^*)$, \algB\ needs at most $\tbO{\frac{t_{\mix}^2\cN C^*}{\varphi^2(1-\gamma)^4\epsilon^2}}$ samples to find a safe $\epsilon$-optimal policy.
\end{theorem}
The main framework for proving Theorem \ref{thm:algB-thm-markov-demo} is similar to that in Section \ref{sec:algB-analysis}, thus we present the proof in the Appendix \ref{appdx:asyn}. However, compared to the synchronous setting, a key difficulty here is that the gradient estimators  $\hg_V(Z^t;\zeta_t)$, $\hg_\lambda(Z^t;\zeta_t)$, and $\hg_x(Z^t;\zeta_t)$ are no longer unbiased, because the samples $\{\zeta_{t}\}_{t=1}^T$ are obtained from a sample path. This brings further difficulties in the analysis because the variance of the estimators can be amplified by the correlation between samples.

The basic idea to deal with this difficulty is to leverage the mixing property of the uniformly ergodic Markov chain. Take the $\hg_x(\cdot)$ estimator for example,  the bias can be well controlled as long as  $T$ is selected larger than the mixing time $t_{\mix}$ of the sample path, which can be illustrated by the following decomposition
\begin{equation}\label{eqn:markov-decomp-demo}
\begin{aligned}
	\!\!\!\hg_x(Z^t;\zeta_t)&\!-\!\nabla_x \cL_w(Z^t)
	\,\,=\,\,\,\underbrace{
		\hg_x(Z^t;\zeta_t)\!-\!\hg_x(Z^{t-\tau};\zeta_t)\!+\!\nabla_x \cL_w(Z^{t-\tau})-\nabla_x \cL_w(Z^t)
	}_{\text{order }\bigO{\tau\eta}}\\
	&+\underbrace{
		\hg_x(Z^{t-\tau}\!;\zeta_t)\!-\!\cond{\hg_x(Z^{t-\tau}\!;\zeta_t)}{Z^{t-\tau}}
	}_{\text{zero mean}}\!+ \underbrace{
		\cond{\hg_x(Z^{t-\tau}\!;\zeta_t)}{Z^{t-\tau}}\!-\!\nabla_x \cL_w(Z^{t-\tau})
	}_{\text{order }\bigO{\exp(-\tau/t_{\mix})}}\!.
\end{aligned}\!\!
\end{equation}
When $t=\tbOm{t_{\mix}}$, one can bound the bias of $\hg_x(Z^t;\zeta_t)$ by $\tbO{t_{\mix}\eta}$ with suitably chosen $\tau$. 

\section{Lower Bound of Sample Complexity for Learning CMDP}
\label{sec:LowerBound}
In this section we will discuss whether the \algB\ Algorithm is the near-optimal  and whether  the Slater's condition (Assumption \ref{assumption:slater}) is necessary in achieving zero constraint violation. We answer these questions affirmatively by establishing the following theorems.

\begin{theorem}\label{thm:lower-bound-demo}
  Suppose $S\geq 4$, $A\geq 3$, $I\geq 8$, $C\geq 2$, $\gamma\in[\frac12,1)$, $N\geq 1$. For any learning algorithm $\mathfrak{A}$, there exists a CMDP $\cM=(\cS, \cA, \PP, r, (u_i)_{i\in[I]}, \gamma, \rho_0)$ and a reference distribution $\mu$, such that the following hold true.

  (1) $\nS\leq 4S+1$, $\nA\leq A$, and the concentrability coefficient $C^*$ for $\cM$ and $\mu$ satisfies $C^*\leq C$.

  (2) Let $\hpi$ be the policy output by $\mathfrak{A}$ given $N$ offline samples from $\mu$, and let $\pi^*$ be the optimal policy, then at least one of the following two inequalities hold true:
  $$
  \EE_{\cM,\mathfrak{A}}\!\left[J(\pi^*)\!-\!J(\hpi)\right]\!\geqsim\! \min\left\{\!\frac{1}{1-\gamma},\sqrt{\frac{\min\left\{SA,S\!+\!I\right\}C}{(1-\gamma)^3N}}\right\},
  \quad\mbox{and}\quad
  \EE_{\cM,\mathfrak{A}}\!\big[\!\mathrm{violation}(\hat{\pi})\!\big]\geqsim 1,
  $$
  where $\mathrm{violation}(\hat{\pi}):= \sum_{i=1}^I\neg{J^u_i(\hpi)}$, and $J^u_i$ is the utility w.r.t. the  constraints $J^u_i\geq 0, \forall i\in[I]$.
\end{theorem}
For \algB, the constraint violation is guaranteed to be zero with high probability, then only the first inequality is valid for our method, which indicates an $\Omega\big(\frac{\cN C^*}{(1-\gamma)^3\epsilon^2}\big)$  sample complexity lower bound. Therefore, the complexity of \algB\ is nearly optimal up to an $\tilde{\mathcal{O}}\big(\frac{1}{1-\gamma}\big)$ factor. 
Besides the lower bound, we also establish the necessity of the Slater's condition in ensuring zero violation. 

\begin{theorem}\label{theorem:slater}
	Let $S,A,C$, $\gamma$ be the same as Theorem \ref{thm:lower-bound-demo}. For any  algorithm $\mathfrak{A}$, there exists a CMDP $\cM=(\cS, \cA, \PP, r, (u_i)_{i\in[I]}, \gamma, \rho_0)$  with $I=1$, $\nS\leq S$, $\nA\leq A$ and a reference distribution $\mu$ with $C^*\leq C$, such that  $\EE_{\cM,\mathfrak{A}}[\mathrm{violation}(\hpi)]\geqsim \min\left\{\!\frac{1}{1-\gamma},\sqrt{\frac{SC}{(1-\gamma)^3N}}\right\},$  
	where $\hpi$ is the output policy  of $\mathfrak{A}$ given $N$ samples from $\mu$.
\end{theorem}
\cref{theorem:slater} is obtained by utilizing the same idea as Theorem \ref{thm:lower-bound-demo}. Thus we only discuss the derivation of Theorem \ref{thm:lower-bound-demo}, while moving all the details to Appendix \ref{appdx:LowerBound}.  

For offline CMDPs, the fixed  data distribution $\mu$ fully dominates the frequency of exploring the state-action pairs.  Therefore, intuitively, the hard CMDP instances will be the ones with a large support $\mathrm{supp}(\nu^{\pi^*})$ that widely spreads across the less frequently visited station-action pairs of $\mu$. 
Based on this intuition, we design a basic block of CMDP presented in \cref{fig:lower-bound}, which is essentially a constrained bandit with $2K+1$ arms. The  instance $\cM$ will be $S$ replicas of the basic blocks, plus an extra ``null'' state $s_{-1}$ to control $C^*$. In this discussion, we only consider the case where $I\simeq KS$, the more general construction that cover full range of $I$ is presented in the appendix.  \vspace{-0.2cm}\\
\begin{figure}[h]
	\centering
	\vspace{-0.5cm}
	\subfloat[$(s_1^j,a_i)$]{\includegraphics[width=0.3\textwidth]{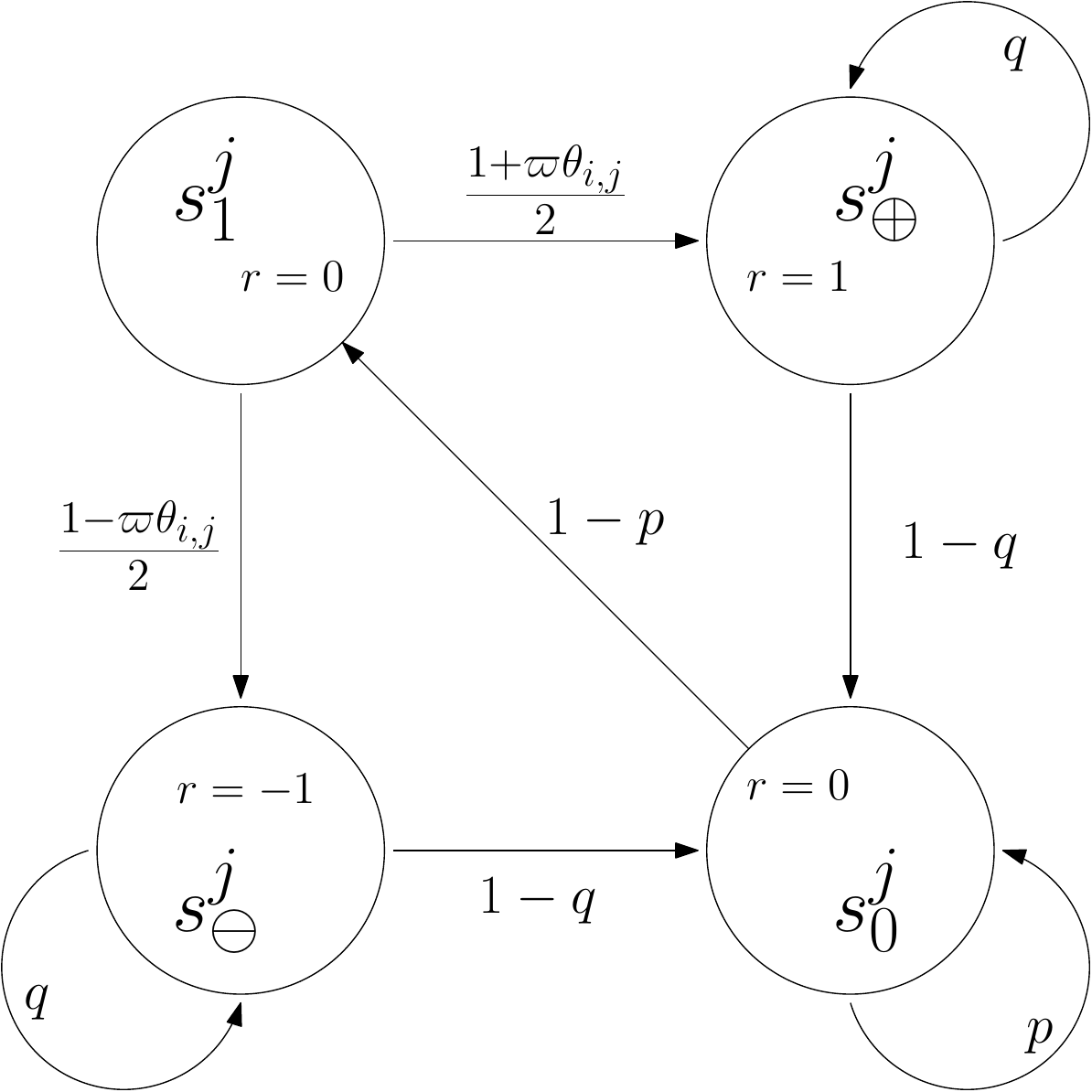}}\quad
	\subfloat[$(s_1^j,b_i)$]{\includegraphics[width=0.3\textwidth]{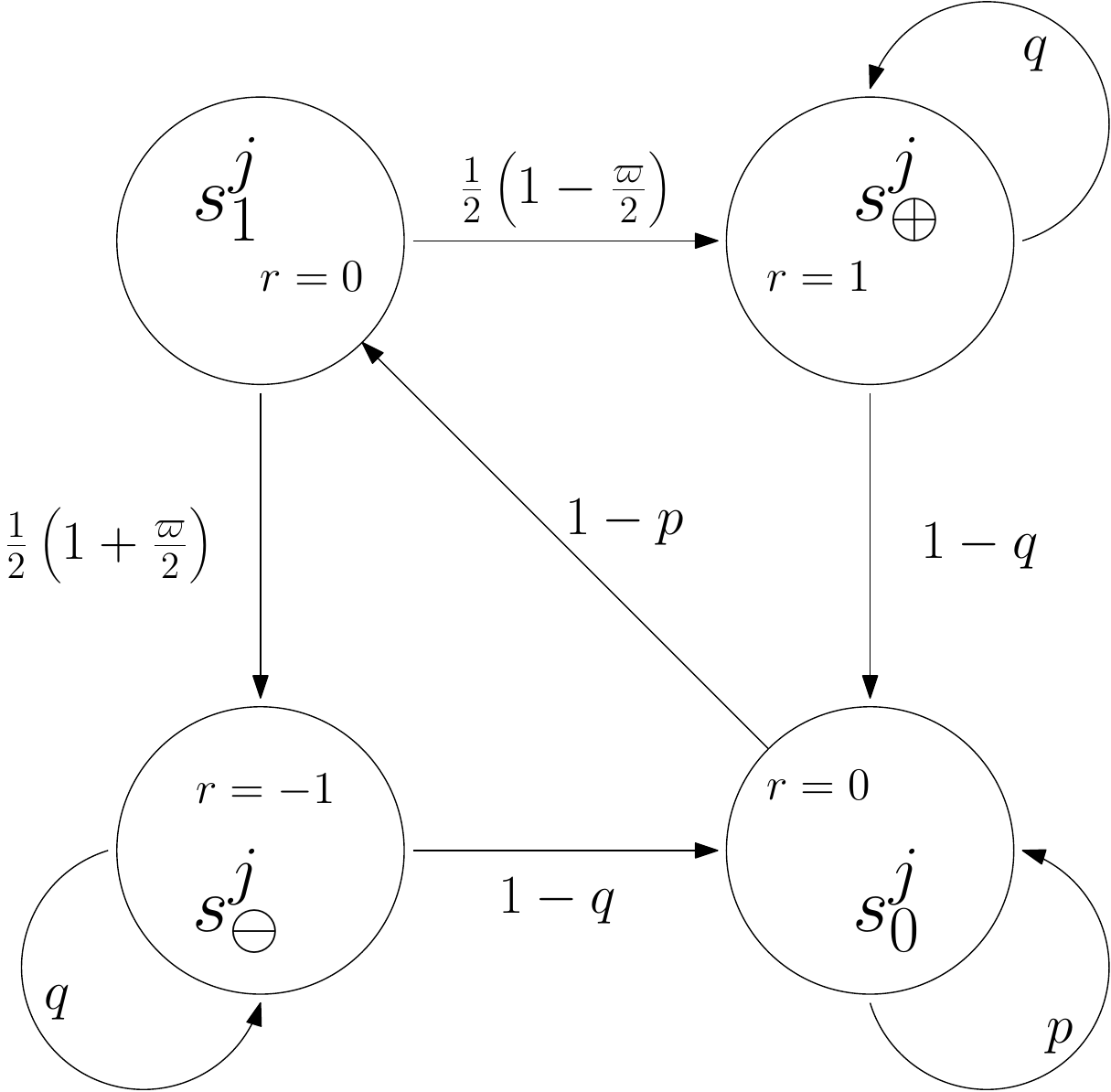}}\quad
	\subfloat[$(s_1^j,e)$]{\includegraphics[width=0.3\textwidth]{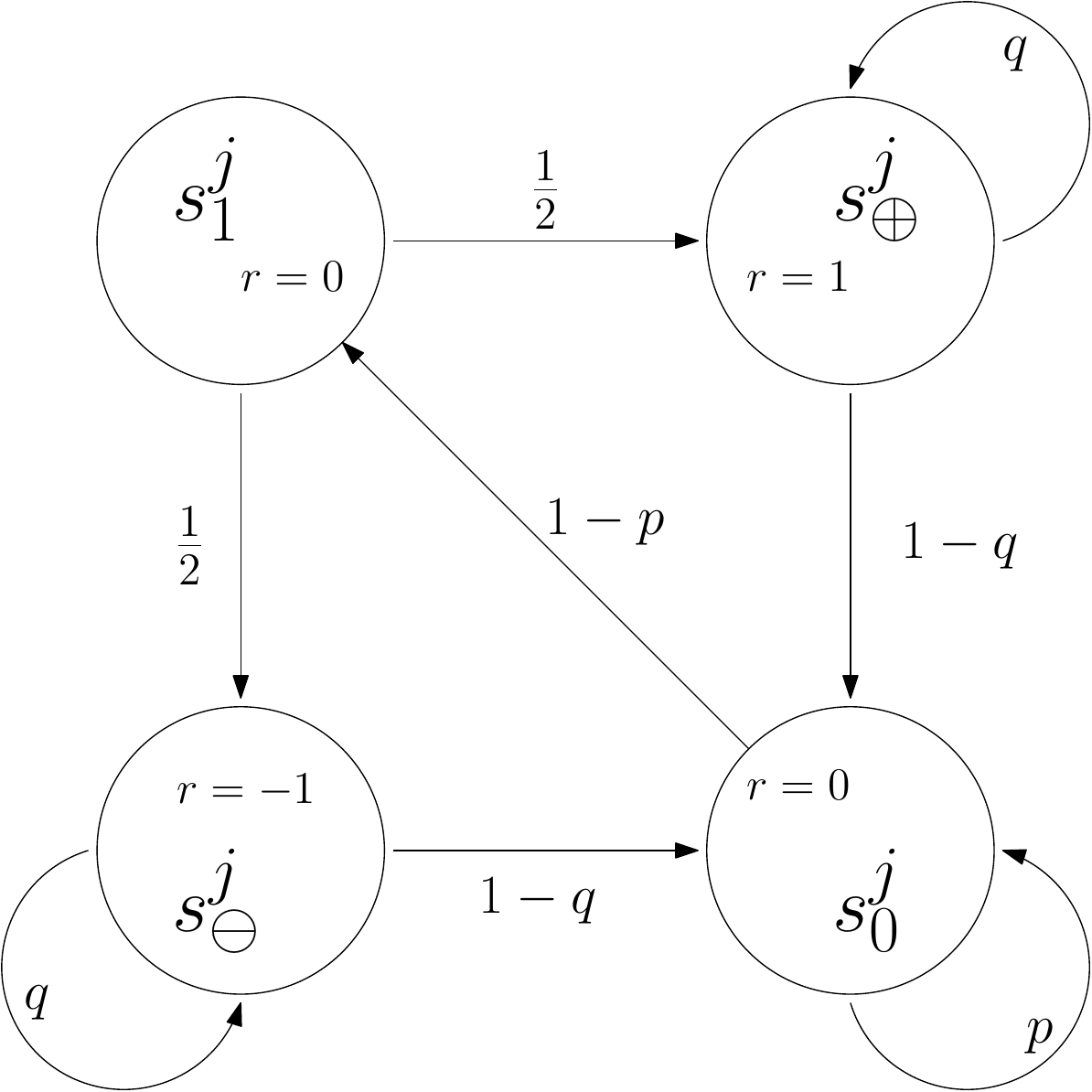}}
	\caption{Transition dynamics of the $j$th replica under different actions, $i\in[K]$.}
	\label{fig:lower-bound}
\end{figure}\vspace{-0.2cm}\\
\textbf{State, action and transition}. At the states $\sp^j,\sm^j,s_0^j$,  there is no action to be taken. At each state $s_1^j$, there are $2K+1$ actions $a_1,b_1,\cdots,a_K,b_K, e$. The transition dynamics of the $j$th replica under different actions are illustrated in \cref{fig:lower-bound} where the directed arcs  and the numbers associated with them are the transitions and the corresponding probabilities, where $p = \frac{1}{2-\gamma}$ and  $q=2-\frac{1}{\gamma}$ are some constants, while $\varpi$ and $\theta_{i,j}\in\{-1,1\}$, $\forall i,j$ are parameters to be designed.  \vspace{-0.1cm}

\textbf{Constraints and Reward}. By carefully selecting the  $u_i$'s, one can construct a set of $I\!=\!2SK$ constraints that indicate $\pi(a_i|s_1^j)\!\leq\!\pi(b_i|s_1^j)\!\leq\!\frac{1}{4K}$, $\forall i,j$.  For the reward, we set $r(s_1^j)=r(s_0^j)=0$, $r(\sp^j)=1$, and $r(\sm^j)=-1$,  regardless of the actions. At any replica $j$, we can view $a_i$, $b_i$, and $e$ as bandit arms with (cumulative) reward $c\varpi\theta_{i,j}$, $-\frac{c\varpi}{2}$, and $0$ respectively, for some $c>0$. When $\theta_{i,j}\!=\!-\!1$, one would rather pick $e$ . But when $\theta_{i,j}=1$, due to the constraint $\pi(a_i|s_1^j)\!\leq\!\pi(b_i|s_1^j)\!\leq\!\frac{1}{4K}$, picking $a_i$ and $b_i$ with equal probability $\frac{1}{4K}$ will be optimal. In fact, this $\frac{1}{4K}$ upper bound forces the support of the optimal policy to widely spread  across the $(i,j)$'s where $\theta_{i,j}=1$, and the task of learning is essentially determining whether $\theta_{i,j}=1$ for each $(i,j)$. \vspace{-0.1cm}


\textbf{Optimal policy}. Based on the above discussion, it is not hard to see that the unique optimal policy is  $\pi^{*,\theta}(a_i|s_1^j)=\pi^{*,\theta}(b_i|s_1^j)\!=\!\frac{\II\{\theta_{i,j}=1\}}{4K}$ and $\pi^{*,\theta}(e|s_1^j) \!=\! 1\!-\!\frac{1}{2K}\!\sum_{i=1}^K\!\II\{\theta_{i,j}\!=\!1\}$. 

Finally, with the above $\pi^{*,\theta}$ and a proper  initial distribution $\rho_0$, the occupancy measure can be explicitly computed and a reference distribution $\mu$ with concentrability coefficient $C^*\leq C$ can be designed. \vspace{-0.1cm}
Moreover, for any policy $\hat{\pi}$, we consider $\hat{\theta}_{i,j}(\hpi)\!:=\!8K\hpi(a_i|s_1^j)-1$, then\vspace{-0.1cm}
	\begin{equation*}
		\begin{aligned}
			\cL(\hpi;\theta):=\pos{J(\pi^{*,\theta};\theta)-J(\hpi;\theta)}+\frac{\gamma\varpi}{1-\gamma}\violation(\hpi;\theta)\geq \frac{\gamma^2\varpi\|\hat{\theta}(\hpi)-\theta\|_1}{64KS(1-\gamma)}.
		\end{aligned}\vspace{-0.05cm}
	\end{equation*}
	Namely, if $\hat{\theta}(\hpi)$ is not close enough to the underlying parameter $\theta$, the policy $\hpi$ will incur a considerable reward gap or constraint violation. 
By setting $\varpi\!=\!\min\left\{\!\sqrt{\frac{(SK-3)C}{16(1-\gamma)N}},\frac12\right\}$ to be a small enough number, any two CMDP instances with different $\theta$ parameters will be non-distinguishable, given $N$ samples from $\mu$.
According to \cite{gilbert1952comparison} and \cite{varshamov1957estimate}, there exists a subset $\varTheta\!\subseteq\!\{-1,1\}^{\!SK}$ such that $|\varTheta| \geq \exp (SK / 8)$, and $\left\|\theta-\theta'\right\|_{1} \!\geq\! \frac{SK}{2}$ for any pair of different $\theta, \theta'\!\in\! \varTheta$. In other words, there will be at least $\exp (SK / 8)$  CMDP instances with different enough $\theta$ parameters while being non-distinguishable under $N$ samples. Then the rest of the arguments will follow by applying the generalized Fano's inequality \cite{assouad1996fano}. A detailed proof is provided in Appendix \ref{appdx:LowerBound}.


\section{Adaptive deviation-control framework of \algB}\label{sect:adaptive-algB}

We should notice that in both Theorems \ref{thm:algB-thm-demo} and \ref{thm:algB-thm-markov-demo}, it has been explicitly emphasized that a prior belief $\hphi\geq C^*$ is required. 
Otherwise, both the reward and the constraints will suffer an extra loss of $\epsapp(\hphi)$. 
In this section, we propose an adaptive deviation-control framework (\cref{algo:adaptive-algB-demo}) to handle the practical situation where no such prior knowledge is available. \vspace{-0.1cm}

\begin{algorithm}
	\caption{The \adaptiveAlg\ framework } \label{algo:adaptive-algB-demo}
	\Input{Sub-optimality $\epsilon$, confidence level $\delta$.} 
	Initialize $\hphi_1$, default $J^K\equiv-\infty,$ for $K=0,1,2,...$\;
	\For{$K=1,2,\cdots$}{
		Call \algB\ with $\hphi=\hphi_K$, obtain an approximate solution $x^{(K)}$ and the policy $\pi^{(K)}$\;
		\If{$\verify\left(\xk;\epsilon,\delta\right)== \true$}{
			Compute $\hJ(\pi^{(K)})$ as an $\mathcal{O}(\epsilon)$-accurate estimator of $J(\pi^{(K)})$, set $J^K=\hJ(\pi^{(K)})$\;
		}
	   \lIf{$-\infty <J^K\leq J^{K-1} +\mathcal{O}(\epsilon)$}{\textbf{Terminate}}
       Set $\hphi_{K+1} = 2\hphi_K$\;
	}
	\Output{
		Policy $\pik$.\vspace{-0.1cm}
	}
\end{algorithm}

At a high-level, \cref{algo:adaptive-algB-demo} consists of the following steps.

\paragraph{Verification} For the output $\ox$ of the \algB, we develop a verification method $\verify(\ox;\epsilon,\delta)$ that, \highprobdemo, returns \true\ only when following two statements hold: (1). The vector $\overline{\nu}:=W\ox$ satisfies $\nrm{\sum_{a} (\id-\gamma \bP_a)\overline{\nu}_a-\rho_0}_1=\mathcal{O}(\epsilon)$, which essentially checks whether $\overline{\nu}$ is approximately a valid occupancy measure; (2). The policy $\opi$ induced by $\ox$ is safe. At step $K$, if any one of the two statements does not hold,  we immediately know $\hphi_K<C^*$ due to the analysis of Theorem \ref{thm:algB-thm-demo}. Consequently, we to double the coefficient $\hphi_{K+1}\leftarrow 2\hphi_K$ in the next iteration. 

\paragraph{Certifying performance improvement} When $\verify(\ox;\epsilon,\delta)$ returns \true, then it holds that
$j_0(\hphi)=J(\pi^{(K)})+\mathcal{O}(\epsilon)$, where $j_0(\hphi)$ denotes the optimal value of problem \eqref{prob:weighted-minimax} with $\kappa=0$. That is,  one can estimate $j_0(\hphi)$ with $\hJ(\pi^{(K)})$ if $\verify(\ox;\epsilon,\delta)=\true$. As long as \verify\ returns \true\ for two consecutive runs, and the performance improvement is small, i.e., $j_0(\hphi_{K})-j_0(\hphi_{K-1}) = \mathcal{O}(\epsilon)$, then Lemma \ref{lemma:adaptive-j0-error-bound-demo} guarantees that the safe policy $\pi^{(K)}$ is $\mathcal{O}(\frac{C^*}{\hphi_K}\epsilon)$-optimal.

\begin{lemma}\label{lemma:adaptive-j0-error-bound-demo}
	The function $j_0(\cdot)$ is strictly increasing in the range $\hphi\in[1,C^*]$, and $j_0(\hphi)=J(\pi^*)$ for $\hphi\geq C^*$. 	For any $\hphi<\hphi'\leq\ C^*$, it holds that\vspace{-0.05cm}
	$$J(\pi^*)-j_0(\hphi')\leq \frac{C^*-\hphi}{\hphi'-\hphi}\left(j_0(\hphi')-j_0(\hphi)\right).$$
\end{lemma}
 \vspace{-0.05cm}
Detailed descriptions of $\verify$ and \adaptiveAlg\ are presented in \cref{appdx:adaptive}, and so does the proof of the following theorem.
\begin{theorem}\label{thm:adaptive-algB-demo}
  Fixed $\epsilon\in\big(0,\frac{1}{10(1-\gamma)}\big], \delta\in(0,1)$. Then \highprobdemo, \adaptiveAlg\ stops at step $K$ such that $\psi_K\leq 4C^*$ and outputs the safe policy $\pi^{(K)}$ with sub-optimality gap $J(\pi^*)-J(\pi^{(K)})\leq\bigO{\frac{C^*}{\hphi_{K}}\epsilon}$. Moreover, there exists a (problem dependent) constant $\epsilon_0(\cM)$ such that, if $\epsilon\leq\epsilon_0(\cM)$, then it must hold that $\hphi_{K}\in [C^*,2C^*)$ and $\pi^{(K)}$ is $\epsilon$-optimal.
\end{theorem}
Intuitively, the \adaptiveAlg\ will quickly terminate within $\mathcal{O}(\log_2C^*)$ calls of \algB, resulting in a total samples complexity of $\tbO{\frac{\cN C^*}{(1-\gamma)^4\epsilon^2}}$. 


\bibliographystyle{plainnat}
\bibliography{cmdp}

\newpage

\appendix

\section{Efficiently solving the subproblems of \algB}\label{subsect:compute}

In this section, we describe how to efficiently solve the subproblems \eqref{eqn:def-update} in the \algB\ Algorithm. In the following discussion, at most $\tilde{\mathcal{O}}(|\cS||\cA|+I)$ flops are needed to compute the update.

\subsection{Closed form solution for the $V$-update} 
The  dual variable $V$ is updated by the formula
$V^{t+1} =\Proj_\cV\left(V^{t}-\eta_t\alpha^{-1}_Vg_V^{t}\right),$ 
where $\cV$ is an $\ell_\infty$ normal ball defined as $\cV:=\big\{V\!\in\!\RR^{\nS}: \nrm{V}_{\infty}\leq R_\cV\big\}$, $R_\cV=\frac{8}{1-\gamma}(1+\frac{2}{\varphi})$. For  any vector $V\in\RR^{\nS}$, the Euclidean projection $V_+=\Proj_\cV(V)$ can be written as a simple  truncation
$$V_+(s) = \begin{cases}
	-R_\cV,& \mbox{if } V(s)<-R_\cV,\\
	V(s),& \mbox{if } -R_\cV\leq V(s)\leq +R_\cV,\\
	+R_\cV, & \mbox{if } V(s)>+R_\cV,
\end{cases}\qquad\mbox{for}\qquad \forall s\in\cS.$$
This update will need $\mathcal{O}(1)$ flops due to the special structure of $g_V^{t}$. 

\subsection{Closed form solution for the $\lambda$-update} 
The dual variable $\lambda$ is updated by the formula $\lambda^{t+1}\!=\!\argmin_{\lambda\in\Lambda}\big(\!\langle g_\lambda^t,\lambda\!-\!\lambda^t\rangle\!+\!\frac{\alpha_\lambda}{\eta_t} \mathrm{KL}(\lambda||\lambda^t)\big)$, where $\Lambda$ is the nonnegative part of an $\ell_1$ norm ball $\Lambda =\big\{\lambda\in\RR^I_{\geq0}: \nrm{\lambda}_1\leq \RLam\big\}$, $\RLam=\frac{8}{\varphi}$. The solution to this subproblem has the following closed form formula
\[\begin{aligned} 
\lambda^{t+1} = \lambda^{t+\frac{1}{2}}\min\left\{\frac{\RLam}{\nrm{\lambda^{t+\frac{1}{2}}}_1}, 1\right\},
\end{aligned}\]
where $\lambda^{t+\frac{1}{2}} =\lambda^t\exp(-\frac{\eta_t}{\alpha_\lambda}g_{\lambda}^t)$ is an intermediate point.
This update will need $\mathcal{O}(I)$ flops. 
\subsection{Efficient implementation of the $x$-update} 
Compared to the previous two updates, the subproblem for $x$-update does not have a closed form solution. By carefully discussing the KKT condition of the problem and utilizing the special structure of $\cX$ and $g_x^t$, we reduce the problem to finding the root of a monotonically decreasing 1-dimensional function. If the bisection method is applied to find the root, then in total $\tilde{\mathcal{O}}(\nS\nA)$ flops are needed. We present the details as follows. For notational simplicity, we rewrite the subproblem as follows.
\begin{problem}
  Given a set $\cY$ defined by the linear constraints
  \[\begin{aligned}
  \cY:=\Big\{y\in \RR^n: 0\leq y_i\leq a_i, \sum_{i=1}^n y_i\leq B_1, \sum_{i=1}^n c_iy_i\leq B_2\Big\},
  \end{aligned}\]
  where $B_1,B_2>0$, and $c_i>0$ are some constants. Let $y^0\in\cY$,  $y^0>0$, and let $g\in\RR^n$ be a vector that has at most 1 non-zero entry. Then the goal is to solve
  \begin{equation}\label{eqn:compute-proj}
  \begin{aligned}
    y^*=\arg\min_{y\in\cY} \left(\iprod{y}{g}+\KL{y}{y^0}\right).
  \end{aligned}
  \end{equation}
\end{problem}

Without loss of generality, we assume $g_2=\cdots=g_n=0$.
For problem \eqref{eqn:compute-proj}, we introduce two Lagrangian multipliers to the coupling constraints $\sum_{i=1}^n y_i\leq B_1, \sum_{i=1}^n c_iy_i\leq B_2$, while remaining the coordinately separable constraints $0\leq y_i\leq a_i$ in the problem. Thus we get the following Lagrangian function: 
\begin{equation} 
L(y,\alpha,\beta):=y_1g_1+\mathrm{KL}(y||y^0)+\alpha\big(\sum_i y_i-B_1\big)+\beta\big(\sum_i c_iy_i-B_2\big). 
\end{equation}
By the strong convexity of KL divergence, there is a unique KKT point $(y^*, \alpha^*, \beta^*)$ of problem \eqref{eqn:compute-proj}. Note that  $y^* = \argmin_{y_i\in[0,a_i],\forall i} L(y,\alpha^*,\beta^*)$. Because $y^i_0>0$, we know 
$$\lim_{y_i\to0+} \nabla_{y_i} L(y, \alpha^*, \beta^*) = \lim_{y_i\to0+} g_1\cdot\II\{i=1\} + \alpha^*+c_i\beta^*+\log y_i-\log y_i^{0} = -\infty,$$
we know $y_i^*$ will not be 0. Thus we can write the KKT condition for problem \eqref{eqn:compute-proj} as
\begin{equation}
	\label{eqn:compute-KKT}
	\begin{cases}
		\nabla_{y_i} L(y^*, \alpha^*, \beta^*) \leq 0,  \quad \mbox{if } y_i^*=a_i,\quad  \forall i\in[n],\\
		\nabla_{y_i} L(y^*, \alpha^*, \beta^*) = 0, \quad \mbox{if } y_i^*\in(0,a_i),\quad  \forall i\in[n],\\
		\alpha^*\big(\sum_i y_i^*-B_1\big)=0,\quad
		\beta^*\big(\sum_i c_iy_i^*-B_2\big)=0,\\
		y^*\in\cY, \alpha^*\geq 0, \beta^*\geq 0.
	\end{cases}
\end{equation}

For $i=2,...,n$, the condition
$\nabla_{y_i} L(y^*, \alpha^*, \beta^*)\leq 0 $ implies that $y_i^*\leq y_i^{0}\exp(-\alpha^*-c_i\beta^*)$. Note that $\alpha^*,\beta^*\geq0,c_i>0$, $y_i^0\leq a_i$. If $y_i^*<a_i$, then $\nabla_{y_i} L(y^*, \alpha^*, \beta^*)=0$ indicates that $y_i^*=y_i^{0}\exp(-\alpha^*-c_i\beta^*)$. If $y_i^*=a_i$, then the only possibility is $y_i^0=a_i$ happen to hold and $\alpha^*=\beta^*=0$, in this case, we still have $y_i^*=y_i^{0}\exp(-\alpha^*-c_i\beta^*)$. A similar formula can also be derived for $y_1^*$. Therefore, utilizing the feasibility of the point $y^0$, we solve the first two rows of the KKT condition and get   
\begin{equation}
	\label{eqn:y}
	\begin{cases}
		y_1^*(\alpha^*,\beta^*)=\min\left\{y_1^0\exp(-g_1-\alpha^*-c_1\beta^*),a_1\right\},\\
		y_i^*(\alpha^*,\beta^*)=y_i^{0}\cdot\exp\{-\alpha^*-c_i\beta^*\}, \quad\mbox{for }i=2,...,n.		
	\end{cases}
\end{equation}
Here, we write $y_i^*$ as functions of $\alpha^*,\beta^*$ for the ease of later discussion. Next, we solve the third row of the KKT condition \eqref{eqn:compute-KKT} by considering the following cases. 

\paragraph{Case 1: $\beta^*=0,\alpha^*=0$.} In this case, if $y^*(0,0),\alpha^*=0,\beta^*=0$ satisfies \eqref{eqn:compute-KKT}, then   $y^*(0,0)$ is the solution to \eqref{eqn:compute-proj}. Otherwise we conclude that $\alpha^*=\beta^*=0$ is not true. 

\paragraph{Case 2: $\beta^*=0,\alpha^*>0$.} In this case, the KKT condition tells us that $\sum_i y_i^*=B_1$. Together with \eqref{eqn:y}, we have the following two possible solutions to $\alpha^*$
$$
\begin{cases}
	\alpha_1 = \ln\left(\frac{y_2^0+\cdots+y_n^0}{B_1-a_1}\right), & \mbox{corresponds to } y_1^*=a_1,\\
	\alpha_2 = \ln\left(\frac{e^{-g_1}\cdot y_1^0+y_2^0+\cdots+y_n^0}{B_1}\right), & \mbox{corresponds to } y_1^*=y_1^0\exp(-g_1-\alpha^*).
\end{cases}
$$
Then if $y^*(\alpha_1,0),\alpha^*=\alpha_1,\beta^*=0$ satisfies \eqref{eqn:compute-KKT}, we conclude that $y^*(\alpha_1,0)$ is the solution to \eqref{eqn:compute-proj}. If $y^*(\alpha_2,0)\in\cY, \alpha^*=\alpha_2,\beta^*=0$ satisfies \eqref{eqn:compute-KKT}, we conclude that $y^*(\alpha_2,0)$ is the solution to \eqref{eqn:compute-proj}. Otherwise, we know $\alpha^*>0, \beta^*=0$ is not possible. 

\paragraph{Case 3: $\beta^*>0,\alpha^*=0$.} In this case, the KKT condition tells us that $\sum_i c_iy_i^*=B_2$. Denote $\hat y_1^0 = y_1^0\exp(-g_1)$, $\hat y^0_i=y^0_i, i=2,...,n$. In this case, depending on the value of $y_1^*$ we set
$$
\begin{cases}
	\beta_1 = \mathrm{Root}_{\beta>0} \big\{\sum_{i=2}^n c_i\hat y_i^0\exp(-c_i\beta)=B_2-c_1a_1\big\},\\
	\beta_2 = \mathrm{Root}_{\beta>0} \big\{ \sum_{i=1}^n  c_i \hat y_i^0\exp(-c_i\beta)=B_2\big\} .
\end{cases}
$$
Note that in both cases, the problem is finding the positive root of a 1-dimensional monotonically decreasing function, which can be solved efficiently. These equations should either have one unique positive solution or no positive solution at all. If there is no positive root, then $\mathrm{Root}_{\beta>0}$ will return \false. One can easily determine whether there is a positive solution. For example, due to the monotonicity, the first equation will have a positive solution if and only if $\sum_{i=2}^n\hat c_iy_i^0>B_2-c_1a_1$. 

Similar to case 2, we check the feasibility of  $\{y^*(0,\beta_1),\alpha^*=0,\beta^*=\beta_1\}$ and  $\{y^*(0,\beta_2),\alpha^*=0,\beta^*=\beta_2\}$  w.r.t.  \eqref{eqn:compute-KKT}. If any one of them is feasible to the KKT condition, then it will be the solution to \eqref{eqn:compute-proj}. Otherwise, we know $\alpha^*=0, \beta^*>0$ is not possible. 

\paragraph{Case 4: $\beta^*>0,\alpha^*>0$.} In this case, the KKT condition implies that $\sum_i\! c_iy_i^*\!=\!B_1,$ $\sum_i\! c_iy_i^*\!=\!B_2$. Let us inherit the $\hat y$ notation from Case 3. Then we need to solve the following group of equations
\begin{equation*}
	\begin{cases}
		\sum_{i=2}^n \hy_i^{0}\exp(-\alpha_3-c_i\beta_3)=B_1-a_1,\\
		\sum_{i=2}^n c_i\hy_i^{0}\exp(-\alpha_3-c_i\beta_3)=B_2-c_1a_1
	\end{cases} \mbox{or}\qquad\begin{cases}
	\sum_{i=1}^n \hy_i^{0}\exp(-\alpha_4-c_i\beta_4)=B_1,\\
	\sum_{i=1}^n c_i\hy_i^{0}\exp(-\alpha_4-c_i\beta_4)=B_2
\end{cases} 
\end{equation*}
We should notice that in both cases, as soon as we determine the value of $\beta$, then $\alpha$ will have a closed form formula given $\beta$. To demonstrate how to determine $\beta$, let us take the second group of equations for example. Taking the quotient between the two equations cancels $\alpha_4$, we get the following equation of $\beta_4$    
\begin{equation}
\begin{aligned}
  f(\beta_4):=\frac{\sum_{i=1}^n c_i\hy_i^{0}\exp(-c_i\beta_4)}{\sum_{i=1}^n \hy_i^{0}\exp(-c_i\beta_4)}=\frac{B_2}{B_1}.
\end{aligned}
\end{equation}
By Cauchy's inequality, we know $f'(\beta)<0$ holds for $\forall\beta\in\RR$ if $c_i\neq c_j$ for some $i,j$. In details
\[\begin{aligned}
  f'(\beta)=\frac{\left(\sum_{i=1}^n c_i\hy_i^{0}\exp(-c_i\beta)\right)^2-\left(\sum_{i=1}^n \hy_i^{0}\exp(-c_i\beta)\right)\left(\sum_{i=1}^n c_i^2\hy_i^{0}\exp(-c_i\beta)\right)}{\left(\sum_{i=1}^n \hy_i^{0}\exp(-c_i\beta)\right)^2}<0.
\end{aligned}\]
Hence, $f$ is again a monotonically decreasing function, and finding its positive root can be implemented efficiently. After finding  $\beta_4$, one immediately know $\alpha_4 = \ln\Big(\frac{\sum_{i=1}^n\hat y_i^0\exp(-c_i\beta_4)}{B_1}\Big)$.

Finally, we need to check the feasibility of $\{y^*(\alpha_3,\beta_3),\alpha^*=\alpha_3,\beta^*=\beta_3\}$ and $\{y^*(\alpha_4,\beta_4),\alpha^*=\alpha_4,\beta^*=\beta_4\}$ w.r.t. \eqref{eqn:compute-KKT}. If any one of them is feasible to the KKT condition, then it will be the solution to \eqref{eqn:compute-proj}. Otherwise, we know $\alpha^*>0, \beta^*>0$ is not possible. Due to the existence of a KKT pair, at least one of the 4 cases will return us a solution. 
\section{Proof of Proposition \ref{prop:empirical-dis}}\label{subsect:empirical-dis}
For the analysis of Proposition \ref{prop:empirical-dis} and later results, let us first introduce a vector version of the Bernstein's inequality, which is a direct specification of the Freedman's inequality of matrix martingale \cite{freedman}. To prove  the current proposition, we only need the scalar case of the following lemma. 
\begin{lemma}[Vector Bernstein Inequality]
	\label{lemma:concen-vec}
	Assume that $\{x_i\}_{i=1}^n$ is a sequence of random vectors in $\RR^{d}$, and it forms a martingale difference sequence with respect to $(\cF_t)$ (i.e. $\cond{x_t}{\cF_{t-1}}=0$ and $x_t$ is $\cF_{t}$-measurable). If $\cond{\|x_t\|^2}{\cF_{t-1}}\leq \sigma^2$ and $\|x_t\|\leq M$ a.s., then with probability at least $1-\delta$,
	\[
	\nrm{\sum_{i=1}^n x^i}\leq
	2\sigma\sqrt{n\log\left(\frac{d+1}{\delta}\right)} + 2M\log\left(\frac{d+1}{\delta}\right).
	\]
	When the $\ell_2$ norm is replaced by the $\ell_\infty$ norm, i.e., $\{x_i\}_{i=1}^n$ satisfies $\cond{\|x_t\|_{\infty}^2}{\cF_{t-1}}\leq \sigma^2$ and $\|x_t\|_{\infty}\leq M$, 
	\[
	\nrm{\sum_{i=1}^n x^i}_{\infty}\leq
	2\sigma\sqrt{n\log\left(\frac{2d}{\delta}\right)} + 2M\log\left(\frac{2d}{\delta}\right)
	\]
	holds with probability at least $1-\delta$.
\end{lemma} 

To prove Proposition \ref{prop:empirical-dis}, 
we consider $\hmu_0(s,a)=\frac{N(s,a)}{N_e}$, then it is clear that $\hmu(s,a) = \max(\hmu_0(s,a),\varsigma)$. Now, according to the Bernstein's inequality, we construct the ``failure event'' 
\[
\Omega:=\bigcup_{s,a}\left\{\left|\mu(s,a)-\hmu_0(s,a)\right|> \sqrt{\mu(s,a)\frac{\ell}{N_e}}+\frac{\ell}{N_e}\right\},
\]
where $\ell\geq 4\log\left(\frac{6\nS\nA}{\delta}\right)$ is a mild logarithmic term. We next prove the three properties listed in Proposition \ref{prop:empirical-dis} one by one. 

\paragraph{Proof of \cref{prop:empirical-dis} (1).} In fact, we only need to show that $\PP(\Omega)\leq \frac{\delta}{3}$, and the event $\Omega^c$ implies that $\mu(s,a)\leq 2\hmu(s,a), \forall s,a$, as long as our choice of batch size satisfies
$N_e \geq \frac{128\cN\hphi\ell}{\varphi^2(1-\gamma)^4\epsilon_e^2}\geq \frac{32\ell\cN\hphi}{\varphi(1-\gamma)^2\epsilon_e}=\frac{32\ell}{\varsigma}$.

By Bernstein's inequality, it holds that
  \[\begin{aligned}
    \prob{\left|\mu(s,a)-\hmu_0(s,a)\right|> \sqrt{\mu(s,a)\frac{\ell}{N_e}}+\frac{\ell}{N_e}}\leq \frac{\delta}{3\nS\nA}, \quad \forall (s,a)\in\cS\times\cA.
  \end{aligned}\]
  Then $\PP(\Omega)\leq \frac{\delta}{3}$ follows directly from the union bound. Conditioning on $\Omega^c$, we have
  \begin{equation}\label{eqn:emp-dis-basic}
  \begin{aligned}
	&\abs{\mu(s,a)-\hmu_0(s,a)}\leq \sqrt{\mu(s,a)\frac{\ell}{N_e}}+\frac{\ell}{N_e}
	\leq \sqrt{\mu(s,a)\frac{\varsigma}{32}}+\frac{\varsigma}{32}
	\leq \frac{\mu(s,a)}{4}+\frac{\varsigma}{16}.
  \end{aligned}
  \end{equation}
  Hence, it holds that
  \[\begin{aligned}
  \mu(s,a)\leq \frac{4}{3}\hmu_0(s,a)+\frac{\varsigma}{12} \leq \frac{3}{2}\max(\hmu_0(s,a),\varsigma)\leq 2\hmu(s,a).
  \end{aligned}\]
From now on, the argument is all conditioning on $\Omega^c$. 

\paragraph{Proof of Proposition \ref{prop:empirical-dis} (2).} Given a $\pi\in\Pi(\hphi)$, we have to prove that $W^{-1}\nu^{\pi}\in\cX$.

Let $\nu=\nu^{\pi}$, $x=W^{-1}\nu$. Then due to $\pi\in\Pi(\hphi)$, we have
  \[\begin{aligned}
    &\max_{s,a} \frac{x(s,a)}{\hmu(s,a)}=\max_{s,a}\frac{\nu(s,a)}{\mu(s,a)}\leq \frac{\hphi}{1-\gamma},\\
    &\sum_{s,a} \frac{x(s,a)}{\hmu(s,a)}=\sum_{s,a}\frac{\nu(s,a)}{\mu(s,a)}\leq \frac{\cN\hphi}{1-\gamma}.
  \end{aligned}\]
Now it remains to show $\sum_{s,a} x(s,a)\leq \frac{4}{1-\gamma}$. Note that \eqref{eqn:emp-dis-basic} also implies
\[\begin{aligned}
\hmu_0(s,a)\leq \frac{5}{4}\mu(s,a)+\frac{\varsigma}{16}.
\end{aligned}\]
Hence if $\mu(s,a)\leq \frac{1}{2}\hmu(s,a)$, then it must hold that $\hmu_0(s,a)<\hmu(s,a)\Rightarrow \hmu_0(s,a)<\varsigma, \hmu(s,a)=\varsigma$. We define $\mathfrak{S} :=\{(s,a)\in\cS\times\cA: \hmu(s,a)=\varsigma\}$, then for $(s,a)\not\in\mathfrak{S} $, it holds that $\mu(s,a)\geq \frac{1}{2}\hmu(s,a)$. Thus, we have
	\[\begin{aligned}
		\sum_{s,a} x(s,a)
		& = \sum_{(s,a)\in\mathfrak{S} } \hmu(s,a)\frac{\nu(s,a)}{\mu(s,a)}+\sum_{(s,a)\not\in\mathfrak{S} } \frac{\hmu(s,a)}{\mu(s,a)}\nu(s,a) \\
		& \leq \varsigma\frac{\cN\hphi}{1-\gamma} +\sum_{(s,a)\not\in\mathfrak{S} }2\nu(s,a)\\
		& \leq \frac{3}{1-\gamma}.
	\end{aligned}\]
The last inequality holds as long as $\varsigma\leq \frac{1}{\cN\hphi}$.

\paragraph{Proof of Proposition \ref{prop:empirical-dis} (3).} We decompose the quantity $\nrm{Wx-x}_1$ as
\begin{equation*}
	\begin{aligned}
		\nrm{Wx-x}_1
		&=\sum_{s,a}\left|\mu(s,a)-\hmu(s,a)\right|\frac{x(s,a)}{\hmu(s,a)}\\
		&=\sum_{(s,a)\in\mathfrak{S} } \left|\mu(s,a)-\hmu(s,a)\right|\frac{x(s,a)}{\hmu(s,a)}
		+\sum_{(s,a)\not\in\mathfrak{S} } \left|\mu(s,a)-\hmu(s,a)\right|\frac{x(s,a)}{\hmu(s,a)}.
	\end{aligned}
\end{equation*}
From our definition of $\mathfrak{S} $, we see if $(s,a)\in\mathfrak{S} $, then $\hmu(s,a)=\varsigma\geq \hmu_0(s,a)$, and from \eqref{eqn:emp-dis-basic} we have $\mu(s,a)\leq 2\varsigma \Rightarrow \abs{\mu(s,a)-\hmu(s,a)}\leq \varsigma$. Thus, the first part can be bounded as
\[\begin{aligned}
  \sum_{(s,a)\in\mathfrak{S} } \left|\mu(s,a)-\hmu(s,a)\right|\frac{x(s,a)}{\hmu(s,a)}
  \leq  \sum_{s,a} \varsigma\frac{x(s,a)}{\hmu(s,a)}
  \leq \varsigma \frac{\cN\hphi}{1-\gamma}.
\end{aligned}\]
As for the second part, we have
\begin{align*}
    &\sum_{(s,a)\not\in\mathfrak{S} } \left|\mu(s,a)-\hmu(s,a)\right|\frac{x(s,a)}{\hmu(s,a)}\\
    &=\sum_{(s,a)\not\in\mathfrak{S} } \left|\mu(s,a)-\hmu_0(s,a)\right|\frac{x(s,a)}{\hmu(s,a)}\\
		&\leq \sum_{(s,a)\not\in\mathfrak{S} } \left(\sqrt{\mu(s,a)\frac{\ell}{N_e}}+\frac{\ell}{N_e}\right)\frac{x(s,a)}{\hmu(s,a)}\\
    &= \sqrt{\frac{\ell}{N_e}} \sum_{(s,a)\not\in\mathfrak{S} } \sqrt{\frac{\mu(s,a)}{\hmu(s,a)}}\sqrt{x(s,a)\cdot\frac{x(s,a)}{\hmu(s,a)}}
		+\frac{\ell}{N_e}\sum_{(s,a)\not\in\mathfrak{S} } \frac{x(s,a)}{\hmu(s,a)}\\
	&\stackrel{(a)}{\leq}
		\sqrt{\frac{2\ell}{N_e}}\sum_{s,a}\sqrt{x(s,a)\cdot\frac{x(s,a)}{\hmu(s,a)}}
		+\frac{\ell}{N_e}\sum_{s,a} \frac{x(s,a)}{\hmu(s,a)}\\
    &\stackrel{(b)}{\leq} 
		\sqrt{\frac{2\ell}{N_e}}\sqrt{\sum_{s,a}x(s,a)\sum_{s,a}\frac{x(s,a)}{\hmu(s,a)}}
		+\frac{\ell}{N_e}\sum_{s,a} \frac{x(s,a)}{\hmu(s,a)}\\
		&\stackrel{(c)}{\leq} 
		\frac{2}{1-\gamma}\sqrt{\frac{2\cN\hphi\ell}{N_e}}
		+\frac{\cN\hphi\ell}{(1-\gamma)N_e},
\end{align*}
where the inequality (a) comes from the fact $\frac{\mu(s,a)}{\hmu(s,a)}\leq 2$; (b) is due to Cauchy's inequality, and (c) is due to $\sum_{s,a}\frac{x(s,a)}{\hmu(s,a)}\leq \frac{\cN\hphi}{1-\gamma}$ and $\sum_{s,a}x(s,a)\leq \frac{4}{1-\gamma}$. Therefore, because we set $\varsigma=\frac{\varphi(1-\gamma)^2\epsilon_e}{2\cN\hphi}$, and $N_e\geq\frac{128\cN\hphi\ell}{\varphi^2(1-\gamma)^4\epsilon_e^2}$, we have $\nrm{Wx-x}_1\leq \varphi(1-\gamma)\epsilon_e,\forall x\in\cX.$ 


\section{The magnitude and variance of the gradient estimators}
\label{appdx:grad-var}
\begin{proposition}\label{prop:grad-var}
	For any sample $\zeta\sim\rho_0\times\cD$, and any feasible solution $Z=[V;\lambda;x]$, the stochastic gradient estimators constructed in \eqref{eqn:def-gradient} are unbiased, and they satisfy the following bounds:\footnote{
		For vectors $u, v\in\RR^n$, we write $\|u\|_v^2:=\sum_{i=1}^n v_i u_i^2$ for simplicity.
	}
	$$
	\!\begin{cases}
		\!\EE\left[\hg_V(Z;\zeta)\right] = \nabla_V\cL_w(Z)\\
		\!\left\|\hg_V(Z;\zeta)\right\|\leq \mathcal{O}\big(\frac{\hphi}{1-\gamma}\big)\\
		\!\EE\Big[\!\!\left\|\hg_V\!(Z;\zeta)\right\|^{\!2}\!\Big]\!\!\leq\! \mathcal{O}\Big(\!\frac{\hphi}{(1-\gamma)^2}\!\Big)
	\end{cases}
	\!\!\!\!\!\begin{cases}
		\!\EE\left[\hg_\lambda(Z;\zeta)\right]= \nabla_\lambda\cL_w(Z)\\
		\!\left\|\hg_\lambda(Z;\zeta)\right\|_{\infty}\leq\mathcal{O}\big(\frac{\hphi}{1-\gamma}\big)\\ 
		\!\EE\Big[\!\!\left\|\hg_\lambda(Z;\zeta)\right\|_{\infty}^2\!\Big]\!\!\leq\!\mathcal{O}\Big(\!\frac{\hphi}{(1-\gamma)^2}\!\Big)
	\end{cases}
	\!\!\!\!\!\begin{cases}
		\!\EE\left[\hg_x(Z;\zeta)\right]= \nabla_{x}\cL_w(Z)\\
		\!\left\|\hg_x(Z;\zeta)\right\|_{x'}^2\!\!\leq\!\mathcal{O}\Big(\!\frac{\hphi^2\cN}{\varphi^{3}(1-\gamma)^{5}\epsilon_e}\!\Big)\\ 
		\!\EE\Big[\!\!\left\|\hg_x(Z;\zeta)\right\|^2_{x'}\!\Big]\!\!\leq\!\mathcal{O}\Big(\!\frac{\cN \hphi}{\varphi^2(1-\gamma)^3}\!\Big)
	\end{cases}
	$$
	where $x'\!\in\!\cX$ is an arbitrary vector.  
\end{proposition}

For any sample $\zeta=(s_0,s,a,s',r,\bu)\sim\rho_0\times\cD$, it is not hard to see that the estimators constructed in \eqref{eqn:def-gradient} are unbiased. Next, we provide the bound on the norm and variance of these estimators. 

For the estimator  $\widehat{g}_V(Z;\zeta):=\II_{s_0}+\frac{x(s,a)}{\hmu(s,a)}\left(\gamma\II_{s'}-\II_{s}\right)$, we have 
$$\left\|\hg_V(Z;\zeta)\right\|\leq1+\frac{x(s,a)}{\hmu(s,a)}(1+\gamma)\overset{(a)}{\leq} 1+\frac{2\hphi}{1-\gamma},$$
\begin{eqnarray*}
	\EE\left[\|g_V(Z;\zeta)\|^2\right]
	&\leq& \sum_{s,a} \mu(s,a)\cdot 2\left(1+4\cdot\frac{x(s,a)^2}{\hmu(s,a)^2}\right),\\
	&\leq& 2+8\cdot\sum_{s,a} \frac{\mu(s,a)}{\hmu(s,a)}\frac{x(s,a)}{\hmu(s,a)} x(s,a)
	\overset{(b)}{\leq} 2+\frac{64\hphi}{(1-\gamma)^2}.
\end{eqnarray*}
Here (a) is due to $x\in\cX$, which indicates that  $\frac{x(s,a)}{\hmu(s,a)}\leq\frac{\hphi}{1-\gamma}$ for all $(s,a)$. The inequality (b) is due to $\frac{\mu(s,a)}{\hmu(s,a)}\leq2$ established in Proposition \ref{prop:empirical-dis}, and $\sum_{s,a}x(s,a)\leq\frac{4}{1-\gamma}$.

Similarly, for the estimator $\widehat{g}_\lambda(Z;\zeta):=\frac{x(s,a)}{\hmu(s,a)}\bu^{\kappa}$, we have 
$$\|\widehat{g}_\lambda(Z;\zeta)\|_\infty\leq \Big\|\frac{x(s,a)}{\hmu(s,a)}\bu^{\kappa}\Big\|_\infty \overset{(a)}{\leq} \frac{x(s,a)}{\hmu(s,a)}(1+(1-\gamma)\kappa)\overset{(b)}{\leq}  \frac{2\hphi}{1-\gamma},$$
\begin{eqnarray*}
	\EE\left[\|g_\lambda(Z;\zeta)\|_{\infty}^2\right]
	&\leq& \sum_{s,a} \mu(s,a)\cdot 4 \frac{x(s,a)^2}{\hmu(s,a)^2} ,\\
	&=&4 \sum_{s,a}\frac{\mu(s,a)}{\hmu(s,a)}\frac{x(s,a)}{\hmu(s,a)}x(s,a)
	\overset{(c)}{\leq} \frac{32\hphi}{(1-\gamma)^2}.
\end{eqnarray*}
Here (a) follows from $\|\bu^\kappa\|_\infty\leq\|\bu\|_\infty+(1-\gamma)\kappa$, and (b) is due to $(1-\gamma)\kappa=5\varphi \epsilon(1-\gamma)<1$, and (c) is similar to the argument of the bound on $\EE\left[\|g_V(Z;\zeta)\|^2\right]$.

Finally, for the estimator $\widehat{g}_x(Z;\zeta) :=\frac{r+\gamma V(s)-V(s')+\iprod{\bu^{\kappa}}{\lambda}}{\hat\mu(s,a)}\II_{s,a}$, we have 
\begin{eqnarray*}
	\nrm{\widehat{g}_x(Z;\zeta)}_{x'}^2 
	&=& \frac{x'(s,a)}{\hmu(s,a)^2}\cdot\abs{r+\gamma V(s)-V(s')+\iprod{\bu^{\kappa}}{\lambda}}^2\\
	& \leq & \frac{x'(s,a)}{\hmu(s,a)^2}\left(1+\frac{16}{1-\gamma}(1+\frac{2}{\varphi})+\frac{8(1+\kappa)}{\varphi}\right)^2\\
	& \leq & \frac{\hphi}{(1-\gamma)\varsigma}\cdot \frac{64^2}{\varphi^2(1-\gamma)^2}\\
	& = &\bigO{
		\frac{\hphi^2\cN}{\varphi^3(1-\gamma)^{5}\epsilon_e}
	},
\end{eqnarray*}

and as long as $\zeta$ is independent of $x'\in\cX$,
\begin{eqnarray*}
	\EE\left[\|\hg_x(Z;\zeta)\|_{x'}^2\right]
	&\leq& \sum_{s,a}  \frac{\mu(s,a)x'(s,a)}{\hmu(s,a)^2} \cdot\left(1+\frac{16}{1-\gamma}(1+\frac{2}{\varphi})+\frac{8(1+\kappa)}{\varphi}\right)^2\\
	&\leq& \sum_{s,a} \frac{\mu(s,a)}{\hmu(s,a)}\frac{x'(s,a)}{\hmu(s,a)}\cdot\frac{64^2}{\varphi^2(1-\gamma)^2}\\
	&\leq& \mathcal{O}\left(\frac{\cN \hphi}{\varphi^2(1-\gamma)^3}\right).
\end{eqnarray*}
This completes the proof of Proposition \ref{prop:grad-var}.

\paragraph{A few notational definitions.} We should notice that the above bounds on the gradient estimators are notationally very complicated.  Therefore, Let us conveniently write the above bounds as 
$$\begin{cases}  
	\nrm{g_V(Z^t;\zeta_t)}\leq M_V,\\
	\nrm{g_\lambda(Z^t;\zeta_t)}_{\infty}\leq M_\lambda,\\
	\nrm{g_x(Z^t;\zeta_t)}_{x'}\leq M_{x} \sqrt{D_{x,1}}, 
\end{cases}\quad\mbox{and}\qquad\,\,
\begin{cases} 
	\EE\left[\|g_V(Z;\zeta)\|^2\right] \leq \sigma_V^2,\\
	\EE\left[\|g_\lambda(Z;\zeta)\|^2_{\infty}\right] \leq \sigma_\lambda^2,\\
	\EE\left[\|g_x(Z^t;\zeta_t)\|_{x'}^2\right] \leq \sigma_x^2D_{x,1}, 
\end{cases}$$
where the constants  $\sigma_{V},\sigma_\lambda,\sigma_x$ and $M_{V},M_\lambda,M_x$ are 
\begin{equation}\label{eqn:def-algB-variance} 
		\sigma_V^2 =\ThO{\frac{\hphi}{(1-\gamma)^2}},
		\qquad \sigma_{\lambda}^2=\ThO{\frac{\hphi}{(1-\gamma)^2}}, \qquad
		\sigma_x^2 =\ThO{\frac{\cN \hphi}{\varphi^2(1-\gamma)^2}}, 
\end{equation}
\begin{equation}\label{eqn:def-algB-norm} 
		M_V =\ThO{\frac{\hphi}{1-\gamma}},
		\qquad M_\lambda=\ThO{\frac{\hphi}{1-\gamma}}, \qquad
		M_x = \ThO{\frac{\hphi}{\varphi(1-\gamma)^2}\sqrt{\frac{\cN}{\varphi\epsilon_e}}}, 
\end{equation} 
and $D_{x,1}$ is a suitable upper bound on the diameter of $\cX$, namely we choose $D_{x,1}=\ThO{\frac{1}{1-\gamma}}$ such that $D_{x,1}\geq \sup_{x,x'\in\cX}\nrm{x'-x}_1$.
Similarly, we define $D_{\lambda,1} := \sup_{\lambda,\lambda'\in\Lambda}\nrm{\lambda'-\lambda}_1=\ThO{\frac{1}{\varphi}}$. 

Furthermore, we also introduce the diameters of the feasible domains w.r.t. the initial solution $V^1, \lambda^1, x^1$. Recall that the initial point of \cref{algo:algB} is chosen as
\[\begin{aligned}
	V^1=\mathbf{0}\in\cV,\qquad
	\lambda^1=\frac{\mathbf{1}}{\varphi I}\in\Lambda,\qquad
	x^{1}=\frac{c_x\hat{\mu}}{1-\gamma}\in\cX,
\end{aligned}\]
where $c_x=\frac{\cN}{\nS\nA}$ ensures that $x^1\in\cX$. Then, we can take $D_V, D_\lambda, D_x$ as
\begin{eqnarray*}
	&&D^2_V:=\sup_{V'\in\cV}\sqr{V'-V^1}=\ThO{\frac{\nS}{\varphi^2(1-\gamma)^2}},\\
	&&D_{\lambda}:=\sup_{\lambda'\in\Lambda}\KL{\lambda'}{\lambda^1}=\ThO{\frac{\log I}{\varphi}},\\
	&&D_{x}\geq \sup_{x'\in\cX}\KL{x'}{x^1}, \qquad D_{x}=\ThO{\frac{\log\hphi}{1-\gamma}}.
\end{eqnarray*}
\begin{remark}
	It is worth noting that, \cref{prop:grad-var} directly implies $\cond{\sqr{\hg_V(Z^t;\zeta_t)}}{Z^t}\leq \sigma_V^2$, $\cond{\sqr{\hg_\lambda(Z^t;\zeta_t)}_{\infty}}{Z^t}\leq \sigma_{\lambda}^2$ and $\cond{\sqr{\hg_x(Z^t;\zeta_t)}_{x^t}}{Z^t}\leq D_{x,1}\sigma_{x}^2$ for each step $t$.
\end{remark}

\begin{remark}\label{rmk:large-var}
	The reason why we bound the term $\nrm{g_x(Z^t;\zeta_t)}_{x^{t}}$ instead of $\nrm{g_x(Z^t;\zeta_t)}_{\infty}$ is that,
	\begin{equation*}
	\begin{aligned}
	  \nrm{g_x(Z^t;\zeta_t)}_{\infty}\leqsim \frac{1}{\varphi(1-\gamma)}\frac{1}{\hmu(s_t,a_t)}\leq \frac{1}{\varphi(1-\gamma)\varsigma}.
	\end{aligned}
	\end{equation*}
	Thus, we have to take $M_{x,\infty}=\ThO{\frac{1}{\varphi(1-\gamma)\varsigma}}$ to ensure a uniformly bound as
	\begin{equation}
	\begin{aligned}
	  \nrm{g_x(Z^t;\zeta_t)}_{\infty}\leq M_{x,\infty}.
	\end{aligned}
	\end{equation}
\end{remark}

\section{Proof of Theorem \ref{thm:algB-gap}}\label{appdx:ThmGap}


To bound $\Gap(\ox)$, let us denote 
\begin{equation}\label{eqn:def-aux-var}
(V',\lambda')=\argmin_{V\in\cV, \lambda\in\Lambda}\cL_{w}(V,\lambda,\ox),\qquad
x'=\argmax_{x\in\cX} \min_{V\in\cV,\lambda\in\Lambda} \cL_{w}(V,\lambda,x),
\end{equation}
and denote $Z'=[V';\lambda';x']$. It is worth mentioning that $V',\lambda'$ are random variables that depend on $\ox$ while $x'$ is deterministic. For the ease of notation, we define
$$\mathcal{G}(Z) := \begin{bmatrix}+\nabla_V\cL_w(Z)\\ + \nabla_\lambda\cL_w(Z)\\-\nabla_{x}\cL_w(Z)\end{bmatrix}\qquad\mbox{and}\qquad\widehat{g}(Z;\zeta):=\begin{bmatrix}+\hg_V(Z;\zeta)\\+\hg_\lambda(Z;\zeta)\\-\hg_x(Z;\zeta)\end{bmatrix}.$$
Then, by the definition of $V',\lambda',x'$ and the bi-linearity of $\cL_w(\cdot)$, we have
\begin{eqnarray}\label{eqn:proof-decomp}
	\Gap(\ox) & = & \max_{x\in\cX}\min_{V\in\cV,\lambda\in\Lambda}\cL_{w}(V,\lambda,x)-\min_{V\in\cV,\lambda\in\Lambda}\cL_{w}(V,\lambda,\ox)\nonumber\\
	& = & \cL_{w}(\oV,\ol,x')-\cL_{w}(V',\lambda',\ox)\nonumber\\
	& = & \frac1T\sum_{t=1}^{T}\Big( \cL_{w}(V^t,\lambda^t,x')-\cL_{w}(V',\lambda',x^t)\Big) \\
	&=& \frac1T\sum_{t=1}^{T} \iprod{\mathcal{G}(Z^t)}{Z^t-Z'}\nonumber\\
	&=&
	\underbrace{
		\frac1T\sum_{t=1}^{T} \iprod{\hg(Z^t;\zeta_t)}{Z^t-Z'}
	}_{S_1}
	+\underbrace{
		\frac1T\sum_{t=1}^{T} \iprod{\mathcal{G}(Z^t)-\hg(Z^t;\zeta_t)}{Z^t-Z'}
	}_{S_2}.\nonumber
\end{eqnarray}
Then with the estimations in \cref{appdx:grad-var}, the $S_1$ and $S_2$ terms can be bounded by 
\begin{eqnarray}
	\label{eqn:S1}
	S_1
	&\leqsim &
	\frac{\alpha_V D^2_V+\alpha_\lambda D_{\lambda}+\alpha_x D_x}{\eta T}
	+\eta\left(\frac{\sigma_V^2}{\alpha_V}+\frac{\sigma_\lambda^2 D_{\lambda,1}}{\alpha_\lambda}+\frac{\sigma_x^2 D_{x,1}}{\alpha_x}\right)\\
	&& +\frac{\eta\clog}{T}\left(\frac{M^2_V}{\alpha_V}+\frac{M^2_\lambda D_{\lambda,1}}{\alpha_\lambda}+\frac{M_x^2 D_{x,1}}{\alpha_x}\right) \nonumber
\end{eqnarray}
and 
\begin{equation} \label{eqn:S2}
	S_2
	\leqsim \left(D_V\sigma_V+D_{\lambda,1}\sigma_\lambda+D_{x,1}\sigma_x\right)\sqrt{\frac{\clog}{T}}
	+\left(D_VM_V+D_{\lambda,1}M_\lambda+D_{x,1}M_x\right)\frac{\clog}{T} 
\end{equation}
\highprobs{10} respectively, as long as the stepsize satisfies
\begin{equation}\label{eqn:def-algB-eta-upper}
	\begin{aligned}
		\eta\leq \frac{1}{2}\min\left(\frac{\alpha_\lambda}{M_\lambda}, \frac{\alpha_x}{M_{x,\infty}}\right).
	\end{aligned}
\end{equation} 
Due to the sophistication of the proof, we move the analysis of \eqref{eqn:S1} and \eqref{eqn:S2} to Appendix \ref{appdx:S1} and \ref{appdx:S2} respectively. 

Finally, combining the inequalities \eqref{eqn:proof-decomp}, \eqref{eqn:S1} and \eqref{eqn:S2}, and requiring that \eqref{eqn:def-algB-eta-upper} holds true for
$\eta = 1/\sqrt{T}$, we have \highprobs{3}
	\begin{align*}
		\Gap(\ox)
		\leqsim& 
		\frac{\alpha_V D^2_V+\alpha_\lambda D_{\lambda}+\alpha_x D_x}{\eta T}
		+\eta\left(\frac{\sigma_V^2}{\alpha_V}+\frac{\sigma_\lambda^2 D_{\lambda,1}}{\alpha_\lambda}+\frac{\sigma_x^2 D_{x,1}}{\alpha_x}\right)\\
		& +\sqrt{\frac{\clog}{T}}\left(D_V\sigma_V+D_{\lambda,1}\sigma_\lambda+D_{x,1}\sigma_x\right)\\
		& +\frac{\clog}{T}\left(D_VM_V+D_{\lambda,1}M_\lambda+D_{x,1}M_x\right)\\
		& +\frac{\eta\clog}{T}\left(\frac{M^2_V}{\alpha_V}+\frac{M^2_\lambda D_{\lambda,1}}{\alpha_\lambda}+\frac{M_x^2 D_{x,1}}{\alpha_x}\right).
	\end{align*}
Note that the normalizing constants are chosen as $\alpha_V=\varphi\sqrt{\frac{\hphi}{\nS}}=\ThO{\frac{\sigma_V}{D_V}}$, $\alpha_\lambda=\frac{1}{1-\gamma}\sqrt{\frac{\hphi}{\log I}}=\ThO{\sigma_\lambda\sqrt{\frac{D_{\lambda,1}}{D_{\lambda}}}}$, $\alpha_x=\frac{1}{\varphi(1-\gamma)}\sqrt{\frac{\cN\hphi}{\log\hphi}}=\ThO{\sigma_x\sqrt{\frac{D_{x,1}}{D_{x}}}}$. Then \eqref{eqn:def-algB-eta-upper} holds true for the stepsize $\eta=\frac{1}{\sqrt{T}}$ with $T\geqsim \frac{\cN\hphi\clog}{\varphi^2(1-\gamma)^4\epsilon_e^2}$, and we can plug in the values of the constants $\alpha,D,M$, then \highprobs{3}\ it holds that 
\[\begin{aligned}
	\Gap(\ox)
	&\leqsim
	\sqrt{\frac{\cN\hphi\clog}{\varphi^2(1-\gamma)^4T}}\left(1+\frac{\clog}{T}\cdot\frac{\hphi}{\varphi(1-\gamma)^2\epsilon_e}\right)\leqsim
	\sqrt{\frac{\cN\hphi\clog}{\varphi^2(1-\gamma)^4T}}.
\end{aligned}\]
Choosing $c_o$ to ensure $\Gap(\ox)\leq\frac{\epsilon}{2}$ completes the proof of \cref{thm:algB-gap}.

\subsection{A few supporting lemmas}
For the proof in the following parts of Appendix \ref{appdx:ThmGap}, we introduce a few supporting lemmas.
\begin{lemma}\label{prop:mirror-seq}
	Let $\{Y^k\}_{k=1}^T$ be generated by $Y^{k+1} \!=\! \mathop{\mathrm{argmin}}_{Y\in\cY}\left(\eta\iprod{Y-Y^k}{g^k}+\KL{Y}{Y^k}\right)$, where $\eta\leq \frac{1}{2\max_{k}\nrm{g_k}_{\infty}}$ and $\cY$ is some convex set. Then for all $Y'\in\cY$, it holds that
	\[\begin{aligned}
		\frac{1}{T}\sum_{t=1}^T\iprod{Y^{t}-Y'}{g^t}
		&\leq \frac{\KL{Y'}{Y^1}}{\eta T}+\frac{4\eta}{T}\sum_{t=1}^{T}\nrm{g^k}^2_{Y^k}\\
		&\leq \frac{\KL{Y'}{Y^1}}{\eta T}+\frac{4\eta D_{Y,1}}{T}\sum_{t=1}^{T}\nrm{g^k}^2_{\infty}.
	\end{aligned}\]
	where $D_{Y,1}$ can be any upper bound of $\max_{Y\in\cY}\nrm{Y}_1$.
\end{lemma}
The proof of Lemma \ref{prop:mirror-seq} is presented in Appendix \ref{appdx:mirror-seq}. 
\begin{lemma}\label{prop:l2-seq}
	Let $\{Y^k\}_{k=1}^T$ be generated by $Y^{k+1}\!=\!\Proj_{\cY}\left(Y^k-\eta g^k\right)$, where $\cY$ is some convex set.
	Then for all $Y'\in\cY$, it holds that
	\[
	\frac{1}{T}\sum_{t=1}^T\iprod{Y^{t}-Y'}{g^t}\leq \frac{\sqr{Y'-Y^1}}{2\eta T}+\frac{\eta}{T}\sum_{t=1}^{T}\nrm{g^k}^2.
	\]
\end{lemma}
The proof of Lemma \ref{prop:l2-seq} is similar but a lot simpler than that of Lemma \ref{prop:mirror-seq},  and is hence omitted. 
\begin{proposition}[Corollary of Bernstein's inequality]\label{prop:Bernstein-var}
	For a sequence of random variables $X_1,\cdots,X_N$ adapted to $(\cF_n)$, and $\cond{|X_i|}{\cF_{i-1}}\leq c$, $|X_i|\leq M$, we have \highprobdemo,
	\[\begin{aligned}
		\left|\frac{1}{N}\sum_{i=1}^{N} X_i\right|\leq 2c+3M\frac{\log(2/\delta)}{N}.
	\end{aligned}\]
\end{proposition}
\begin{proof}
	Notice that $\cond{X_i^2}{\cF_{i-1}}\leq cM$, and by Bernstein's inequality
	\begin{align*}
		&\left|\frac{1}{N}\sum_{i=1}^{N} \left(X_i-\cond{X_i}{\cF_{i-1}}\right)\right|\leq \sqrt{\frac{2cM\log(2/\delta)}{N}}+2M\frac{\log(2/\delta)}{N},\\
		\Rightarrow& \left|\sum_{i=1}^{N} X_i\right|\leq cN+\sqrt{2cMN\log(2/\delta)}+2M\log(2/\delta),
	\end{align*}
	holds \highprobdemo. By the AM-GM inequality, $\sqrt{2cMN\log(2/\delta)}\leq \frac{1}{2}cN+M\log(2/\delta)$, which completes the proof.
\end{proof}

\subsection{Bounding the term $S_1$}
\label{appdx:S1}
First, by definition of $\hg(\cdot)$, we have
\[\begin{aligned}
	S_{1}= 
	\underbrace{
		\frac1T\sum_{t=1}^{T} \iprod{\hg_V(Z^t;\zeta_t)}{V^t-V'}
	}_{S_{1,V}}
	+\underbrace{
		\frac1T\sum_{t=1}^{T} \iprod{\hg_\lambda(Z^t;\zeta_t)}{\lambda^t-\lambda'}
	}_{S_{1,\lambda}} +\underbrace{
		\frac1T\sum_{t=1}^{T} \iprod{-\hg_x(Z^t;\zeta_t)}{x^t-x'}
	}_{S_{1,x}}.
\end{aligned}\]

Applying Lemma \ref{prop:l2-seq} with $Y^t=V^t$, $g^t = \hg_V(Z^t;\zeta_t)$ yields
\[
S_{1,V}\leq \frac{\alpha_V\sqr{V'-V^{1}}}{2\eta T}+\frac{\eta}{\alpha_V T}\sum_{t=1}^T \sqr{\hg_V(Z^t;\zeta_t)}.
\]
Applying Lemma \ref{prop:mirror-seq} with $Y^t=\lambda^t$, $g^t = \hg_\lambda(Z^t;\zeta_t)$, we have
\[\begin{aligned}
  S_{1,\lambda}
  &\leq \frac{\alpha_\lambda \KL{\lambda'}{\lambda^1}}{\eta T}+\frac{4\eta D_{\lambda,1}}{\alpha_\lambda T}\sum_{t=1}^T \sqr{\hg_\lambda(Z^t;\zeta_t)}_{\infty},
\end{aligned}\]
as long as $\nrm{\hg_\lambda(Z^t;\zeta_t)}_{\infty}\leq \frac{\alpha_\lambda}{2\eta}$ holds for all $t$, and $\frac{1}{\eta}\geq \frac{2M_{\lambda}}{\alpha_\lambda}$ suffices. 

Finally, applying Lemma \ref{prop:mirror-seq} with $Y^t=x^t$, $g^t = -\hg_x(Z^t;\zeta_t)$, we obtain
\[
S_{1,x}\leq \frac{\alpha_x \KL{x'}{x^1}}{\eta T}+\frac{4\eta}{\alpha_x T}\sum_{t=1}^T \sqr{\hg_x(Z^t;\zeta_t)}_{x^t},
\]
as long as $\nrm{\hg_x(Z^t;\zeta_t)}_{\infty}\leq \frac{\alpha_x}{2\eta}$ holds for all $t$, and $\frac{1}{\eta}\geq \frac{2M_{x,\infty}}{\alpha_x}$ suffices.

Combining all the estimations above, as long as the stepsize $\eta$ satisfies \eqref{eqn:def-algB-eta-upper}, we have
\begin{equation}\label{eqn:duality-proof-1}
\begin{aligned}
  S_1
  \leq &
  \frac{\alpha_V\sqr{V'-V^1}+\alpha_\lambda \KL{\lambda'}{\lambda^1}+\alpha_x \KL{x'}{x^1}}{\eta T}\\
  & +\frac{4\eta}{T}\sum_{t=1}^{T}\left(\frac{\sqr{\hg_V(Z^t;\zeta_t)}}{\alpha_V}+\frac{D_{\lambda,1}\sqr{\hg_\lambda(Z^t;\zeta_t)}_{\infty}}{\alpha_\lambda}+\frac{\sqr{\hg_x(Z^t;\zeta_t)}_{x^t}}{\alpha_x}\right).
\end{aligned}
\end{equation}

For the second term of $S_1$ in \eqref{eqn:duality-proof-1}, with the variance and magnitude bounds provided in Proposition \ref{prop:grad-var}, applying \cref{prop:Bernstein-var} to the sequences $\{\sqr{\hg_V(Z^t;\zeta_t)}\}_{t=1}^T$, $\{\sqr{\hg_\lambda(Z^t;\zeta_t)}_{\infty}\}_{t=1}^T$ and $\{\sqr{\hg_x(Z^t;\zeta_t)}_{x^t}\}_{t=1}^T$ proves the inequality \eqref{eqn:S1} \highprobs{10}.

\subsection{Bounding the term $S_2$} 
\label{appdx:S2}
For the term $S_2$, we introduce the martingale difference sequences
\[\begin{aligned}
  \Delta^t_V&:=\hg_V(Z^t;\zeta_t)-\nabla_V \cL_{w}(V^t,\lambda^t,x^t),\\
  \Delta^t_\lambda&:=\hg_\lambda(Z^t;\zeta_t)-\nabla_\lambda \cL_{w}(V^t,\lambda^t,x^t),\\
  \Delta^t_x&:=\hg_x(Z^t;\zeta_t)-\nabla_x \cL_{w}(V^t,\lambda^t,x^t),
\end{aligned}\]
Then $S_2$ can be decomposed as
\[\begin{aligned}
  S_2=&\underbrace{
  \frac1T\sum_{t=1}^{T} \left(\iprod{\Delta_V^t}{V'-V^1}+\iprod{\Delta_\lambda^t}{\lambda'-\lambda^1}\right)
}_{S_{2,c}}\\
&+\underbrace{
  \frac1T\sum_{t=1}^{T} \left(\iprod{\Delta_V^t}{V^1-V^t}+\iprod{\Delta_\lambda^t}{\lambda^1-\lambda^t}+\iprod{-\Delta_x^t}{x'-x^t}\right)
}_{S_{2,m}}.
\end{aligned}\]
Note that the martingale part $S_{2,m}$ has expectation zero. However, for the first part, $V'$ and $\lambda'$ are random variables depending on $\bar x$. Thus the correlated part $S_{2,c}$ may not have zero mean.

\paragraph{Bounding the term $S_{2,c}$} For the correlated part $S_{2,c}$, the sequence $\Delta^t_V$ and $\Delta^t_\lambda$ are (vector-valued) martingale difference sequences, and hence
\[\begin{aligned}
  S_{2,c}
  &= \iprod{\frac1T\sum_{t=1}^{T}\Delta_V^t}{V'-V^1}+\iprod{\frac1T\sum_{t=1}^{T}\Delta_\lambda^t}{\lambda'-\lambda^1}\\
  &\leq \nrm{V'-V^1}\cdot\frac{1}{T}\nrm{\sum_{t=1}^T \Delta^t_V}
  +\nrm{\lambda'-\lambda^1}_{1}\cdot\frac{1}{T}\nrm{\sum_{t=1}^T \Delta^t_\lambda}_{\infty}.
\end{aligned}\]
The quantity $\nrm{\sum_{t=1}^T \Delta^t_V}$ and $\nrm{\sum_{t=1}^T \Delta^t_\lambda}_{\infty}$ both can be bounded by applying \cref{lemma:concen-vec}. More specifically, \highprobs{20}, it holds that
\[\begin{aligned}
  &\nrm{\frac{1}{T}\sum_{t=1}^T \Delta^t_V}\leqsim \sigma_V\sqrt{\frac{\log(\nS/\delta)}{T}}+M_V\frac{\log(\nS/\delta)}{T},\\
  &\nrm{\frac{1}{T}\sum_{t=1}^T \Delta^t_\lambda}_{\infty}\leqsim \sigma_\lambda \sqrt{\frac{\log(I/\delta)}{T}}+M_\lambda\frac{\log(I/\delta)}{T}.
\end{aligned}\]
Therefore, we have
\[\begin{aligned}
  S_{2,c}
  \leqsim 
    \left(D_V\sigma_V+D_{\lambda,1}\sigma_{\lambda}\right)\sqrt{\frac{\clog}{T}}
    +\left(D_VM_V+D_{\lambda,1}M_{\lambda}\right)\frac{\clog}{T}.
\end{aligned}\]

\paragraph{Bounding the term $S_{2,m}$} In order to bound the martingale part $S_{2,m}$, we have to consider martingales difference sequences\footnote{
	They are martingale difference sequences w.r.t. the filtration $(\cF_t)$ defined by $\cF_t=\sigma(\zeta_1,\cdots,\zeta_{t-1})$.
} $\odelta_V^t:=\iprod{\Delta_V^t}{V^1-V^t},$
$\odelta^t_\lambda:=\iprod{\Delta_\lambda^t}{\lambda^1-\lambda^t}$, 
$\odelta^t_x:=\iprod{\Delta_x^t}{x^t-x'}.$ 
We estimate the variance and magnitude as
\[\begin{aligned}
  & \abs{\odelta_V^t}\leq 2D_V M_V,
  &&\cond{\left(\odelta_V^t\right)^2}{\cF_t}\leq \cond{\sqr{V^1-V'}\sqr{\Delta_V^t}}{\cF_t}\leq D_V^2\sigma_V^2, \\
  & \abs{\odelta_\lambda^t}\leq 2D_{\lambda,1}M_\lambda,
  &&\cond{\left(\odelta_\lambda^t\right)^2}{\cF_t}\leq \cond{\sqr{\lambda^1-\lambda^t}_{1}\sqr{\Delta_\lambda^t}_{\infty}}{\cF_t}\leq D_{\lambda,1}^2\sigma_\lambda^2, \\
  & \abs{\odelta_x^t}\leq 2D_{x,1}M_x,
  &&\cond{\left(\odelta_x^t\right)^2}{\cF_t}\leq \cond{\sqr{\frac{x'-x^t}{\sqrt{x'+x^t}}}\sqr{\Delta_x^t}_{x'+x^t}}{\cF_t}\leq 2D_{x,1}^2\sigma_x^2.
\end{aligned}\]
Thus, by the Bernstein's Inequality, the following holds \highprobs{20}:
\[\begin{aligned}
  \frac1T\sum_{t=1}^{T} \odelta^t_V
  &\leqsim D_V\sigma_V\sqrt{\frac{\dlog}{T}}+\frac{D_VM_V\dlog}{T},\\
  \frac1T\sum_{t=1}^{T} \odelta^t_\lambda
  &\leqsim D_{\lambda,1}\sigma_\lambda\sqrt{\frac{\dlog}{T}}+\frac{D_{\lambda,1}M_\lambda\dlog}{T},\\
  \frac1T\sum_{t=1}^{T} \odelta^t_x
  &\leqsim D_{x,1}\sigma_x\sqrt{\frac{\dlog}{T}}+\frac{D_{x,1}M_x\dlog}{T}.
\end{aligned}\]
Therefore, \highprobs{20},
\[\begin{aligned}
    S_{2,m}& \leqsim
    \left(D_V\sigma_V+D_{\lambda,1}\sigma_\lambda+D_{x,1}\sigma_x\right)\sqrt{\frac{\clog}{T}}
    +\left(D_VM_V+D_{\lambda,1}M_\lambda+D_{x,1}M_x\right)\frac{\clog}{T}.
\end{aligned}\]

\paragraph{Bounding the term $S_2$} Finally, combining the bounds on $S_{2,m}$ and $S_{2,c}$ proves the inequality \eqref{eqn:S2}.

\subsection{Basics of mirror descent}
\label{appdx:mirror-seq}



Before we provide the proof of \cref{prop:mirror-seq}, we state a basic property of the mirror descent (see e.g. \cite{chen1993convergence}).
\begin{lemma}\label{lemma:mirror-basic}
	Under the same assumption in \cref{prop:mirror-seq}, it holds that for any $Y'\in\cY$, 
	\[
	\eta\iprod{Y^{k+1}-Y'}{g^k}\leq \KL{Y'}{Y^k}-\KL{Y'}{Y^{k+1}}-\KL{Y^{k+1}}{Y^{k}}.
	\]
	In particular,
	\[
	\begin{aligned}
		\KL{Y^k}{Y^{k+1}}+\KL{Y^{k+1}}{Y^{k}}\leq \eta\iprod{Y^{k}-Y^{k+1}}{g^k}.
	\end{aligned}
	\]
\end{lemma}
\begin{proof}[Proof of \cref{prop:mirror-seq}]
	By the fact that $(x-y)\log\frac{x}{y}\geq \frac{(x-y)^2}{\max(x,y)}$, we have
	\[\begin{aligned}
		\KL{Y^k}{Y^{k+1}}+\KL{Y^{k+1}}{Y^{k}}=
		\iprod{Y^{k}-Y^{k+1}}{\log Y^{k}-\log Y^{k+1}}
		\geq \sum_i \frac{(Y^k_i-Y_i^{k+1})^2}{\max(Y_i^k,Y_i^{k+1})}.
	\end{aligned}\]
	Together with \cref{lemma:mirror-basic}, the estimation above yields
	\[
	\begin{aligned}
		\nrm{\frac{Y^k-Y^{k+1}}{\sqrt{Y^{k}+Y^{k+1}}}}^2
		\leq \KL{Y^k}{Y^{k+1}}+\KL{Y^{k+1}}{Y^k}
		\leq \eta\iprod{Y^{k}-Y^{k+1}}{g^k}.
	\end{aligned}
	\]
	By Cauchy inequality, $\iprod{Y^{k}-Y^{k+1}}{g^k} \leq \nrm{\frac{Y^k-Y^{k+1}}{\sqrt{Y^{k}+Y^{k+1}}}}\nrm{g^k\sqrt{Y^{k}+Y^{k+1}}}$, and hence 
	\[\begin{aligned}
		&\nrm{\frac{Y^k-Y^{k+1}}{\sqrt{Y^{k}+Y^{k+1}}}}\leq \eta\nrm{g^k\sqrt{Y^{k}+Y^{k+1}}}=\eta\nrm{g^k}_{Y^k+Y^{k+1}},\\
		&\iprod{Y^{k}-Y^{k+1}}{g^k} \leq \nrm{\frac{Y^k-Y^{k+1}}{\sqrt{Y^{k}+Y^{k+1}}}}\nrm{g^k\sqrt{Y^{k}+Y^{k+1}}}\leq \eta \nrm{g^k}_{Y^k+Y^{k+1}}^2.
	\end{aligned}\]
	To further bound $\nrm{g^k}_{Y^k+Y^{k+1}}$ in terms of $\nrm{g^k}_{Y^k}$, we estimate it as
	\[
	\begin{aligned}
		\nrm{g^k}^2_{Y^k+Y^{k+1}}
		&=\sum_{i} (Y_i^k+Y_i^{k+1})(g^k_i)^2\\
		&\leq 2\nrm{g^k}_{Y^k}^2+\sum_i \abs{Y_i^{k+1}-Y_i^k}(g^k_i)^2\\
		&\leq 2\nrm{g^k}_{Y^k}^2+\max_{i}\abs{g^k_i}\nrm{\frac{Y^k-Y^{k+1}}{\sqrt{Y^{k}+Y^{k+1}}}}\nrm{g^k}_{Y^k+Y^{k+1}}\\
		&\leq 2\nrm{g^k}_{Y^k}^2+\eta\nrm{g^k}_{\infty}\nrm{g^k}^2_{Y^k+Y^{k+1}}.
	\end{aligned}
	\]
	Thus, as long as $\eta\leq \frac{1}{2\nrm{g^k}_{\infty}}$, it holds that $\nrm{g^k}_{Y^k+Y^{k+1}}\leq 2\nrm{g^k}_{Y^k}$. Therefore, for all $Y'\in\cY$, 
	\[\begin{aligned}
		\iprod{Y^{k}-Y'}{g^k}
		&\leq \frac{1}{\eta}\left[\KL{Y'}{Y^k}-\KL{Y'}{Y^{k+1}}\right]
		+ \iprod{Y^{k}-Y^{k+1}}{g^k}\\
		&\leq \frac{1}{\eta}\left[\KL{Y'}{Y^k}-\KL{Y'}{Y^{k+1}}\right]
		+ 4\eta \nrm{g^k}^2_{Y^k}.
	\end{aligned}\]
	Summing over $k=1,\cdots,T$ completes the proof.
\end{proof}

\begin{corollary}\label{cor:mirror-seq-diff}
	Under the same assumption in \cref{prop:mirror-seq}, it holds that for each $k$,
	\[\begin{aligned}
		&\nrm{\frac{Y^k-Y^{k+1}}{\sqrt{Y^{k}+Y^{k+1}}}}\leq 2\eta\nrm{g^k}_{Y^k},\\
		&\nrm{Y^{k+1}-Y^{k}}_1\leq 4\eta \sqrt{D_{Y,1}} \nrm{g^k}_{Y^k}\leq 4\eta D_{Y,1} \nrm{g^k}_{\infty}.
	\end{aligned}\]
\end{corollary}
\begin{proof}
	From the proof of \cref{prop:mirror-seq} above, we see
	\[\begin{aligned}
		J(Y^{k},Y^{k+1})=\KL{Y^k}{Y^{k+1}}+\KL{Y^{k+1}}{Y^{k}}\leq \eta\iprod{Y^{k}-Y^{k+1}}{g^k}\leq 4 \eta^2 \nrm{g^k}_{Y^k}^2.
	  \end{aligned}\]
	Then by \cref{lemma:pinsker} we have
	  \begin{align*}
	  \nrm{Y^{k+1}-Y^{k}}_1\leq \left(\sqrt{\nrm{Y^k}_1}+\sqrt{\nrm{Y^{k+1}}_1}\right)\sqrt{J(Y^{k},Y^{k+1})}\leq 4\eta \sqrt{D_{Y,1}} \nrm{g^k}_{Y^k}. &\qedhere
	  \end{align*}
\end{proof}

\begin{lemma}[Generalized Pinsker's Inequality]\label{lemma:pinsker}
	For $y,y'\in\RR_{>0}^n$, we consider the generalized Jeffery divergence between them:
	\[\begin{aligned}
	J(y,y'):=\KL{y}{y'}+\KL{y'}{y}=\sum_{i} (y_i-y_i')\log\frac{y_i}{y_i'}.
	\end{aligned}\]
	Then it holds that
	\[\begin{aligned}
	\nrm{y-y'}_1\leq \left(\sqrt{\nrm{y}_1}+\sqrt{\nrm{y'}_1}\right)\sqrt{J(y,y')}.
	\end{aligned}\]
\end{lemma}

\begin{proof}
	Denote $J=J(y,y')$, $Y=\nrm{y}_1$, $Y'=\nrm{y'}_1$. We consider two (normalized) distributions $\oy=\frac{y}{Y}$ and $\oy'=\frac{y'}{Y'}$, then
	\begin{align*}
	  J(y,y')
	  &=\sum_{i} (y_i-y_i')\log\frac{y_i}{y_i'}\\
	  &=\sum_{i} \left(Y\oy_i-Y'\oy_i'\right)\left(\log\frac{\oy_i}{\oy_i'}+\log\frac{Y}{Y'}\right)\\
	  &=Y\KL{\oy}{\oy'}+Y'\KL{\oy'}{\oy}+(Y-Y')\log\frac{Y}{Y'}\\
	  &\geq (Y+Y')\cdot\frac12\nrm{\oy-\oy'}_1^2+\frac{\abs{Y-Y'}^2}{\max(Y,Y')},
	\end{align*}
	where the last inequality is due to Pinsker's inequality and the fact $(x-y)\log\frac{x}{y}\geq \frac{(x-y)^2}{\max(x,y)}$. Therefore, w.l.o.g. $Y<Y'$, then $\abs{Y-Y'}\leq \sqrt{Y'J}$, and
	\[\begin{aligned}
	  &\sqrt{\frac{2J}{Y+Y'}}\geq \nrm{\oy-\oy'}_1= \nrm{\frac{y}{Y}-\frac{y'}{Y'}}_1
	  =\nrm{\frac{y-y'}{Y}+\frac{y'}{Y'}\left(\frac{Y'}{Y}-1\right)}_1.
	\end{aligned}\]
	Hence, we have
	\begin{align*}
		\nrm{y-y'}_1
		&\leq \nrm{\frac{y'}{Y'}\left(Y'-Y\right)}_1+Y\sqrt{\frac{J}{Y+Y'}}\\
		&=\abs{Y'-Y}+Y\sqrt{\frac{J}{Y+Y'}}\\
		&\leq \sqrt{Y'J}+\sqrt{YJ}.\qedhere
	\end{align*}
\end{proof}

\section{Proof of Theorem \ref{thm:algB-thm-demo}}\label{subsect:duality-to-regret}

In this section, we provide the proof of \cref{thm:algB-thm-demo} and \cref{remark:unknownC}. We should notice that if $\hphi\geq C^*$, then $\epsapp(\hphi)$ reduces to 0, and the result in \cref{remark:unknownC} actually agrees with \cref{thm:algB-thm-demo}. Thus we handle them simultaneously. 
The key to the analysis is controlling the reward sub-optimality gap and the constraint violation in terms of the duality gap $\Gap(\ox)$ that is bounded in \cref{thm:algB-gap}. Before presenting the proof, let us introduce a few notations and lemmas. 

\subsection{Notations and supporting lemmas}\label{appdx:duality-notation}
In this proof, we will view $\nu$ as vectors in $\RR^{\nS\nA}$, and we define a matrix $A$ as
\begin{equation}
	\label{eqn:A}
	A:=\left[\mathbbm{1}_{\left\{s^{\prime}=s\right\}}-\gamma\PP\left(s'|s, a\right)\right]_{(s, a),s^{\prime}} \in \mathbb{R}^{|\mathcal{S}||\mathcal{A}|\times |\mathcal{S}|}.
\end{equation}  
Given the matrix $A$, we  conveniently write $\sum_{a} (\id\!-\!\gamma \bP_a)\nu_a$ as $A^\top \nu$. For the reweighted saddle point problem \eqref{prob:weighted-minimax}, one can easily partially minimize over $V$ and $\lambda$ since their domains are simple normal balls. Therefore, we define 
\begin{equation} 
		\cJ_{\kappa}(x) :=\min_{V\in\cV,\lambda\in\Lambda}\cL_{w}(V,\lambda,x)
		=r^TWx-\RV \nrm{A^\top  Wx-\rho_0}_1-\RLam \nrm{[U_{\kappa}Wx]_-}_\infty, 
\end{equation}
where we denote $\RV =\frac{8}{1-\gamma}\left(1+\frac{2}{\varphi}\right), \RLam =\frac{8}{\varphi}$. We also define 
\begin{equation}
	\begin{aligned}
		j(\hphi)&:=\min_{V\in\cV,\lambda\in\Lambda}\cL_{w}(V,\lambda,x) = \max_{x\in\cX}\cJ_{\kappa}(x)
	\end{aligned}
\end{equation}
as the optimal value of problem \eqref{prob:weighted-minimax}. Then $j(\hphi)$ has an implicit dependence on $\kappa$ due to the term $\nrm{[U_{\kappa}Wx]_-}_\infty$. In particular, we will write $j_0(\hphi)$ for the case where $\kappa=0$. Finally, we define $\pi^*_\kappa$ as the optimal policy with $\kappa$ conservative constraints. That is, 
$$\pi^*_\kappa = \argmax_{\pi} J(\pi)\,\,\st\,\, J_i^{u}(\pi)\geq\kappa,\,\, \forall i\in[I].$$ 
Then the following lemmas hold true.  
\begin{lemma}
	\label{lemma:J-0-kappa}
	Let $\pi^*$ be the optimal policy, and let $\pi^*_\kappa$ be defined above, then it holds that
	$$J(\pi^*)\geq J(\pi^*_\kappa)\geq J(\pi^*) - \frac{2\kappa}{\varphi}.$$
\end{lemma}
\begin{proof}
	The inequality $J(\pi^*)\geq J(\pi^*_\kappa)$ follows from definition. For the other inequality, we fix a ``baseline'' policy $\tpi$ satisfying the Slater's condition, namely $J^u_i(\tpi)\geq \frac{\varphi}{1-\gamma}$. Let $s=\frac{(1-\gamma)\kappa}{\varphi}$, we interpolate $\nu_{s}:=s\nu^{\pi^*} +(1-s)\nu^{\tpi}$. $\nu_{s}$ is still an occupancy measure such that $\iprod{u_i}{\nu_s}\geq s\iprod{u_i}{\nu^{\tpi}}\geq \kappa$ for $\forall i\in[I]$, and
	\[
	\iprod{r}{\nu_s}=\iprod{r}{\nu^{\pi^*}}-s(\iprod{r}{\nu^{\pi^*}}-\iprod{r}{\nu^{\tpi}})
	\geq \iprod{r}{\nu^{\pi^*}}- \frac{2s}{1-\gamma}=J(\pi^*)-\frac{2\kappa}{\varphi}.
	\]
	We complete the proof by noticing $J(\pi^*_\kappa)\geq \iprod{r}{\nu_s}$.
\end{proof}

The next lemma discusses the property of $j(\cdot)$. 
\begin{lemma}\label{lemma:eps-approx}
	Suppose the policy class $\Pi(\hphi)$ satisfies Slater's condition, then it holds that
	$$
	j(\hphi)\geq  \max_{\pi\in\Pi(\hphi)\cap\mathfrak{S}} J(\pi)-\frac{2\kappa}{\varphi}=J(\pi^*)-\epsapp(\hphi)-\frac{2\kappa}{\varphi}.
	$$
\end{lemma}
\begin{proof}
	Similar to the proof of \cref{lemma:J-0-kappa}, we fix a $\hpi=\argmax_{\pi\in\Pi(\hphi)\cap\mathfrak{S}} J(\pi)$ and a ``baseline'' policy $\tpi\in \Pi(\hphi)$ satisfying the Slater's condition. Let $\hnu:=\nu^{\hpi}$ and $\tilde{\nu}:=\nu^{\tpi}$ be the corresponding occupancy measures. Let $s=\frac{(1-\gamma)\kappa}{\varphi}$, then $\nu_{s}:=s\tnu +(1-s)\hnu$ is still an occupancy measure for which the corresponding policy belongs to $\Pi(\hphi)$. For $i\in[I]$, $\iprod{u_i}{\nu_s}\geq s\iprod{u_i}{\tnu}\geq \kappa$, and
	\[
	\iprod{r}{\nu_s}=\iprod{r}{\hnu}-s(\iprod{r}{\hnu}-\iprod{r}{\tnu})
	\geq \iprod{r}{\hnu}- \frac{2s}{1-\gamma}=\iprod{r}{\tnu}-\frac{2\kappa}{\varphi}.
	\]
	Now $W^{-1}\nu_s\in\cX$ by \cref{prop:empirical-dis}, and
	\[
	j(\hphi)\geq \!\cJ_\kappa(W^{-1}\nu_{s})\!=\!\iprod{r}{\nu_{s}}\!\geq\! \iprod{r}{\hnu}-\frac{2\kappa}{\varphi}\!=\!\!\max_{\pi\in\Pi(\hphi)\cap\mathfrak{S}} J(\pi)-\frac{2\kappa}{\varphi}\!=\!J(\pi^*)-\epsapp(\hphi)-\frac{2\kappa}{\varphi}.\qedhere\]
\end{proof}

The following result is obtained from \cite[Lemma 3]{bai2021achieving}, by replacing $\lambda$ and $(v,u)$ in \cite[Lemma 3]{bai2021achieving} with our notation $(1-\gamma)\nu$ and $(V,\lambda)$, respectively.
\begin{lemma}\label{lemma:dual-feasible}\label{lemma:duality-gap-decomp}
	For any dual optimal solution $(V_{\kappa}^*,\lambda_{\kappa}^*)$ of the problem \eqref{prob:CMDP-LP}, where the constraint utilities $u_i$ is replaced with the shifted utilities $u_i^\kappa$, we have
	\begin{equation*}
		\begin{aligned}
			\nrm{\lambda_{\kappa}^*}_1\leq\frac{2}{\varphi}\qquad\mbox{and}\qquad \nrm{V_{\kappa}^*}\leq \frac{1}{1-\gamma}\left(1+\frac{2}{\varphi}\right).
		\end{aligned}
	\end{equation*}
	For any $\nu\in\RR_{\geq0}^{\nS\nA}$ and any $\Delta>0$, the inequality
	$J(\pi^*_{\kappa})-\cJ(W^{-1}\nu)\leq \dgap$
	immediately implies that  
	\begin{equation*}
    	J(\pi^*_{\kappa})-\iprod{r}{\nu}\leq \dgap,\quad
		\nrm{A^\top  \nu-\rho_0}_1\leq \frac{2\dgap}{\RV }, \quad\mbox{and}\quad
		\nrm{[U_{\kappa}\nu]_-}_\infty \leq \frac{2\dgap}{\RLam }
	\end{equation*}
    as long as $\RV \geq 2\nrm{V^*_\kappa}_\infty, \RLam \geq 2\nrm{\lambda^*_\kappa}_1$. 
\end{lemma}
Finally, we introduce the last lemma that is needed in this proof. 
\begin{lemma}\label{lemma:LP-to-policy}
	For any vector $\tnu\in\RR^{\nS\nA}_{\geq 0}$ that is an \textit{approximate} visitation measure, consider its associate policy $\tpi$ defined by $\tpi(a|s)=\frac{\tnu(s,a)}{\sum_{a'}\tnu(s,a')}$. Let $\nu^{\tpi}$ be the \textit{true} visitation measure of $\tpi$, then
	\[
	\nrm{\tnu-\nu^{\tpi}}_1\leq \frac{1}{1-\gamma}\nrm{A^\top \tnu-\rho_0}_1.
	\] 
\end{lemma}
\begin{proof}
	For policy $\pi$, we consider its state visitation measure $\nu_\pi$ defined by $\nu_{\pi}(s)=\sum_a \nu^{\pi}(s,a)$.
	Then $\nu^{\pi}(s,a)=\pi(a|s)\nu_{\pi}(s)$. With the transition matrix  $\PP_{\pi}(s'|s)=\sum_{a}\pi(a|s)\PP(s'|s,a)$, then the constraint $A^\top\nu^{\pi}=\rho_0$ is equivalent to $(I-\gamma\PP_{\pi})\nu_{\pi}=\rho_0.$
	
	Let $\tpi$ induced by $\tnu$, then $\nu_{\tpi}$ satisfies $(I-\gamma\PP_{\tpi})\nu_{\tpi}=\rho_0$. Let $\tnu'$ be defined by $\tnu'(s)=\sum_a \tnu(s,a)$, then $\tnu(s,a)=\tpi(a|s)\tnu'(s)$, and hence $(I-\gamma\PP_{\tpi})\tnu'=A\tnu$. Therefore,
	\[\begin{aligned}
		\nrm{\nu_{\tpi}-\tnu'}_1=\nrm{(I-\gamma\PP_{\tpi})^{-1}(\rho_0-A^\top  \tnu)}_1
		\leq \nrm{(I-\gamma\PP_{\tpi})^{-1}}_1\nrm{A^\top \tnu-\rho_0}_1 \leq \frac{1}{1-\gamma} \nrm{A^\top  \tnu-\rho_0}_1.
	\end{aligned}\]
We finalize the proof by the following equality
\begin{align*}
	\nrm{\nu^{\tpi}\!-\!\tnu}_1 \!=\! \sum_{s,a}\Big|\tpi(a|s)(\nu_{\tpi}(s)\!-\!\tnu'(s))\Big| \!=\! \sum_{s}\Big(\sum_{a}\tpi(a|s)\Big)|\nu_{\tpi}(s)-\tnu'(s)| \!=\! \nrm{\nu_{\tpi}-\tnu'}_1. &\qedhere
\end{align*}
\end{proof}

\subsection{Analysis}

Now we are ready to present the proof of \cref{remark:unknownC} and \cref{thm:algB-thm-demo}.

\begin{proof}
By definition of $\Gap(\ox)$, we have
\begin{equation}
	\label{eqn:subopt-1}
	\begin{aligned}
		\Gap(\ox)=\max_{x\in\cX}\min_{V\in\cV, \lambda\in\Lambda}\cL_{w}(V,\lambda,x)-\min_{V\in\cV, \lambda\in\Lambda}\cL_{w}(V,\lambda,\overline{x})=
		j(\hphi)-\cJ_{\kappa}(\ox).
	\end{aligned}
\end{equation} 
Define $\onu=W\overline x$, and define $\dgap:=J(\pi^*_{\kappa})-\cJ_{\kappa}(W^{-1}\onu)$, then we have
\begin{equation}\label{eqn:duality-gap-bound}
	\dgap	=\Gap(\ox)+J(\pi^*_{\kappa})-j(\hphi)\overset{(i)}{\leq} \Gap(\ox)+J(\pi^*)-j(\hphi)\overset{(ii)}{\leq} \Gap(\ox)+\epsapp(\hphi)+\frac{2\kappa}{\varphi},
\end{equation}
where (i) is because $J(\pi^*_{\kappa})\leq J(\pi^*)$ and (ii) is due to 
\cref{lemma:eps-approx}.  Now, let $\nu^{\opi}$ be the true visitation measure of $\opi$, where $\opi(a|s):=\frac{\ox(s,a)}{\sum_{a'}\ox(s,a')}$. Then \cref{lemma:LP-to-policy} immediately indicates that
\begin{eqnarray*}
	\nrm{\nu^{\opi}-\ox}_1&\leq& \frac{1}{1-\gamma}\nrm{A^\top  \ox-\rho_0}_1\\
	& \leq & \frac{1}{1-\gamma}\Big(\|A^\top  (\ox-W\ox)\|_1+\|A^\top  W\ox-\rho_0\|_1\Big)\\
	& \leq & \frac{1}{1-\gamma}\Big(2\nrm{\ox-\onu}_1+\nrm{A^\top \onu-\rho_0}_1\Big)
\end{eqnarray*}
which further gives
\begin{equation}
	\label{eqn:regret-true-to-fake}
	 \nrm{\nu^{\opi}-\onu}_1\leq \frac{1}{1-\gamma}\big(\nrm{A^\top \onu-\rho_0}_1+3\nrm{\ox-\onu}_1\big).
\end{equation}
Consequently, we have 
 \begin{eqnarray*}
 	&&
 	J(\pi^*_{\kappa})-\cJ_{\kappa}(W^{-1}\nu^{\opi}) \\
 	&\overset{(i)}{=}& 	J(\pi^*_{\kappa})-\iprod{r}{\nu^{\opi}}+\RLam\big\|\left[U_{\kappa}\nu^{\opi}\right]_{-}\big\|_{\infty}\\
 	& \overset{(ii)}{\leq} & J(\pi^*_{\kappa})-\iprod{r}{\onu}+\RLam\nrm{\left[U_{\kappa}\onu\right]_{-}}_{\infty} +\frac{1+\frac{3}{2}\RLam }{1-\gamma}\left(\nrm{A^\top \onu-\rho_0}_1+3\nrm{\ox-\onu}_1\right)\\
    & \overset{(iii)}{\leq}  & J(\pi^*_{\kappa}) - \cJ_{\kappa}(W^{-1}\onu) +\frac{5(\RLam +1)}{1-\gamma}\nrm{\ox-\onu}_1\\
    &\overset{(iv)}{\leq}& \Delta + 45\epsilon_e,
 \end{eqnarray*}
where (i) is because $\|A^\top\nu^{\opi} - \rho_0\|_1=0$, (ii) is due to the fact that $\big|\iprod{r}{\nu^{\opi}}-\iprod{r}{\onu}\big|\leq \nrm{\nu^{\opi}-\onu}_1$ and $\abs{\|[U_{\kappa}\nu^{\opi}]_{-}\|_{\infty}-\|[U_{\kappa}\onu]_{-}\|_{\infty}}\leq \frac{3}{2}\nrm{\nu^{\opi}-\onu}_1$, (iii) is because of $\frac{1+\frac{3}{2}\RLam }{1-\gamma}\leq \RV $, and (iv) is because of $\nrm{\ox-\onu}_1\leq \varphi(1-\gamma)\epsilon_e$ by \cref{prop:empirical-dis}. Finally, applying \cref{lemma:duality-gap-decomp} to $\nu^{\opi}$ yields
\[\begin{aligned}
	J(\pi^*_{\kappa})-\iprod{r}{\nu^{\opi}}\leq \dgap+45\epsilon_e,\qquad
	\nrm{\left[U_{\kappa}\nu^{\opi}\right]_{-}}_{\infty}\leq \frac{\varphi}{4}\left(\dgap+45\epsilon_e\right).
\end{aligned}\]
By \cref{lemma:eps-approx}, we have
\begin{equation}
\begin{aligned}
	&J(\pi^*)-\iprod{r}{\nu^{\opi}}\leq J(\pi^*)-j(\hphi)+\Gap(\ox)+45\epsilon_e\leq \Gap(\ox)+\epsapp(\hphi)+\frac{2\kappa}{\varphi}+45\epsilon_e,\\
	&J^u_i(\opi) \geq \kappa-\big\|\left[U_{\kappa}\nu^{\opi}\right]_{-}\big\|_{\infty}\geq \frac{\kappa}{2} - \frac{\varphi}{4}\big(\Gap(\ox)+\epsapp(\hphi)\big)-12\varphi\epsilon_e.
\end{aligned}
\end{equation}
Combining the above inequality with the fact that $\epsilon_e=\frac{\epsilon}{100}$, $\kappa=5\varphi\epsilon$, $\Gap(\ox)\leq\epsilon/2$ completes the proof. 
\end{proof}



Finally, we point out a by-product of the above analysis, which is useful for the $\verify$ method.
\begin{corollary}\label{corollary:regret-ox-ov-v}
Under the same assumption of \cref{thm:algB-thm-demo}, with probability at least $1-\nicefrac{2\delta}{3}$, it holds that
\begin{equation} 
\begin{aligned}
  \nrm{A^\top  \onu-\rho_0}_1\leq \frac{11}{8}\varphi(1-\gamma)\epsilon,\qquad
  \nrm{\neg{U_{\kappa}\onu }}_{\infty}\leq \frac{11}{4}\varphi\epsilon
\end{aligned}
\end{equation}
for $\onu:=W\ox$.
\end{corollary}
\begin{proof}
Due to $\hphi\geq C^*$ and \cref{lemma:eps-approx}, we have
\[\begin{aligned}
	J(\pi^*_{\kappa})-\cJ_{\kappa}(\ox)\leq j(\hphi)-\cJ_{\kappa}(\ox)+\frac{2\kappa}{\varphi}=\Gap(\ox)+\frac{2\kappa}{\varphi}\leq 11\epsilon.
\end{aligned}\]
Applying \cref{lemma:duality-gap-decomp} yields
\begin{align*}
  \nrm{A^\top  \onu-\rho_0}_1\leq \frac{2\cdot 11\epsilon}{\RV }\leq \frac{11}{8}\varphi(1-\gamma)\epsilon, \qquad
  \nrm{[U_{\kappa}\onu]_-}_\infty \leq \frac{2\cdot 11\epsilon}{\RLam }=\frac{11}{4}\varphi\epsilon. &\qedhere
\end{align*}
\end{proof}



\section{Proofs for Section \ref{sec:LowerBound}}\label{appdx:LowerBound}

\subsection{Proof of Theorem \ref{thm:lower-bound-demo}}\label{appdx:LowerBound-1}
In this section, we provide the complete version of the construction illustrated in Section \ref{sec:LowerBound}. 
Let us define
$$K:=\min\left(\floor{\frac{I}{2}},\floor{\frac{A-1}{2}}\right),\quad S_c=\min\left(\floor{\frac{I}{2K}},S\right),\quad S_u = \begin{cases}
	S-S_c, & \mbox{if }S_c<S-3,\\
	0, & \mbox{otherwise}.
\end{cases} $$
 
The CMDP instance $\cM$ that we construct consists of two groups of basic blocks. The first group includes $S_c$ replicas of the basic block characterized in \cref{fig:lower-bound}, each with actions $\{a_1,b_1,...,a_k,b_k,e\}$ and $2K$ constraints. The second group includes $S_u$ replicas of the basic blocks characterized by \cref{fig:lower-bound} (a) and \cref{fig:lower-bound}(c), each basic block only has two actions $\{a,e\}$ and no constraint. In fact the construction of the second group (``unconstrained part'') is similar to the hard MDP constructed in \cite{OfflineRL_Lower_Bound}. The transition kernel $\PP_{\theta}$ of $\cM$ is parametrized by $\theta=(\theta_c, \theta_u)\in\varTheta:=\{-1,+1\}^{S_cK}\times \{-1,+1\}^{S_u}$ and $\varpi_c,\varpi_u\in(0,\frac12]$. The details of $\cM$ are listed as follows. 

\paragraph{States and actions} The state space $\cS$ consists of $S_c+S_u$  4-state basic blocks, plus an extra ``null'' state $s_{-1}$. 
The first $S_c$ basic blocks are exactly what we described in \cref{sec:LowerBound}, we write  $\cS_c=\bigsqcup_{j=1}^{S_c} \big\{s_0^j,s_1^j,\sp^j,\sm^j\big\}$. The next $S_u$ basic blocks will be described below, we write $\cS_u=\bigsqcup_{j=S_c+1}^{S_c+S_u} \big\{s_0^j,s_1^j,\sp^j,\sm^j\big\}$. By default, $\cS_u=\emptyset$ if $S_u=0$. Then $\cS = \cS_c\bigsqcup\cS_u\bigsqcup\big\{s_{-1}\big\}.$
Next, we describe the detailed information of each block $j$.
\begin{itemize}
    \item At $s_0^j$, $\sp^j$ and $\sm^j$, there is no action, and the transition does not depend on $\theta$:
    \begin{equation}\label{eqn:lower-bound-transition}
    \begin{aligned}
      & \condP{s_0^j}{s_0^j}=p,
      &&\condP{s_1^j}{s_0^j}=1-p,\\
      & \condP{\sp^j}{\sp^j}=q, 
      &&\condP{s_0^j}{\sp^j}=1-q,\\
      & \condP{\sm^j}{\sm^j}=q, 
      &&\condP{s_0^j}{\sm^j}=1-q,
    \end{aligned}
    \end{equation}
      where $p=\frac{1}{2-\gamma}$ and $q=2-\frac{1}{\gamma}$. 
      We assign reward as $r(\sp^j)=1$, $r(\sm^j)=-1$. 
    \item \textbf{Constrained state} At $s_1^j\in\cS_c$, there are $2K+1$ actions $a_1,b_1,\cdots,a_K,b_K, e$ such that
    \[\begin{aligned}
    &\condPt{\sp^j}{s_1^j,a_i}=\frac{1+\varpi_c\theta_{i,j}}{2}, 
    &&\condPt{\sm^j}{s_1^j,a_i}=\frac{1-\varpi_c\theta_{i,j}}{2}, \\
    &\condP{\sp^j}{s_1^j,b_i}=\frac{1}{2}\left(1-\frac{\varpi_c}{2}\right), 
    &&\condP{\sm^j}{s_1^j,a_i}=\frac{1}{2}\left(1+\frac{\varpi_c}{2}\right), \\
    &\condP{\sp^j}{s_1^j,e}=\frac{1}{2}, 
    &&\condP{\sm^j}{s_1^j,e}=\frac{1}{2}. \\
    \end{aligned}\]
    Here we use subscript $\theta$ to emphasize the dependency of $\PP_{\theta}$ on $\theta$.\footnote{
      Here we view $\theta_c\in\{-1,1\}^{S_cK}$ as a vector indexed by $(i,j)\in[K]\times[S_c]$, and $\theta_{i,j}$ stands for the $(i,j)$-th component of $\theta_c$. Similarly, we view $\theta_u\in\{-1,1\}^{S_u}$ as a vector indexed by $j$ with $S_c+1\leq j\leq S_c+S_u$, and $\theta_j$ stands for the $j$-th component of $\theta_u$.
    } 
    \item \textbf{Unconstrained state} At $s_1^j\in\cS_u$, there are two actions $a, e$ such that
    \[\begin{aligned}
    &\condPt{\sp^j}{s_1^j, a}=\frac{1+\varpi_u\theta_{j}}{2}, 
    &&\condPt{\sm^j}{s_1^j, a}=\frac{1-\varpi_u\theta_{j}}{2}, \\
    &\condP{\sp^j}{s_1^j, e}=\frac{1}{2}, 
    &&\condP{\sm^j}{s_1^j, e}=\frac{1}{2}. \\
    \end{aligned}\]
    \item The null state $s_{-1}$ has no action or reward, and it always transits to itself.
\end{itemize}

\paragraph{Initial distribution} In the initial distribution, $\rho_0(s_{-1})=\rho_0(s_1^j) = \rho_0(\sp^j)=\rho_0(\sm^j) = 0, \forall j$. The nonzero probabilities only spread across the $\{s_0^j\}$.
In the case $S_u>0$, we choose $\rho_0$ to be
\[\begin{aligned}
\rho_0(s_0^j)=\begin{cases}
  \frac{\II\{s_0^j\in\cS_c\}}{2S_c} + \frac{\II\{s_0^j\in\cS_u\}}{2S_u}, \quad& \mbox{if }\cS_u\neq\emptyset,\\
  \frac{1}{S_c}, \quad& \mbox{otherwise}.
\end{cases}
\end{aligned}\]
Without loss of generality, we will only deal with the case where $\cS_u\neq\emptyset$. 
  
\paragraph{Constraints} At each constrained block in $S_c$, for each pair of actions $(a_i,b_i)$ at the state $s_0^j\in\cS_c$, we introduce two constraints defined by the utilities
  \[\begin{aligned}
  u_{i,j}(s_1^j,a_i)=-1, 
  \quad\quad u_{i,j}(s_1^j,b_i)=1,
  \quad\quad \tu_{i,j}(s_1^j,b_i)=-1.
  \end{aligned}\]
At all the other state and actions, $u_{i,j}$ and $\tu_{i,j}$ returns 0. 
Then we set the constraints to be
  \[\begin{aligned}
  J^u_{i,j}(\pi):=\iprod{\nu^\pi}{u_{i,j}}\geq 0,
  \qquad\mbox{and}\qquad \tJ^u_{i,j}(\pi):=\iprod{\nu^\pi}{\tu_{i,j}}\geq -\frac{\rho_{c} v_1}{4K},
  \end{aligned}\]
where $\rho_c$ and $v_1$ are constants specified later in \eqref{eqn:lower-bound-const}. After suitable shifting we can make sure that each constraint has the form $J^u\geq 0$. Basically, these two constraints are equivalent to $\pi(a_i|s_1^j)\leq \pi(b_i|s_1^j)\leq\frac{1}{4K}$. We remark that there are in total $S_cK\leq I$ constraints.


\paragraph{Optimal policy} First, let us calculate the visitation measure of any given policy $\pi$. According to the proof of \cref{lemma:LP-to-policy}, we set $\nu_{\pi}$ be the state visitation measure and let $\PP_{\pi}$ be the state transition matrix under policy $\pi$, then $\nu_{\pi}$ will be the unique solution to  $(I-\gamma\PP_{\pi})\nu_{\pi}=\rho_0$. Note that the $S_c+S_u$ basic blocks are in fact independent blocks, i.e., there are no transitions between different blocks. The matrix $(I-\gamma\PP_{\pi})$ is in fact a block-diagonal with $S_c+S_u$ 4 by 4 blocks and a 1 by 1 block, and we can solve the $\nu_{\pi}$ block by block. Define the constants
\begin{equation}\label{eqn:lower-bound-const}
	\begin{aligned}
		v_0=\frac{2}{(2+\gamma)},\,\,\,
		v_1=\frac{2\gamma}{(2+\gamma)(2-\gamma)},\,\,\,
		v=\frac{\gamma^2}{(2+\gamma)(2-\gamma)},\,\,\, \rho_c=\frac{1}{2S_c},\,\,\, \rho_u = \frac{1}{2S_u},
	\end{aligned}
\end{equation}
and we consider
$$
	r_j(\pi)=\begin{cases}
		\sum_{i}\left(\theta_{i,j}\pi(a_i|s_1^j)-\frac{1}{2}\pi(b_i|s_1^j)\right), & s_1^j\in\cS_c,\\
		\theta_{j}\pi(a|s_1^j), & s_1^j\in\cS_u.
	\end{cases}
$$
By a direct computation, the state visitation measure of $\pi$ is given by
$$\nu_{\pi}(\sp^j)=\frac{v}{1-\gamma}\frac{1+\varpi_\diamond r_j(\pi)}{2}, \qquad \nu_{\pi}(s_0^j)=\frac{\rho_{\diamond} v_0}{1-\gamma},$$
$$\nu_{\pi}(\sm^j)=\frac{v}{1-\gamma}\frac{1-\varpi_\diamond r_j(\pi)}{2},\qquad \nu_{\pi}(s_1^j)=\frac{\rho_{\diamond}v_1}{1-\gamma},$$
where $\diamond$ stands for $c$ if the block $j$ belongs to $\cS_c$, and $\diamond$ stands for $u$ if the block $j$ belongs to $\cS_u$. 
Consequently, the cumulative reward and the utilities are
\begin{equation}\label{eqn:lower-bound-reawrd-formula}
\begin{aligned}
&J(\pi;\theta)=\sum_{j}\Big(\nu_{\pi}(\sp^j)-\nu_{\pi}(\sm^j)\Big)
=\frac{v}{1-\gamma}\Big(
  \rho_c\varpi_c\sum_{j:s_1^j\in\cS_c} r_j(\pi)+\rho_u\varpi_u\sum_{j:s_1^j\in\cS_u} r_j(\pi)
\Big),\\
&J_{i,j}(\pi;\theta)=\nu^{\pi}(s_1^j,b_i)-\nu^{\pi}(s_1^j,a_i)=\rho_{c} v_1\left(\pi(b_i|s_1^j)-\pi(a_i|s_1^j)\right),\\
&\tJ_{i,j}(\pi;\theta)=-\nu^{\pi}(s_1^j,b_i)=-\rho_{c} v_1\pi(b_i|s_1^j).\\
\end{aligned}
\end{equation}
Therefore, $\pi$ being safe is equivalent to requiring $\pi(a_i|s_1^j)\leq\pi(b_i|s_1^j)\leq\frac{1}{4K}$ for all the constrained block $j$ in $\cS_c$, and any $1\leq i\leq K$. With the above explicit expression of $J(\pi;\theta)$, we know that the (unique) optimal policy $\pi^{*,\theta}$ under the transition dynamic $\PP_\theta$ is 
\begin{equation}\label{eqn:lower-bound-optimal-policy}
\begin{aligned}
  &\pi^{*,\theta}(a_i|s_1^j)=\pi^{*,\theta}(b_i|s_1^j)=\frac{\II\{\theta_{i,j}=1\}}{4K}, 
  \quad&& s_1^j\in\cS_c,\\
  &\pi^{*,\theta}(a|s_1^j)=\II\{\theta_{j}=1\},
  && s_1^j\in\cS_u.\\
\end{aligned}
\end{equation}
Denote $J^*_\theta:=J(\pi^{*,\theta};\theta)$ the optimal safe reward and $\ttheta=\frac{\theta+1}{2}$, then
\begin{equation}\label{eqn:lower-bound-optimal-value}
\begin{aligned}
  J^*_\theta=J(\pi^{*,\theta};\theta)=\frac{v}{1-\gamma}\left(
    \varpi_c\rho_c\sum_{j: s_1^j\in\cS_c}\sum_{i=1}^K\frac{\ttheta_{i,j}}{8K}+\varpi_u\rho_u\sum_{j: s_1^j\in\cS_u} \ttheta_j
  \right).
\end{aligned}
\end{equation}

\paragraph{Reference distribution} Finally, we set the reference distribution $\mu$ as 
  \[\begin{aligned}
  &\mu(s_0^j)=\frac{v_0}{C}\rho_{\diamond},\qquad\mu(\sp^j)=\frac{3}{4}\frac{v}{C}\rho_{\diamond},\qquad \mu(\sm^j)=\frac{1}{2}\frac{v}{C}\rho_{\diamond},\qquad \mu(s_1^j,e)=\frac{v_1(1-\gamma)}{C}\rho_{\diamond},\\
  &\begin{cases}
    \mu(s_1^j,a_i)=\mu(s_1^j,b_i)=\frac{\rho_c v_1(1-\gamma)}{4KC}, \ i\in[I] \quad & s_1^j\in\cS_c,\\
    \mu(s_1^j,a)=\frac{\rho_u v_1(1-\gamma)}{C}, \quad & s_1^j\in\cS_u,\\
  \end{cases}\\
  &\mu(s_{-1})=1-\sum_{j}\left(\mu(s_0^j)+\mu(s_1^j)+\mu(\sp^j)+\mu(\sm^j)\right).
  \end{aligned}\]
  As long as $C\geq 2$, $\mu(s_{-1})$ defined above is positive. Also, for any $\theta$, it holds that
  \[\max_{s,a}\frac{\nu^{\pi^{*,\theta}}(s,a)}{\mu(s,a)}\leq \frac{C}{1-\gamma}, \quad \sum_{s,a}\frac{\nu^{\pi^{*,\theta}}(s,a)}{\mu(s,a)}\leq \frac{(\nS+I)C}{1-\gamma}.\]

  We denote $\mu_{\theta}=\mu\otimes\PP_{\theta}$ as the probability measures of the transition pair $\zeta=(s,a,s')$ generated from the reference distribution $\mu$.

\paragraph{Output policy as an estimator of $\theta$} Assume that an algorithm $\mathfrak{A}$ consumes $N$ samples generated from $\mu_{\theta}$, and outputs a policy $\hat\pi$ that is possibly dependent on the internal randomness of $\mathfrak{A}$. Consider the corresponding random vector $\hpi_c:=\big(\hpi(a_i|s_1^j)\big)_{i,j}$ and $\hpi_u:=\big(\hpi(a|s_1^j)\big)_{j}$. Then, $4K\hpi_c$ can be viewed as an estimator of $\ttheta_c$, and $\hpi_u$ can be viewed as an estimator of $\ttheta_u$. We establish the  following lemma to characterize the error for ``misspecifying'' the parameter $\theta$.

\begin{lemma}\label{lemma:lower-bound-regret}
  For any policy $\pi$, we define
  \begin{equation}
  \begin{aligned}
   \cL(\pi;\theta)
   :=&\pos{J^*_{\theta}-J(\hpi;\theta)}
   +\frac{\gamma\varpi_c}{1-\gamma}\sum_{i,j}\left(\neg{J_{i,j}(\hpi;\theta)}+\neg{\tJ_{i,j}(\hpi;\theta)-\frac{v_1}{4SK}}\right)\\
   =&\pos{J^*_{\theta}-J(\hpi;\theta)}+\frac{\gamma\varpi_c}{1-\gamma}\violation(\hpi;\theta).
  \end{aligned}
  \end{equation}
  Then it holds that
  \begin{equation}\label{eqn:lower-bound-regret}
  \begin{aligned}
    \cL(\pi;\theta)\geq \frac{v\rho_c\varpi_c}{8K(1-\gamma)}\nrm{4K\hpi_c-\ttheta_c}_1+\frac{v\rho_u\varpi_u}{1-\gamma}\nrm{\hpi_u-\ttheta_u}.
  \end{aligned}
  \end{equation}
\end{lemma}
\begin{proof}
  The description of $J^*_{\theta}$ in \eqref{eqn:lower-bound-optimal-value} gives
  \begin{equation*}
  \begin{aligned}
    \cL(\hpi;\theta)
    =&\frac{v}{1-\gamma}\pos{\rho_c\varpi_c\sum_{i,j}\left(\frac{\ttheta_{i,j}}{8K}-\theta_{i,j}\pi(a_i|s_1^j)+\frac{\pi(b_i|s_1^j)}{2}\right)+\rho_u\varpi_u\sum_j \left(\ttheta_j-\theta_{j}\pi(a|s_1^j)\right)}\\
    &+\frac{\gamma v_1\rho_c\varpi_c}{1-\gamma}\sum_{i,j}\left(\neg{\pi(b_i|s_1^j)-\pi(a_i|s_1^j)}+\neg{\frac{1}{4K}-\pi(b_i|s_1^j)}\right)\\
    \geq& \frac{v}{1-\gamma}\left(\rho_c\varpi_c\sum_{i,j}\delta_{i,j}+\rho_u\varpi_u\sum_j \left(\ttheta_j-\theta_{j}\pi(a|s_1^j)\right)\right),
  \end{aligned}
  \end{equation*}
  where we use the fact $\gamma v_1=2v$, and denote
  \[
    \delta_{i,j}=\frac{\ttheta_{i,j}}{8K}-\theta_{i,j}\pi(a_i|s_1^j)+\frac{\pi(b_i|s_1^j)}{2}
    +2\neg{\pi(b_i|s_1^j)-\pi(a_i|s_1^j)}+2\neg{\frac{1}{4K}-\pi(b_i|s_1^j)}.
  \]
  Clearly $\ttheta_j-\theta_{j}\pi(a|s_1^j)\geq \abs{\ttheta_j-\pi(a|s_1^j)}$ for all $s_1^j\in\cS_u$. As for $s_1^j\in\cS_c$, we consider the case $\theta_{i,j}=1$ and $\theta_{i,j}=-1$ separately.

  Case 1, $\theta_{i,j}=-1$. Directly $\delta_{i,j}\geq \pi(a_i|s_1^j)=\abs{\pi(a_i|s_1^j)-\frac{\ttheta_{i,j}}{4K}}$.

  Case 2, $\theta_{i,j}=1$. By the fact that
  \[\frac{z}{2}-x+\frac{y}{2}+2\neg{y-x}+2\neg{z-y}
  \geq \frac{z-x}{2}+\frac{3}{2}\neg{z-x}
  \geq \frac{\abs{z-x}}{2}\quad \forall x,y,z,\]
  we can plug in $x=\pi(a_i|s_1^j)$, $y=\pi(b_i|s_1^j)$ and $z=\frac{1}{4K}$ and derive
  \[
    \delta_{i,j}
    \geq \frac{1}{2}\abs{\frac{1}{4K}-\pi(a_i|s_1^j)}.
  \]
  Consequently, \eqref{eqn:lower-bound-regret} is established by combining the above inequalities.
\end{proof}

We now invoke the following lemma due to \cite{gilbert1952comparison} and \cite{varshamov1957estimate}.
\begin{lemma}\label{lemma:lower-bound-theta}
  For any integer $n\geq 1$, there exists a subset $\varTheta_n$ of $\{-1,1\}^{n}$ such that $|\varTheta_n| \geq \exp (n / 8)$, and for any pair of different $\theta, \theta'\in \varTheta_n$, one has $\left\|\theta-\theta'\right\|_{1} \geq \frac{n}{2}$.
\end{lemma}
Fix a $\varTheta_c$ with $n=S_cK$ and a $\varTheta_u$ with $n=S_u$, we consider the family of CMDPs $\mathfrak{M}:=\left\{\cM_{\theta}\right\}_{\theta\in\varTheta_c\times\varTheta_u}$. Intuitively, CMDPs from this family are hard to distinguish according to samples. This idea can be shown mathematically by the following generalized version of Fano's inequality from \cite[Lemma 3]{assouad1996fano}.
\begin{lemma}[Generalized Fano's inequality]\label{lemma:fano}
Let $r \geq 2$ be an integer and let $\mathcal{P}$ be a set of $r$ probability measures on $(\Omega,\cF)$. Assume that $\theta(\PP)$ is the parameter of interest with values in a pseudo-metric space $(\mathcal{D}, d)$. Let $\hat{\theta}=\hat{\theta}(X)$ be an estimator of $\theta(\PP)$ based on a sample $X$ from a distribution $\PP\in\cP$. Assume that
\[
d\left(\theta(\PP), \theta(\PP')\right) \geq \alpha,\quad\forall \PP,\PP'\in\cP,
\]
and
\[
\KL{\PP}{\PP'}=\int_{\Omega} \log \left(\frac{d\PP}{d\PP'}\right) d\PP \leq \beta.
\]
Then it holds that
\[
\max_{\PP\in\cP}\EE_{\PP} d\left(\hat{\theta}, \theta\left(\PP\right)\right) \geq \frac{\alpha}{2}\left(1-\frac{\beta+\log 2}{\log r}\right).
\]
\end{lemma}
It is worth noting that the estimator needs not to belong to $\left\{\theta(\PP)\right\}_{\PP\in\cP}$.
In our problem, the underlying space $(\Omega,\cF)$ depends on the internal randomness of $\mathfrak{A}$, and the probability measure on $(\Omega,\cF)$ is the extension of $\mu_{\theta}^{\otimes N}$ ($\mu_{\theta}^{\otimes N}$ is the probability measure on $\Omega_0=\left(\cS\times\cA\times\cS\right)^N$, the space of the $N$-tuple of samples $(\zeta_1,\cdots,\zeta_N)$).


\paragraph{The proof of \cref{thm:lower-bound-demo}} 
We have already demonstrated that $4K\hpi_c$ can be viewed as an estimator of $\ttheta_c$ in \cref{lemma:lower-bound-regret}, and hence $8K\hpi_c-1$ can be viewed as an estimator of $\theta_c$. We fix a $\theta_u\in\varTheta_u$, then Fano's inequality (\cref{lemma:fano}) yields
\[\begin{aligned}
  \max_{\theta_c\in\varTheta_c}\EE_{(\theta_c,\theta_u)}\nrm{8K\hpia-1-\theta_c}_1
  &\geq \frac{S_cK}{2}\left(1-\frac{\max_{\theta_c,\theta'_c\in\varTheta_c}\KL{\mu_{(\theta_c,\theta_u)}^{\otimes N}}{\mu_{(\theta_c',\theta_u)}^{\otimes N}}+\log 2}{\log |\varTheta_c|}\right)\\
  &= \frac{S_cK}{2}\left(1-\frac{N\max_{\theta_c,\theta'_c\in\varTheta_c}\KL{\mu_{(\theta_c,\theta_u)}}{\mu_{(\theta_c',\theta_u)}}+\log 2}{\log |\varTheta_c|}\right).
\end{aligned}\]
For any $\theta_c,\theta'_c\in\varTheta_c$, we have
\[\begin{aligned}
  \KL{\mu_{(\theta_c,\theta_u)}}{\mu_{(\theta_c',\theta_u)}}
  =&\sum_{i,j}\mu(s_1^j,a_i)\KL*{\frac{1+\theta_{i,j}\varpi_c}{2}}{\frac{1+\theta_{i,j}'\varpi_c}{2}}\\
  \leq& \sum_{i,j}\mu(s_1^j,a_i)\frac{4\varpi_c^2}{1-\varpi_c^2}
  = \frac{(1-\gamma)v_1}{2C}\frac{\varpi_c^2}{1-\varpi_c^2}.
\end{aligned}\]
Then, taking $\varpi_c=\min\left\{\sqrt{\frac{(S_cK-3)C}{8(1-\gamma)N}},\frac12\right\}$ is enough to ensure
\[\begin{aligned}
  \frac{N\max_{\theta_c,\theta'_c\in\varTheta_c}\KL{\mu_{(\theta_c,\theta_u)}}{\mu_{(\theta_c',\theta_u)}}+\log 2}{\log |\varTheta_c|}\leq \frac56,
\end{aligned}\]
which further gives $\max_{\theta_c\in\varTheta_c}\EE_{(\theta_c,\theta_u)}\nrm{8K\hpia-1-\theta_c}_1\geq \frac{S_cK}{12}$, and hence
\[\begin{aligned}
  \max_{\theta_c\in\varTheta_c}\EE_{(\theta_c,\theta_u)}\left[\frac{1}{S_cK}\nrm{4K\hpia-\ttheta_c}_1\right]\geq \frac{1}{24}.
\end{aligned}\]
Similarly, we can take $\varpi_u=\min\left\{\sqrt{\frac{(S_u-3)C}{8(1-\gamma)N}},\frac12\right\}$ to ensure that for any fixed $\theta_c\in\varTheta_c$, 
\[\begin{aligned}
  \max_{\theta_u\in\varTheta_u}\EE_{(\theta_c,\theta_u)}\left[\frac{1}{S_u}\nrm{\hpi_u-\ttheta_u}_1\right]\geq \frac{1}{24}
\end{aligned}\]
Therefore, we obtain
\begin{align*}
&\max_{\theta\in\varTheta}\EE_{\theta}\cL(\hpi;\theta)\\
&\geq \max_{\theta\in\varTheta}\EE_{\theta}\left[
  \frac{v\rho_c\varpi_c}{8K(1-\gamma)}\nrm{4K\hpi_c-\ttheta_c}_1+\frac{v\rho_u\varpi_u}{1-\gamma}\nrm{\hpi_u-\ttheta_u}
\right]\\
&=\frac{v}{2(1-\gamma)}\max_{\theta_c\in\varTheta_c}\max_{\theta_u\in\varTheta_u}\left\{
  \frac{\varpi_c}{8}\EE_{(\theta_c,\theta_u)}\left[\frac{1}{S_cK}\nrm{4K\hpi_c-\ttheta_c}_1\right]
  +\varpi_u\EE_{(\theta_c,\theta_u)}\left[\frac{1}{S_u}\nrm{\hpi_u-\ttheta_u}\right]
\right\}\\
&\geq \frac{v}{2(1-\gamma)}\left(\frac{\varpi_c}{192}+\frac{\varpi_u}{24}\right)
\geqsim \min\left\{\frac{1}{1-\gamma},\sqrt{\frac{S_cK+S_u}{(1-\gamma)^3N}}\right\}
\geqsim \min\left\{\frac{1}{1-\gamma},\sqrt{\frac{\minop{SA, S+I}}{(1-\gamma)^3N}}\right\}.
\end{align*}

In conclusion, for a fixed algorithm $\mathfrak{A}$, there exists some $\theta\in\varTheta_c\times \varTheta_u$, such that for the policy $\hpi$ output by $\mathfrak{A}$ on $\cM_{\theta}$, either
\[
  \EE_{\cM_\theta}\left[J^*_{\theta}-J(\hpi)\right]\geqsim \min\left\{\frac{1}{1-\gamma},\sqrt{\frac{\minop{SA, S+I}}{(1-\gamma)^3N}}\right\},
\]
or 
\[\EE_{\cM_\theta}\left[\violation(\hpi)\right]\geqsim 1.\]
This completes the proof of \cref{thm:lower-bound-demo}.

\begin{remark}
  The family $\left(\cM_{\theta}\right)$ constructed here does not satisfy the Slater's condition with $\varphi=\Theta(1)$, but a small modification can be made to ensure a $\varphi$ with constant order. Namely, at each $s_0^j$ we add two extra arms $e,e'$, such that $r(s_0^j,e)=0, r(s_0^j,e')=-1$ and all utilities of $e'$ is 1. The transition at $s_0^j$ is not affected by $e,e'$. We omit this construction in the argument above for the sake of cleanness and simplicity.
\end{remark}

\subsection{Proof of Theorem \ref{theorem:slater}}

We further extend the idea of construction in \cref{appdx:LowerBound-1} to show that, when the Slater's condition does not hold, no zero constraint violation can be ensured. Intuitively, we can directly include an extra constraint $J(\pi)\geq J^*$ in the previous construction. However, the subtlety in such a transfer is that, the constraint will leak information of the underlying parameters $\theta, \varpi$. Thus, rather than making ad hoc adaption from \cref{appdx:LowerBound-1}, we present a more interesting construction for the case $I=1$, as follows.

\paragraph{States and actions} We take the state space $\cS=\left\{s_{-1}, s_0, \sp, \sm\right\}\bigsqcup_{j=1}^S \left\{s^j\right\}$, with actions and transition dynamic specified as follows. Here we merge the states $s_0^j, \sp^j, \sm^j$ in \cref{appdx:LowerBound-1} for notational simplicity. The transition dynamic is parametrized by $\theta\in\{0,1\}^S$ and $\varpi\in(0, \frac{1}{2}]$, as follows.
\begin{itemize}
    \item At $s_0$, $\sp$ and $\sm$, there is no action, and the transition does not depend on $\theta$:
    \begin{equation}
    \begin{aligned}
      & \condP{s_0}{s_0}=p,
      &&\condP{s^j}{s_0}=\frac{1-p}{S},\ j\in[S],\\
      & \condP{\sp}{\sp}=q, 
      &&\condP{s_0}{\sp}=1-q,\\
      & \condP{\sm}{\sm}=q, 
      &&\condP{s_0}{\sm}=1-q,
    \end{aligned}
    \end{equation}
      where $p=\frac{1}{2-\gamma}$ and $q=2-\frac{1}{\gamma}$. 
    \item At $s^j$, there are two actions $a,b$ such that
    \[\begin{aligned}
    &\condPt{\sp}{s^j,a}=\frac{1-\varpi\theta_{j}}{2}, 
    &&\condPt{\sm}{s^j,a}=\frac{1+\varpi\theta_{j}}{2}, \\
    &\condPt{\sp}{s^j,b}=\frac{1-\varpi(1-\theta_{j})}{2}, 
    &&\condPt{\sm}{s^j,b}=\frac{1+\varpi(1-\theta_{j})}{2}.
    \end{aligned}\]
    \item The null state $s_{-1}$ always transits to itself.
\end{itemize}
  
\paragraph{Utilities and rewards} We assign $u(\sp)=+1,u(\sm)=-1$, and $u(s_0)=u(s^j)=0$. No reward is assigned to $\cM$, namely the only goal in $\cM$ is to fulfill the constraint: $J^u(\pi)\geq 0$. Basically, this constraint requires us to determine whether $\theta_i=1$ for each $i$.

\newcommand{\ovr}{\overline{r}}

\paragraph{Optimal policy} For any policy $\pi$, we define
\[\begin{aligned}
  r_j(\pi)=\theta_j\pi(a|s_1^j)+(1-\theta_j)\pi(b|s_1^j),\qquad \ovr(\pi)=\frac{1}{S}\sum_{j=1}^S r_j(\pi).
\end{aligned}\]
Then by exactly the same calculation as in \cref{appdx:LowerBound-1}, we have
\begin{align*}
&\nu_{\pi}(s_0)=\frac{v_0}{1-\gamma},
&&\nu_{\pi}(s^j)=\frac{v_1}{S},\\
&\nu_{\pi}(\sp)=\frac{1-\varpi\overline{r}(\pi)}{2}\frac{v}{1-\gamma},
&&\nu_{\pi}(\sm)=\frac{1+\varpi\overline{r}(\pi)}{2}\frac{v}{1-\gamma}.
\end{align*}
Therefore, it holds that
\begin{equation}\label{eqn:lower-bound-2-est}
\begin{aligned}
  J^u(\pi)=-\frac{v}{1-\gamma}\ovr(\pi)=-\frac{v}{S(1-\gamma)}\nrm{\pi_b-\theta}_1,
\end{aligned}
\end{equation}
where we denote $\pi_b=\left(\pi(b|s_1^j)\right)_j$ for a policy $\pi$. Hence, there is a unique safe policy $\pi^{*,\theta}$ in $\cM_{\theta}$ that can be specified by
\[\begin{aligned}
  \pi^{*,\theta}(a|s^j)=1-\theta_j, \qquad \pi^{*,\theta}(b|s^j)=\theta_j,\qquad j\in[S].
\end{aligned}\]
The formula \eqref{eqn:lower-bound-2-est} also indicates that, for $\hpi$ outputed by an algorithm $\mathfrak{A}$ after consuming $N$ samples, the vector $\hpi_b$ can be viewed as an estimator of $\theta$.

\paragraph{Reference distribution} We take $\rho_0(s_0)=1$. The reference distribution $\mu$ is chosen similar to \cref{appdx:LowerBound-1}, namely
  \[\begin{aligned}
  &\mu(s_0)=\frac{v_0}{C},\qquad\qquad\qquad\mu(s^j,a)=\mu(s^j,b)=\frac{v_1(1-\gamma)}{SC},\\
  &\mu(\sp)=\mu(\sm)=\frac{v}{C},\qquad\,\,
  \mu(s_{-1})=1-\mu(s_0)-\mu(\sp)-\mu(\sm)-\sum_{j} \mu(s^j).
  \end{aligned}\]
As long as $C\geq 2$, $\mu(s_{-1})$ defined above is positive. Also, for any $\theta$, it holds that
  \[\max_{s,a}\frac{\nu^{\pi^{*,\theta}}(s,a)}{\mu(s,a)}\leq \frac{C}{1-\gamma}, \quad \sum_{s,a}\frac{\nu^{\pi^{*,\theta}}(s,a)}{\mu(s,a)}\leq \frac{(\nS+1) C}{1-\gamma}.\]

\paragraph{Lower bound} Still, we take a subset $\varTheta$ of $\{0,1\}^{S}$ such that $|\varTheta| \geq \exp (S / 8)$, and for any pair of different $\theta, \theta'\in \varTheta$ it holds $\left\|\theta-\theta'\right\|_{1} \geq \frac{S}{4}$. We next consider the family of CMDPs $\mathfrak{M}:=\left\{\cM_{\theta}\right\}_{\theta\in\varTheta}$, with the reference $\mu_{\theta}=\mu\otimes\PP_{\theta}$. 

By Fano's inequality (\cref{lemma:fano}), it holds that
\[\begin{aligned}
  \max_{\theta\in\varTheta}\EE_{\theta}\nrm{\hpi_b-\theta}_1
  &\geq \frac{S}{4}\left(1-\frac{N\max_{\theta,\theta'\in\varTheta}\KL{\mu_{\theta}}{\mu_{\theta'}}+\log 2}{\log |\varTheta|}\right).
\end{aligned}\]
We also have $\max_{\theta,\theta'\in\varTheta}\KL{\mu_{\theta}}{\mu_{\theta'}}\leq \frac{2\varpi^2(1-\gamma)}{C}$ by a simple calculation.
Therefore, taking $\varpi=\min\left\{\sqrt{\frac{(S-3)C}{16(1-\gamma)N}},\frac12\right\}$ is enough to ensure $\max_{\theta\in\varTheta}\EE_{\theta}\nrm{\hpi_b-\theta}_1\geq \frac{S}{24}$. Hence, we obtain
\[\begin{aligned}
  \max_{\theta\in\varTheta}\EE_{\theta}\neg{J^u(\hpi;\theta)}
  =\max_{\theta\in\varTheta}\EE_{\theta}\left[\frac{v\varpi}{S(1-\gamma)}\nrm{\hpi_b-\theta}_1\right]
  \geq \frac{v\varpi}{24(1-\gamma)}
  \geqsim \min\left\{\sqrt{\frac{SC}{(1-\gamma)^3N}},\frac{1}{1-\gamma}\right\}.
\end{aligned}\]


\section{The \adaptiveAlg\ framework}
\label{appdx:adaptive}
\subsection{The verification method}\label{subsect:verify}
First, let us provide the details of the $\verify(\cdot)$ method that is used in \cref{algo:adaptive-algB-demo}. 

\begin{algorithm}[h]
	\caption{$\verify(\ox)=\verify(\ox;\epsilon,\delta)$}\label{algo:verify}
	\Input{
		The output $\ox$ and the parameters $\epsilon,\delta>0$ in  \cref{algo:algB}.\vspace{0.1cm}
	} 
	Obtain $N_v$ offline samples $\big\{(s_t,a_t,s_t',r_t,\bu_t)\big\}_{t=1}^{N_v}$ from $\cD$\;\vspace{0.1cm}
	Compute the estimators $\hJ(\opi), \hJuk(\opi)\in\RR$ and $\widehat{\Delta}_p\in\RR^{\nS}$ as \vspace{-0.1cm}
	\begin{equation}
			\hJ(\opi):=\frac{1}{N_v}\sum_{t=1}^{N_v} r_t\frac{\ox(s_t,a_t)}{\hmu(s_t,a_t)},\quad\qquad
			\hJuk(\opi):=\frac{1}{N_v}\sum_{t=1}^{N_v} \bu_t^{\kappa}\frac{\ox(s_t,a_t)}{\hmu(s_t,a_t)},\nonumber
	\end{equation}
	\begin{equation}
			\widehat{\Delta}_p(s'):=
			\sum_{a} \frac{N(s',a)}{N_v}\frac{\ox(s',a)}{\hmu(s',a)}-\gamma\sum_{s,a}\frac{N(s,a,s')}{N_v}\frac{\ox(s,a)}{\hmu(s,a)}-\rho_0(s'),\,\,\,\, \forall s'\in\cS,\nonumber
	\end{equation}
	where $N(s,\!a), N(s,\!a,\!s')$ are the times that $(s,\!a)$ and $(s,\!a,\!s')$ are observed in the $N_v$ samples.\vspace{0.2cm}
	
	\eIf{$\,\,\,\,\big\|\hdelta_p\big\|_1\leq \frac{3}{2}\varphi(1-\gamma)\epsilon\,\,\,\,\&\&\,\,\,\,\big\|\hJuk(\opi)\big\|_{\infty}\leq 3\varphi\epsilon$\vspace{0.1cm}}{
		Return $\verify(\ox)=\true$, and return $\hJ(\opi)$ as an estimate of $J(\opi)$\;
	}{
		Return $\verify(x)=\false$\;
	}
\end{algorithm}

As a remark, $\widehat{\Delta}_p$ is an estimator of the residual $A^\top W\ox-\rho_0$, where $A$ is defined in \eqref{eqn:A}, that is $\EE_{\cD}\big[\widehat{\Delta}_p\big] =  A^\top W\ox-\rho_0$. By a direct computation, we also know $\EE_{\cD}\big[\hJ(\opi)\big] = r^\top W\ox$ and  $\EE_{\cD}\big[\hJuk(\pi)\big] = U_\kappa W\ox$. Intuitively, when $\|\widehat{\Delta}_p\|_1$ is small, then $W\ox$ is a good approximation of $\nu^{\opi}$ and thus $\hJ(\opi),\hJuk(\opi)$ are good approximations of $J(\opi),J^{u^\kappa}(\opi)$. With this in mind, we present the following proposition that characterizes the \verify\ method, whose proof is moved to \cref{appdx:pf-verify}.

\begin{proposition}\label{prop:verify}
	For the \verify\ method, if we choose $N_v\geq \frac{64\nS\hphi\ell}{\varphi^2(1-\gamma)^4\epsver^2}$, with $\ell=4\log\left(\frac{40\nS I}{\delta}\right)$ and $\epsver=\frac{\epsilon}{10}$, then \highprobdemo, it holds that:
	
	(1). If $\verify(\outx)=\false$, then $\hphi<C^*$.
	
	(2). If $\verify(\ox)=\true$, then $ J^{u^\kappa}(\opi)\geq 0$, and $\,\,j_0(\hphi)-400\epsilon\leq \hJ(\opi)\leq j_0(\hphi)+100\epsilon.$
\end{proposition}

Basically, this proposition states that if $\verify(\outx)=\false$, then we know $\hphi<C^*$ with high probability. If $\verify(\ox)=\true$, then we know that $\opi$ is safe, and $j_0(\hphi) = \hJ(\opi) +\mathcal{O}(\epsilon)$. We can apply Lemma \ref{lemma:adaptive-j0-error-bound-demo} to determine whether the current policy is good enough.

\subsection{The adaptive-\algB\ method} 
\label{subsect:adaptive}
In this section we will discuss the details of \cref{algo:adaptive-algB-demo}. The key to the analysis of this section is \cref{lemma:adaptive-j0-error-bound-demo}, whose proof is presented in \cref{appdx:subsec-adaptive-proof}.

\paragraph{Setting of sub-routine} We use $\epsilon'$ for the input sub-optimality of \adaptiveAlg. At each step $K$, we call \algB\ and $\verify$ with $\epsilon=\frac{\epsilon'}{15}$ and $\delta_K:=\frac{6\delta}{\pi^2K^2}$. The $\delta_K$ is chosen so that $\sum_K \delta_K=\delta$.

\paragraph{Exit condition} In \cref{algo:adaptive-algB-demo}, line 4 to 6, we write the exit condition as $-\infty <J^K\leq J^{K-1} +\mathcal{O}(\epsilon)$. More specifically, the exit condition can be equivalently stated as
\begin{equation}\label{eqn:def-exit-cond}
	\verify\big(\xk\big) \,\,\,\&\&\,\,\,
	\verify\big(\xks\big)  
	\,\,\,\&\&\,\,\, 
	\hJ\big(\pik\big)\!-\!\hJ\big(\piks\big)\!\leq 500\epsilon.
\end{equation}
Here the third condition only needs to be checked when both $\verify\big(\xk\big)$ and $\verify\big(\xks\big)$ return \true. The constant $500$ is chosen to ensure that \adaptiveAlg\ will exit for $\psi_K>2C^*$, as will be demonstrated in the following  proposition, whose proof is presented in \cref{appdx:prop-adp}.


\begin{proposition}\label{prop:adp}
	Suppose \cref{algo:adaptive-algB-demo} exits at step $K$. Then with probability at least $1\!-\!\delta$, the following results hold.
	(1) $\pik$ is safe and $\hphi_K\leq 4C^*$. 
	(2) It holds that $J^*-J(\pik)\leq \mathcal{O}\big(\frac{C^*}{\hphi_K}\epsilon\big)$. 
	(3) There is a constant $\epsilon_0(\cM)$ such that for $\epsilon'\leq \epsilon_0(\cM)$, $\hphi_K\geq C^*$.
\end{proposition}

As a remark, $\epsilon_0$ is (up to a scalar factor) the minimum performance improvement by increasing $\psi\to 2\psi$, and the minimum of slope of $j$ as a function of $\log\hphi$ for $\hphi\in[1,C^*]$. Therefore, when \adaptiveAlg\ exits at some step $K$, the improvement that can be achieved by increase $\hphi$ grows as at most $\frac{\epsilon_0}{\hphi_K}$. If in this case $\hphi_K$ is still far small from $C^*$, then the difficulty essentially comes from a prohibitively large $C^*$.

\paragraph{Sample complexity of \adaptiveAlg} At step $K$, the samples needed for \algB\ are $\tbO{\frac{\cN\hphi_K}{\varphi^2(1-\gamma)^4\epsilon^2}}$, and the samples needed for verification are $\tbO{\frac{\nS\hphi_K}{(1-\gamma)^4\epsilon^2}}$. There are at most $\ceil{\log_2(C^*/\hphi^1)}+1$ outer steps and the $\psi_K$ is twofold at each step, thus the total samples needed are $\tbO{\frac{\cN\hphi_K}{\varphi^2(1-\gamma)^4\epsilon^2}}$ if it exits at step $K$. Especially, as long as $\epsilon\leq \epsilon_0(\cM)$, \adaptiveAlg\ ends after consuming $\tbO{\frac{\cN C^*}{\varphi^2(1-\gamma)^4\epsilon^2}}$ samples and outputs a policy which is safe and $\bigO{\epsilon}$-optimal.

\subsection{Proof of Proposition \ref{prop:verify}}
\label{appdx:pf-verify}
\begin{proof}
First, we provide the following lemma for the estimators $\widehat{\Delta}_p, \hJ(\opi)$ and $\hJuk(\opi)$. The calculation of \cref{lemma:verify-error}  is very closed to \cref{subsect:empirical-dis}, and is thus omitted. 
\begin{lemma}\label{lemma:verify-error}
	Suppose that $N_v$ and $\epsilon_{\text{ver}}$ are chosen according to \cref{prop:verify}. Denote $\onu=W\ox$, then \highprobs{3}, we have 
	$$\max\left\{\big\|\hdelta_p-(A^\top\onu-\rho_0)\big\|_1, \big|\hJ(\opi)-\iprod{r}{\onu}\big|, \big\|\hJuk(\opi)-U_{\kappa}\onu\big\|_\infty \right\}\leq\varphi(1-\gamma)\epsilon_{\text{ver}}.$$
\end{lemma}

	
	

\paragraph{Proof of the case $\verify(\outx)=\false$.}
	By \cref{corollary:regret-ox-ov-v}, it holds that when $\hphi\geq C^*$,
	\begin{equation} 
			\|A^\top\onu - \rho_0\|_1\leq \frac{11}{8}\varphi(1-\gamma)\epsilon,\quad \mbox{and}\quad \|\neg{U_{\kappa}\onu}\|_{\infty}\leq \frac{11}{4}\varphi\epsilon.\nonumber
	\end{equation}
	Combining the above inequality with \cref{lemma:verify-error} indicates that $\|\hdelta_p\|_1\leq \frac{3}{2}\varphi(1-\gamma)\epsilon$ and $\|[\hJuk(\opi)]_{-}\|_\infty\leq 3\varphi\epsilon$. This contradicts the condition for returning $\false$. Therefore, we know that $\hphi<C^*$.

\paragraph{Proof of the case $\verify(\outx)=\true$.} 
	By the condition for returning \true, we know $\|\hdelta_p\|_1\leq \frac{3}{2}\varphi(1-\gamma)\epsilon$ and $\|[\hJuk(\opi)]_{-}\|_\infty\leq 3\varphi\epsilon$. Together with \cref{lemma:verify-error}, we have
	\begin{equation} 
			\|A^\top\onu -\rho_0\|_1\leq 1.6\varphi(1-\gamma)\epsilon \qquad\mbox{and}\qquad  \nrm{\neg{U_{\kappa}\onu}}_{\infty}\leq 3.1\varphi\epsilon. \nonumber
	\end{equation}
	Similar to our analysis in \cref{subsect:duality-to-regret}, we write $\nu^{\opi}$ the true visitation measure of $\opi$. Then by \eqref{eqn:regret-true-to-fake},
	\begin{equation*}
		\begin{aligned}
			\nrm{\nu^{\onu}-\onu}_1&\leq\frac{1}{1-\gamma}\left(\nrm{A^\top\onu-\rho_0}_1+3\nrm{\ox-\onu}_1\right)\leq  1.63\varphi\epsilon. 
		\end{aligned}
	\end{equation*}
	where the term $\|\ox-\onu\|_1$ is controlled by \cref{prop:empirical-dis}. Due to the fact that $\big|\big\|\neg{U_{\kappa}\nu^{\opi}}\big\|_{\infty} - \big\|\neg{U_{\kappa}\onu}\big\|_{\infty}\big|\leq (1+5(1-\gamma)\varphi\epsilon)\nrm{\nu^{\opi}-\onu}_1\leq1.1\nrm{\nu^{\opi}-\onu}_1$ for small $\epsilon\leq \frac{1}{50(1-\gamma)}$, it holds that 
	\begin{equation}
		\begin{aligned}
			\min_i J^u_i(\opi) \geq \kappa- \big\|\neg{U_{\kappa}\nu^{\opi}}\big\|_{\infty}\geq \kappa-\nrm{\neg{U_{\kappa}\onu}}_{\infty}-1.1\nrm{\nu^{\opi}-\onu}_1\geq0.\\
		\end{aligned}
	\end{equation}
	Moreover, the definition of $\Gap(\ox)$ gives
	\[
	j(\hphi)-\Gap(\ox)=\iprod{r}{\onu}-\RV \nrm{A^\top\onu-\rho_0}_1-\RLam \nrm{[U_{\kappa}\onu]_-}_\infty\leq j(\hphi)\leq j_0(\hphi),
	\]
	which yields
	\begin{equation*}
		\begin{aligned}
			&\iprod{r}{\onu}\leq j_0(\hphi)+\RV \nrm{A^\top\onu-\rho_0}_1+\RLam \nrm{[U_{\kappa}\onu]_-}_\infty \leq j_0(\hphi)+100\epsilon,\\
			&\iprod{r}{\onu}\geq j(\hphi)-\Gap(\ox)\geq j_0(\hphi)-400\epsilon,
		\end{aligned}
	\end{equation*}
	where we use the fact that $0 \leq j_0(\hphi)-j_\kappa(\hphi)\leq \frac{64\kappa}{\varphi}=320\epsilon$. The same bound for $\hJ(\opi)$ can be derived by \cref{lemma:verify-error}. 
 \end{proof}

\subsection{Proof of Lemma \ref{lemma:adaptive-j0-error-bound-demo}} \label{appdx:subsec-adaptive-proof}
\begin{proof} 
	First we show that, when $\hphi<C^*$, $j_0(\hphi)< J^*$. Otherwise, for $x_*=\argmax_{x\in\cX} \cJ_0(x)$, it holds that $\cJ_0(x_*)=j_0(\hphi)\geq J^*$, i.e., for $\nu_*=Wx_*$,
	\[J^*-\iprod{r}{\nu_*}+\RV \nrm{A^\top  \nu_*-\rho_0}_1+\RLam\nrm{\left[U\nu_*\right]_{-}}_{\infty}\leq 0.\]
	Applying \cref{lemma:duality-gap-decomp} gives $\nrm{A^\top  \nu_*-\rho_0}_1\leq 0$, $\nrm{\left[U\nu_*\right]_{-}}_{\infty}\leq 0$, $J^*-\iprod{r}{\nu_*}\leq 0$. Thus, $\nu\in\vms\cap\safe$, and $\iprod{r}{\nu_*}\geq J^*$, which imply that $\nu_*$ is indeed an optimal solution of problem \eqref{prob:CMDP-LP}. However, $\nu_*\in W\cX\Rightarrow \nu_*\in\vms(\hphi)\Rightarrow \hphi\geq C^*$, a contradiction.
	
	Now the monotonicity is easy. We still fix an optimal $\nu_*\in\vms(C^*)$ and let $x_*:=W^{-1}\nu_*$. For $1\leq \hphi'<\hphi$, we write $x_{\hphi'}=\argmax_{x\in\cX(\hphi')}\cJ_0(x)$, $c=\frac{\hphi-\hphi'}{C^*-\hphi'}$, and we consider $x_{\hphi}:=cx_{*}+(1-c)x_{\hphi'}\in\cX(\hphi)$. It holds that
	\[j_0(\hphi)\geq \cJ_0(x_{\hphi})\geq (1-c)\cJ_0(x_{\hphi'})+c\cJ_0(x_*)=j_0(\hphi')+c(J^*-j_0(\hphi)).\]
	The proof is completed by reorganizing the above inequality.
\end{proof}

\subsection{Proof of Proposition \ref{prop:adp}}
\label{appdx:prop-adp}
\begin{proof}
	By \cref{prop:verify}, if $\hphi_{K}\geq 2C^*$, then $\hphi_{K-1}\geq C^*$. By \cref{prop:verify}, \highprobdemo\ it holds that $\verify\big(\xk\big)=\verify\big(\xks\big)=\true$, and
	\begin{equation*}
		\begin{aligned}
			&\hJ(\opi^{(K)}), \hJ(\opi^{(K-1)}) \in \left[J(\pi^*)-400\epsilon, J(\pi^*)+100\epsilon\right],\\
			\Rightarrow & \abs{\hJ(\opi^{(K)})-\hJ(\opi^{(K-1)})}\leq 500\epsilon,
		\end{aligned}
	\end{equation*}
	where we use the fact $j_0(\hphi_K)=j_0(\hphi_{K-1})=j_0(C^*)=J(\pi^*)$ from \cref{lemma:adaptive-j0-error-bound-demo}. Therefore, if $\hphi_{K}\geq 2C^*$, \adaptiveAlg\ must exit at step $K$. 
	
	Now, we only need to consider the case that \adaptiveAlg\ ends at some step $K$, but $\psi_K$ might not be greater than $C^*$. Because $\verify\big(\xk\big)=\verify\big(\xks\big)=\true$, we combine the exit condition \eqref{eqn:def-exit-cond} with \cref{prop:verify}  and derive $\abs{j_0(\hphi_K)-j_0(\hphi_{K-1})}\leq 1000\epsilon.$
	Then by \cref{lemma:adaptive-j0-error-bound-demo}, we have
	\[
	J(\pi^*)-j_0(\hphi_K)\leq \frac{2(C^*-\hphi_K)}{\hphi_K}\left(j_0(\hphi_K)-j_0(\hphi_{K-1})\right).
	\]
	Thus $J(\pi^*)-J(\pik)\leqsim \frac{C^*}{\hphi_K}\epsilon$.  Furthermore, we can define the following quantity
	\[
	\epsilon_0:=\min_{1\leq \hphi\leq C^*} \left(j_0(\hphi)-j_0(\frac{\hphi}{2})\right)>0.
	\]
	Here $\epsilon_0>0$ is due to \cref{lemma:adaptive-j0-error-bound-demo}. If $j_0(\hphi)-j(\frac{\hphi}{2})<\epsilon_0$ for some $\hphi\geq 1$, then immediately we have $\hphi\geq C^*$. If $\epsilon'=15\epsilon\leq \epsilon_0/100=:\epsilon_0(\cM)$,  \adaptiveAlg\ must exit at step $K$ with $C^*\leq \hphi_K\leq 4C^*$. By \cref{thm:algB-thm-demo}, the output policy $\pik$ is safe and $J(\pi^*)-J(\pik)\leq \epsilon'$.
\end{proof}


\section{Convergence Analysis in Asynchronous Setting}
\label{appdx:asyn}
\subsection{Mixing property of Markov chain}\label{appdx:markov-mixing}
Under the setting of the asynchronous learning (\cref{assumption:AsyncData}), we can observe a sequence of state-action trajectory generated under the behavioral policy $\pib$, namely
\[s_1,a_1,s_2,a_2,s_3,\cdots,s_n,a_n,s_{n+1},\cdots.\]
This sequence can be naturally viewed as a Markov chain $(X_t)_{t\geq1}$ where $X_t=(s_t)$, plus a marginal component $a_t\in\cA$. In the asynchronous setting, the reference distribution $\mu$ is the stationary distribution $\mub$ of this chain product with the policy $\pib$. 
As in the synchronous setting, we denote $\cF_t$ for all the history information at time $t$. Actually, by the Markov property and our update rule, conditioning on $\cF_t$ is equivalent to conditioning on $s_t,Z^t$. 
According to \cite[Section 4]{markov-chain-book}, we define the mixing time of this Markov chain as 
\begin{equation}
	\label{def:mixing-time}
	\begin{cases}
		\mathcal{E}(t):=\sup _{s \in \cS} d_{\mathrm{TV}}\left(\mub,\PP^t_{\pib}(\cdot|s_0=s)\right),\\
		t_{\mix}:= \min\{t: \mathcal{E}(t)\leq\frac{1}{4}\},
	\end{cases}
\end{equation}
where $\PP^t_{\pib}(\cdot|s_0=s)$ denotes the distribution of $s_t$ given $s_0=s$ and policy $\pib$.
By \cite[Remark 4.12]{markov-chain-book}, it holds that
\[\mathcal{E}(t)\leq 2^{-\floor{\frac{t}{t_{\mix}}}}.\]

Given the concept of the mixing time, we modify the standard Bernstein inequality for Markov chain \cite[etc.]{Bernstein-Markov,paulin2015concentration} to cover the non-stationary Markov chains.
\begin{restatable}{proposition}{Bernstein}\label{prop:bernstein-markov}
	Suppose that $\left(X_{t}\right)_{t \geq 1}$ is a Markov chain with invariant distribution $\pi$ and mixing time $t_{\mix}<+\infty$. Let $f$ be a measurable function such that $\EE_{\pi}\left[f(X)\right]=0$, $\abs{f(X)}\leq M$. Denote $\sigma^2=\EE_{\pi}\left[f(X)^2\right]$, then for $\delta\in(0,1)$, the following holds \highprobdemo
	\[\begin{aligned}
		\abs{\sum_{t=1}^n f(X_t)}\leq \sqrt{32t_{\mix}n\sigma^2\log\frac{4}{\delta}}+82t_{\mix}M\log\frac{4}{\delta}.
	\end{aligned}\]
\end{restatable}



The difficulty of analyzing Markovian gradients is the correlation between updates and samples. 
As demonstrated in \cref{subsect:async}, in our analysis, we leverage the fact that $s_{t+\tau}$ is a sample ``almost'' from $\mub$ and ``almost'' independent of $s_{t}$, as long as $\tau\geq t_\mix\cdot \log\text{factor}$. We further demonstrate this idea in the following proposition, by comparing $\cond{\hg(\cdot;\zeta_{t+\tau})}{\cF_t}$ and $\Gr(Z^t)$. 

\begin{proposition}[Almost unbiased]\label{prop:markov-diff-grad}
  For a $\cF_t$-measurable random variable $Z\in\cZ:=\cV\times\Lambda\times\cX$, it holds that
  \[\begin{aligned}
  \nrm{\cond{\hg_V(Z;\zeta_{t+\tau})}{\cF_t}-\gLv(Z)}_{1}&\leq \frac{2\hphi}{1-\gamma}\Eps(\tau),\\
  \nrm{\cond{\hg_\lambda(Z;\zeta_{t+\tau})}{\cF_t}-\gLl(Z)}_{\infty}&\leq \frac{2\hphi}{1-\gamma}\Eps(\tau),\\
  \nrm{\cond{\hg_x(Z;\zeta_{t+\tau})}{\cF_t}-\gLx(Z)}_{\infty}&\leq \frac{64}{\varphi(1-\gamma)\varsigma}\Eps(\tau).
  \end{aligned}\]
  Furthermore, for any $Z'\in\cZ$, we have
  \[\begin{aligned}
  \abs{\iprod{Z'}{\Gr(Z)-\cond{\hg(Z;\zeta_{t+\tau})}{\cF_t}}}\leq \frac{128\hphi}{\varphi(1-\gamma)^2}\Eps(\tau).
  \end{aligned}\]
\end{proposition}

The following proposition indicates that, the estimator $\hg(\cdot;\zeta_{t+\tau})$ is not only ``nearly unbiased'' conditional on $\cF_t$, but it also has a well bounded moment.
\begin{proposition}[Bounded moment]\label{prop:markov-bounded-moment}
  For any $\cF_t$-measurable random variable $Z\in\cZ$, it holds that
  \[\begin{aligned}
    &\cond{\nrm{\hg_V(Z;\zeta_{t+\tau})}}{\cF_t}\leqsim \frac{C(\tau)}{1-\gamma},\\
    &\cond{\nrm{\hg_\lambda(Z;\zeta_{t+\tau})}_{\infty}}{\cF_t}\leqsim \frac{C(\tau)}{1-\gamma},\\
    &\cond{\nrm{\hg_x(Z;\zeta_{t+\tau})}_{x^t}^2}{\cF_t}\leqsim \frac{C(\tau)\cN \hphi}{\varphi^2(1-\gamma)^3},
  \end{aligned}\]
  where $C(\tau)=2+\frac{\Eps(\tau)}{\varsigma}$.
\end{proposition}
Therefore, there is a universal constant $c_\tau$ such that for $\tau\geq \floor{c_\tau t_{\mix} \clog}$, we have $C(\tau)\leq 3$ and $\Eps(\tau)\leq\frac{1}{T}$ (the $\log$ factor $\clog$ and the range of $T$ are specified in \cref{thm:algB-gap-markov}). We denote $\tau_0=\floor{c_\tau t_{\mix} \clog}$.






\subsection{Proof sketch of Theorem \ref{thm:algB-thm-markov-demo}}\label{sect:gda-markov}

Before our analysis of \algB\ on $\cD_{async}$, we have to first provide an analogue of \cref{prop:empirical-dis}. As in the synchronous setting, we set $\epsilon_e=\frac{\epsilon}{100}$ and $\varsigma=\frac{\varphi(1-\gamma)^2\epsilon_e}{2\cN\hphi}$.
\begin{proposition}\label{prop:empirical-dis-markov}
  Given $N_e\geq c'_e\frac{t_{\mix}\cN\hphi\clog}{\varphi^2(1-\gamma)^4\epsilon_e^2}$ samples from a trajectory generated by $\pib$, the $\hmu$ constructed in \eqref{eqn:def-empirical-dis} satisfies the following properties \highprobs{3}.\\
  (1) For all $s,a$, $\frac{\mu(s,a)}{\hmu(s,a)}\leq 2$, and $\hmu(s,a)\geq \varsigma$.\\
  (2) For any $\pi\in\Pi(\hphi)$, $W^{-1}\nu^{\pi}\in\cX$.\\
  (3) For any $x\in\cX$, $\nrm{Wx-x}_1\leq \varphi(1-\gamma)\epsilon_e$.
\end{proposition}

Now, we present the convergence guarantee of the duality gap $\Gap(\ox)$.
\begin{theorem}\label{thm:algB-gap-markov}
  Given $\epsilon\in\left(0,\frac{1}{1-\gamma}\right]$, $\delta\in\left(0,\frac12\right)$, we denote $\clog=\log\left(T\nS\nA I/\delta\right)$. Then as long as $T\geqsim \frac{\tau_0^2 \cN\hphi\clog}{\varphi^2(1-\gamma)^4\epsilon_e^2}$, \highprobs{3}\ it holds 
  \[
    \Gap(\ox)\leqsim \frac{t_\mix}{\varphi(1-\gamma)^2}\sqrt{\frac{\cN\hphi\clog^3}{T}}\leq \epsilon.
  \]
\end{theorem}
Therefore, there is a universal constant $c_o'$ such that $\Gap(\ox)\leq\frac{\epsilon}{2}$ as long as $T\geq c_o' \frac{t^2_\mix \cN\hphi\clog^3}{\varphi^2(1-\gamma)^4\epsilon^2}$. Then the proof in \cref{subsect:duality-to-regret} can be applied directly. In conclusion, the number of samples needed is
\[\tbO{\frac{t_\mix^2 \cN\hphi}{\varphi^2(1-\gamma)^4\epsilon^2}}.\]

We sketch the proof of \cref{thm:algB-gap-markov} as follows. The detailed proofs of propositions are organized by order in the rest of this section.

\paragraph{Decomposition of duality gap} We define the auxiliary variables $V',\lambda',x'$ as in \cref{appdx:ThmGap},
\begin{equation*}
  (V',\lambda')=\argmin_{V\in\cV, \lambda\in\Lambda}\cL_w(V,\lambda,\ox),\quad
  x'=\argmax_{x\in\cX} \min_{V\in\cV,\lambda\in\Lambda} \cL_w(V,\lambda,x),\quad
  Z'=[V';\lambda';x'].
\end{equation*}
Recall the decomposition \eqref{eqn:proof-decomp}, we have
\begin{align*}
\Gap(\ox)
=
\underbrace{
    \frac1T\sum_{t=1}^{T} \iprod{\hg(Z^t;\zeta_{t})}{Z^t-Z'}
}_{S_1}+\underbrace{
    \frac1T\sum_{t=1}^{T} \iprod{\Gr(Z^t)-\hg(Z^t;\zeta_{t})}{Z^t-Z'}
}_{S_2}.
\end{align*}

\paragraph{Bounding the term $S_1$} The proof in \cref{appdx:S1} can be applied without change. Namely, as long as $\eta\leq \frac{1}{2}\min\left(\frac{\alpha_\lambda}{M_\lambda}, \frac{\alpha_x}{M_{x,\infty}}\right)$, it holds that
\begin{equation*}\label{eqn:markov-part1}
\begin{aligned}
  S_1
	\leqsim 
	\frac{\alpha_V D^2_V+\alpha_\lambda D_{\lambda}+\alpha_x D_x}{\eta T}
	 +\frac{\eta}{T}\sum_{t=1}^{T}\left(\frac{\sqr{\hg_V(Z^t;\zeta_t)}}{\alpha_V}+\frac{D_{\lambda,1}\sqr{\hg_\lambda(Z^t;\zeta_t)}_{\infty}}{\alpha_\lambda}+\frac{\sqr{\hg_x(Z^t;\zeta_t)}_{x^t}}{\alpha_x}\right).
\end{aligned}
\end{equation*}

\paragraph{Bounding the term $S_2$} In the asynchronous setting, $\zeta_1,\cdots,\zeta_T$ are no longer i.i.d samples. To deal with this issue, let us 
consider the following decomposition
\begin{align*}
  \Gamma^t:=\iprod{\Gr(Z^t)-\hg(Z^t;\zeta_{t})}{Z^t-Z'}
  =& \underbrace{
    \iprod{\Gr(Z^t)}{Z^t-Z'}-\iprod{\Gr(Z^{t-\tau})}{Z^{t-\tau}-Z'}
  }_{\Gamma_1^t}\\
  &+ \underbrace{
    \iprod{\Gr(Z^{t-\tau})-\cond{\hg(Z^{t-\tau};\zeta_{t})}{\cF_{t-\tau}}}{Z^{t-\tau}-Z'}
  }_{\Gamma_2^{t-\tau}}\\
  &+ \underbrace{
    \iprod{\cond{\hg(Z^{t-\tau};\zeta_{t})}{\cF_{t-\tau}}-\hg(Z^{t-\tau};\zeta_{t})}{Z^{t-\tau}-Z'}
  }_{\Gamma_3^{t-\tau}}\\
  &+ \underbrace{
    \iprod{\hg(Z^{t-\tau};\zeta_{t})}{Z^{t-\tau}-Z'}-\iprod{\hg(Z^t;\zeta_{t})}{Z^t-Z'}
  }_{\Gamma_4^t},
\end{align*}
where $1\leq \tau\leq \tau_0$ is a fixed integer. The quantity $\Gamma_2^{t-\tau}$ can be bounded by \cref{prop:markov-diff-grad}, and $\Gamma_3^{t-\tau}$ can be bounded as in \cref{appdx:S2}. As of $\Gamma_1^t$, $\Gamma_4^t$, we bound it in terms of $Z^t-Z^{t-\tau}$. In conclusion, \highprobs{5}, we have
  \begin{align}\label{eqn:markov-S-2}
  \frac{1}{T}\sum_{t=1}^T \Gamma^t
  \leqsim &
  \frac{\hphi}{\varphi(1-\gamma)^2}\Eps(\tau)
  +\frac{1}{\varphi(1-\gamma)^2}\sqrt{\frac{\tau C(\tau)\cN\hphi\clog}{T}} +\frac{1}{\varphi(1-\gamma)}\sum_{t=\tau+1}^T\frac{\abs{x^t-x^{t-\tau}}(s_t,a_t)}{\hmu(s_t,a_t)}\notag\\
  &+\frac{\eta}{T}\sum_{t=\tau+1}^T \left(\frac{\nrm{\hg_V(Z^t;\zeta_t)}^2}{\alpha_V} + \frac{D_{\lambda,1} \nrm{\hg_\lambda(Z^t;\zeta_t)}_{\infty}^2}{\alpha_\lambda} \right).
  \end{align}
The detailed analysis is presented in \cref{appdx:markov-s2}.

\paragraph{Bounding the variance and magnitude of the updates} It remains to bound $\nrm{\hg_V(Z^t;\zeta_t)}$, $\nrm{\hg_\lambda(Z^t;\zeta_t)}_{\infty}$, $\nrm{\hg_x(Z^t;\zeta_t)}_{x^t}$, and the term $\abs{x^t(s_t,a_t)-x^{t-\tau}(s_t,a_t)}$. For any $x\in\RR_{\geq 0}^{\nS\nA}$, and any $(s,a)\in\cS\times\cA$,
we introduce the following abbreviation for the ease of notation 
\[\begin{aligned}
p(x;s,a):=\frac{x(s,a)}{\hmu(s,a)},\quad q(x;s,a):=\frac{x(s,a)}{\hmu(s,a)^2}.
\end{aligned}\]
For any sample $\zeta=(s_0,s,a,s',r,\bu)$, we also reload the notations $p,q$ as $p(x;\zeta):=p(x;s,a)$ and $q(x;\zeta):=q(x;s,a)$.

It is not hard to see that $p(x^t;\zeta_t)$ and $q(x^t;\zeta_t)$ dominate the variance of the gradient estimators (for detailed discussion, see \cref{appdx:markov-pq}). More specifically, we have
\begin{equation*}\label{eqn:markov-grad-norm}
  \begin{aligned}
    &\nrm{\hg_V(Z^t;\zeta_t)}\leqsim p(x^t;\zeta_t),\qquad
    \nrm{\hg_\lambda(Z^t;\zeta_t)}_{\infty}\leqsim p(x^t;\zeta_t),\\
    &\nrm{\hg_x(Z^t;\zeta_t)}_{x^{t}}\leqsim \frac{1}{\varphi(1-\gamma)}\sqrt{q(x^t;\zeta_t)}.
  \end{aligned}
\end{equation*}
Then, we only need to bound $\sum_{t=1}^T q(x^t;\zeta_t)$, $\sum_{t=\tau+1}^Tp(\abs{x^t-x^{t-\tau}};\zeta_t)$ and $\sum_{t=1}^T p(x^t;\zeta_t)^2$. By leveraging the idea of the decomposition \eqref{eqn:markov-decomp-demo}, we can derive the desired estimation, as follows.
\begin{proposition}\label{prop:markov-bound-p-q-tau-demo}
  There is a universal constant $c$ such that for $T\geq c\frac{\tau_0^2\cN\hphi\clog}{\varphi^2(1-\gamma)^4\epsilon_e^2}$, the following holds for all $1\leq\tau\leq \tau_0$ simultaneously, \highprobs{10}:
    \begin{align*}\label{eqn:markov-pq-sum-all}
      &\frac1T \sum_{t=1}^T p(x^t;\zeta_t)^2
      \leqsim \frac{\hphi}{(1-\gamma)^2},\\
      &\frac1T \sum_{t=1}^{T} q(x^t;\zeta_t)
      \leqsim \frac{\cN\hphi}{1-\gamma},\\
      &\frac1T \sum_{t=\tau+1}^{T} p(\abs{x^t-x^{t-\tau}};\zeta_{t})
      \leqsim \frac{\tau C(\tau)}{1-\gamma}\sqrt{\frac{\cN\hphi\clog}{T}}.
    \end{align*}
\end{proposition}

\paragraph{Conclusion}
Combining \cref{prop:markov-bound-p-q-tau-demo} with the estimations of $S_1$ and $S_2$, we have \highprobs{3},
\begin{equation}
\begin{aligned}
  \Gap(\ox)\leqsim \frac{\tau C(\tau)}{\varphi(1-\gamma)^2}\sqrt{\frac{\cN\hphi\clog}{T}}+\frac{\hphi}{\varphi(1-\gamma)^2}\Eps(\tau).
\end{aligned}
\end{equation}
Now, we can take $\tau=\tau_0=\floor{c_{\tau}t_{\mix}\clog}$. Then by the definition, it holds $C(\tau_0)\leq 3$ and $\epsilon(\tau_0)\leq \frac{1}{T}$, and hence \highprobs{3}\ we have
\[
    \Gap(\ox)\leqsim \frac{t_\mix}{\varphi(1-\gamma)^2}\sqrt{\frac{\cN\hphi\clog^3}{T}}.
\]

As a remark, if we have an (empirical) estimation $\htmx$ such that $\htmx\geq t_{\mix}$, then by taking $\eta=\frac{1}{\sqrt{\htmx T}}$, the final bound can be improved to $\Gap(\ox)\leqsim \frac{1}{\varphi(1-\gamma)^2}\sqrt{\frac{\htmx\cN\hphi\clog^3}{T}},$ as long as $T\geqsim \frac{t^2_\mix}{\htmx}\frac{\cN\hphi\clog^3}{\varphi^2(1-\gamma)^4\epsilon^2}$.


\subsection{Proof of Proposition \ref{prop:bernstein-markov}}

In order to prove \cref{prop:bernstein-markov}, we invoke the following standard version of the Bernstein's inequality. We also leverage the idea of the proof of \cite[Lemma 8]{Async-Q}.
\begin{theorem}[{\cite[Theorem 3.9]{paulin2015concentration}}]\label{thm:Bernstein-markov}
	Suppose $\left\{X_{i}\right\}_{i \geq 1}$ is a stationary Markov chain with invariant distribution $\pi$ and pseudo spectral gap $\gamma_{\mathrm{ps}}$. Let $f$ be a measurable function such that $\EE_{\pi}\left[f(X)\right]=0$, $\abs{f(X)}\leq M$. Denote $\sigma^2=\EE_{\pi}\left[f(X)^2\right]$, then for all $x\geq 0$,
	$$
	\mathbb{P}\left(\left|\sum_{i=1}^n f(X_i)\right| \geq x\right) \leq 2 \exp \left(-\frac{x^{2} \cdot \gamma_{\mathrm{ps}}}{8\left(n+1 / \gamma_{\mathrm{ps}}\right) \sigma^2+20 x M}\right).
	$$
	In particular, for uniformly ergodic chains with mixing time $t_{\mix}$, $\gamma_{\mathrm{ps}}\geq \frac{1}{2t_{\mix}}$.
\end{theorem}


\begin{proof}[Proof of \cref{prop:bernstein-markov}]
	Without loss of generality, we assume the Markov chain $(X_t)$ has a finite state space $\cX$. We fix integer $\tau$ and $x\geq 0$ to be specified later, and let $\pi_n$ be the distribution of $X_n$.
	\cref{thm:Bernstein-markov} yields
	$$
	\condP{\left|\sum_{i=\tau+1}^n f(X_i)\right| \geq x\ }{X_{1}\sim \pi}
	\leq 2 \exp \left(-\frac{x^{2}}{16t_{\mix}\left(n+2t_{\mix}-\tau\right) \sigma^2+40 x t_{\mix} M}\right).
	$$
	Let $\cB_{\tau}$ be the event $\left\{\abs{\sum_{i=\tau+1}^n f(X_i)} \geq x\right\}$, then
	\begin{align*}
		&\abs{\condP{\cB_{\tau}}{X_1\sim\pi}-\condP{\cB_{\tau}}{X_1\sim\pi_1}}\\
		=&\abs{\sum_{x\in\cX} \condP{\cB_{\tau}}{X_{\tau+1}=x}\left(\condP{X_{\tau+1}=x}{X_1\sim \pi}-\condP{X_{\tau+1}=x}{X_1\sim \pi_1}\right)}\\
		=&\abs{\sum_{x\in\cX} \condP{\cB_{\tau}}{X_{\tau+1}=x}\left(\pi(x)-\pi_{\tau+1}(x)\right)}\\
		\leq& \max\left(\nrm{\pos{\pi-\pi_{\tau+1}}}_1, \nrm{\neg{\pi-\pi_{\tau+1}}}_1\right)\\
		=& d_{\TV}(\pi,\pi_{\tau+1})\leq \Eps(\tau).
	\end{align*}
	Therefore, we can take $x=\sqrt{32t_{\mix}(n-\tau+2t_{\mix})\log\frac{4}{\delta}}+80t_{\mix}M\log\frac{4}{\delta}$ and $\tau=\ceil{\log_2 \frac{2}{\delta}}t_{\mix}$, then
	\[\begin{aligned}
		\condP{\cB_{\tau}}{X_1\sim\pi_1}
		\leq \Eps(\tau)+\condP{\cB_{\tau}}{X_1\sim\pi}\leq \frac{\delta}{2}+\frac{\delta}{2}=\delta.
	\end{aligned}\]
	Hence \highprobdemo, it holds that
	\[\begin{aligned}
		\abs{\sum_{i=\tau+1}^n f(X_i)}\leq \sqrt{32t_{\mix}(n-\tau+2t_{\mix})\log\frac{4}{\delta}}+80t_{\mix}M\log\frac{4}{\delta}.
	\end{aligned}\]
	The proof is completed by noticing that $\abs{\sum_{i=1}^\tau f(X_i)}\leq \tau M\leq 2t_{\mix}M\log\frac{4}{\delta}$ and $\tau\geq 2t_{\mix}$.
\end{proof}

\subsection{Proof of Proposition \ref{prop:markov-diff-grad}  }
\begin{proof} 
Recall that the gradient estimators are constructed as
\[\begin{aligned}
  \recall{eqn:def-gradient}.
\end{aligned}\]
Therefore, for $Z=[V;\lambda;x]$ that is $\cF_t$ measurable, we have
\begin{equation*}
\begin{aligned}
  &\cond{\hg_x(Z;\zeta_{t+\tau})}{\cF_t}\\
  =&\cond{\frac{r_{t+\tau}+\gamma V(s_{t+\tau})-V(s_{t+\tau+1})+\iprod{\bu_{t+\tau}^{\kappa}}{\lambda}}{\hat\mu(s_{t+\tau},a_{t+\tau})}\II_{s_{t+\tau},a_{t+\tau}}}{s_t,Z}\\
  =&\cond{\frac{r(s_{t+\tau},a_{t+\tau})+\gamma V(s_{t+\tau})-V(s_{t+\tau+1})+\iprod{\bu^{\kappa}(s_{t+\tau},a_{t+\tau})}{\lambda}}{\hat\mu(s_{t+\tau},a_{t+\tau})}\II_{s_{t+\tau},a_{t+\tau}}}{s_t,Z}\\
  =&\sum_{s,a,s'} \condP{s_{t+\tau}=s,a_{t+\tau}=a,s_{t+\tau+1}=s'}{s_t} \frac{r(s,a)+\gamma V(s)-V(s')+\iprod{\bu^{\kappa}(s,a)}{\lambda}}{\hat\mu(s,a)}\II_{s,a}\\
  =&\sum_{s,a} \frac{\condP{s_{t+\tau}=s,a_{t+\tau}=a}{s_t}}{\hat\mu(s,a)} \left(r(s,a)+\gamma V(s)-\EE_{s'|s,a}\left[V(s')\right]+\iprod{\bu^{\kappa}(s,a)}{\lambda}\right)\II_{s,a}.\\
\end{aligned}
\end{equation*}
For the sake of simplicity, we denote
\begin{equation}
\begin{aligned}
  W^{\tau,s_t}:=\diag\left(\frac{\condPb{s_{t+\tau}=s,a_{t+\tau}=a}{s_t}}{\hmu(s,a)}\right)
  =\diag\left(\frac{\condPb{s_{t+\tau}=s}{s_t}\pib(a|s)}{\hmu(s,a)}\right)_{s,a},
\end{aligned}
\end{equation} 
and we follow the matrix notation introduced in \cref{appdx:duality-notation}. Then
\begin{align*}
\cond{\hg_x(Z;\zeta_{t+\tau})}{\cF_t}&=W^{\tau,s_t}(r-AV+U_{\kappa}^\T \lambda),\\
\cond{\hg_V(Z;\zeta_{t+\tau})}{\cF_{t}}
&=\cond{\II_{s_0}+\frac{x(s_{t+\tau},a_{t+\tau})}{\hmu(s_{t+\tau},a_{t+\tau})}\left(\gamma\II_{s_{t+\tau+1}}-\II_{s_{t+\tau}}\right)}{s_t,Z}\\
&=\rho_0+\sum_{s,a} \condPb{s_{t+\tau}=s,a_{t+\tau}=a}{s_t}\frac{x(s,a)}{\hmu(s,a)}\left(\gamma\EE_{s'|s,a}\left[\II_{s'}\right]-\II_{s}\right)\\
&=\rho_0- A^\T W^{\tau,s_t} x,\\
\cond{\hg_\lambda(Z;\zeta_{t+\tau})}{\cF_t}
&=\cond{\frac{x(s_{t+\tau},a_{t+\tau})}{\hmu(s_{t+\tau},a_{t+\tau})}\bu^{\kappa}}{s_t,Z}\\
&=\sum_{s,a} \condPb{s_{t+\tau}=s,a_{t+\tau}=a}{s_t}\frac{x(s,a)}{\hmu(s,a)} \bu^{\kappa}(s,a)\\
&=U_{\kappa} W^{\tau,s_t}x.
\end{align*}

Therefore, we have
\begin{equation*}
\begin{aligned}
  \nrm{\cond{\hg_V(Z;\zeta_{t+\tau})}{\cF_t}-\gLv(Z)}_{1}
  &=\nrm{A^\T \left(W^{\tau,s_t}-W\right)x}_1
  \leq 2\nrm{\left(W^{\tau,s_t}-W\right)x}_1\\
  &\leq 2\sum_{s,a} \abs{\condPb{s_{t+\tau}=s}{s_t}-\mub(s)}\frac{\pib(a|s)x(s,a)}{\hmu(s,a)}\\
  &\leq \frac{2\hphi}{1-\gamma}d_{\mathrm{TV}}\left(\PP_{\pi_b}^{\tau}\left(\cdot|s_t\right), \mub\right)\leq \frac{2\hphi}{1-\gamma}\Eps(\tau),
\end{aligned}
\end{equation*}
where $\PP_{\pi_b}^{\tau}\left(\cdot|s_t\right)$ is the distribution of $s_{t+\tau}$ conditioning on $s_t$, and the last inequality is due to the definition of $\Eps(\cdot)$.
Similarly, it holds that
\[\begin{aligned}
  \nrm{\cond{\hg_\lambda(Z;\zeta_{t+\tau})}{\cF_t}-\gLl(Z)}_{\infty}&\leq \frac{2\hphi}{1-\gamma}\Eps(\tau),\\
  \nrm{\cond{\hg_x(Z;\zeta_{t+\tau})}{\cF_t}-\gLx(Z)}_{\infty}&\leq \frac{64}{\varphi(1-\gamma)\varsigma}\Eps(\tau).
  \end{aligned}\]
Furthermore, for any $Z'=[V';\lambda';x']\in\cZ$, we have
\begin{align*}
&\abs{\iprod{Z'}{\Gr(Z)-\cond{\hg(Z;\zeta_{t+\tau})}{\cF_t}}}\\
\leq& \nrm{V'}_{\infty} \nrm{\cond{\hg_V(Z;\zeta_{t+\tau})}{\cF_t}-\gLv(Z)}_{1}+\nrm{\lambda}_{1} \nrm{\cond{\hg_\lambda(Z;\zeta_{t+\tau})}{\cF_t}-\gLl(Z)}_{\infty}\\
&+\abs{\iprod{r-AV+U_{\kappa}^\T \lambda}{\left(W^{\tau,s_t}-W\right)x}}\\
\leq& \frac{128\hphi}{\varphi(1-\gamma)^2}\Eps(\tau).\qedhere
\end{align*}
\end{proof}

\subsection{Proof of Proposition \ref{prop:markov-bounded-moment}}\label{appdx:markov-pq}
In fact, to prove \cref{prop:markov-bounded-moment}, let us prove a more general result stated as follows.  \cref{prop:markov-bounded-moment} will follow directly from the (2) and (3) of \cref{prop:markov-p-q-basic}. This proposition will also be useful for our later discussion. Recall that we introduce the notation $p(x;s,a):=\frac{x(s,a)}{\hmu(s,a)}$ and $q(x;s,a):=\frac{x(s,a)}{\hmu(s,a)^2}$, and the reloaded notation
 $p(x;\zeta):=p(x;s,a)$ and $q(x;\zeta):=q(x;s,a)$ for sample $\zeta=(s_0,s,a,s',r,\bu)$. Then the following proposition holds true. 
\begin{proposition}\label{prop:markov-p-q-basic}
	\textbf{(1).} For all $x\in\cX$ and $\zeta$, it holds that
	\[\begin{aligned}
		p(x;\zeta)\leq \frac{\hphi}{1-\gamma},\qquad
		q(x;\zeta)\leq \frac{1}{\varsigma}p(x;\zeta)\leq \frac{\hphi}{(1-\gamma)\varsigma}.
	\end{aligned}\]
	
	\textbf{(2).} For all $Z=[V;\lambda;x]$ and $\zeta$, it holds that
	\[\begin{aligned}
		&\nrm{\hg_V(Z;\zeta)}\leq 3p(x;\zeta),\qquad
		\nrm{\hg_\lambda(Z;\zeta)}_{\infty}\leq 2p(x;\zeta),\\
		&\nrm{\hg_x(Z;\zeta)}_{x}\leq \frac{64}{\varphi(1-\gamma)}\sqrt{q(x;\zeta)}.
	\end{aligned}\]
	
	\textbf{(3).} For $x\in\cX$ a (possibly random) vector that is $\cF_t$-measurable, the (asynchronous) moments of $p,q$ can be bounded as
	\begin{align*}
		&\cond{p(x;\zeta_{t+\tau})}{\cF_t}=\sum_{s,a} \frac{\condP{s_{t+\tau}=s,a_{t+\tau}=a}{s_t}}{\hmu(s,a)}x(s,a)\leq C(\tau)\frac{4}{1-\gamma},\\
		& \cond{q(x;\zeta_{t+\tau})}{\cF_t}=\sum_{s,a} \frac{\condP{s_{t+\tau}=s,a_{t+\tau}=a}{s_t}}{\hmu(s,a)}\frac{x(s,a)}{\hmu(s,a)}\leq C(\tau)\frac{\cN\hphi}{1-\gamma},\\
		& \cond{p(x^t;\zeta_{t+\tau})^2}{\cF_t}\leq \frac{\hphi}{1-\gamma}\cond{p(x^t;\zeta_{t+\tau})}{\cF_t}\leq C(\tau)\frac{4\hphi}{(1-\gamma)^2}.
	\end{align*}
\end{proposition} 
Since each step of this proposition can be proved by a direct computation similar to the one in \cref{appdx:grad-var}, we omit the proof for succinctness.

\subsection{Proof of Proposition \ref{prop:empirical-dis-markov}}

Similar to the proof of \cref{prop:empirical-dis}, we consider $\hmu_0(s,a)=\frac{N(s,a)}{N_e}$ and the ``failure event'' 
\[
\Omega:=\bigcup_{s,a}\left\{\left|\mu(s,a)-\hmu_0(s,a)\right|> \sqrt{\mu(s,a)\frac{\ell}{N_e}}+\frac{\ell}{N_e}\right\},
\]
where $\ell= 100t_{\mix}\log\left(\frac{12\nS\nA}{\delta}\right)$. Then by the Bernstein's inequality (\cref{prop:bernstein-markov}), it holds that
\[\begin{aligned}
  \prob{\left|\mu(s,a)-\hmu_0(s,a)\right|> \sqrt{\mu(s,a)\frac{\ell}{N_e}}+\frac{\ell}{N_e}}\leq \frac{\delta}{3\nS\nA}, \quad \forall (s,a)\in\cS\times\cA,
\end{aligned}\]
which further gives $\PP(\Omega)\leq\frac{\delta}{3}$.
The proof is completed by exactly repeating the estimations in the proof of \cref{prop:empirical-dis}, conditioning on $\Omega^c$.

\subsection{Bounding the term $S_2$}\label{appdx:markov-s2}
By separately considering each term in the decomposition
\[\begin{aligned}
\Gamma^t=\Gamma_1^t+\Gamma_{2}^{t-\tau}+\Gamma_{3}^{t-\tau}+\Gamma_4^t,
\end{aligned}\]
the following inequalities hold true. The detailed derivations  are placed at the end of \cref{appdx:markov-s2}. 
\begin{equation}
	\label{prop:markov-part-2-1}
	\sum_{t=1}^{\tau}\Gamma^t+\sum_{t=\tau+1}^T \Gamma_1^t
	\leqsim \frac{\tau\hphi}{\varphi(1-\gamma)^2},
\end{equation}   
\begin{equation}
  \label{prop:markov-martingale} 
  \sum_{t=1}^{T-\tau}\Gamma_3^t \leqsim \frac{1}{\varphi(1-\gamma)^2}\sqrt{T\tau C(\tau)\cN\hphi\clog}\qquad\mbox{\emph{with probability at least }} 1-\frac{\delta}{10},  
\end{equation} 
\begin{equation}
  \label{prop:markov-diff-obj}
  \abs{\Gamma_4^t}\leqsim 
  \frac{1}{\varphi(1\!-\!\gamma)}\frac{\abs{x^t\!-\!x^{t-\tau}}\!(s_t,a_t)}{\hmu(s_t,a_t)}
  \!+\!\Big(\!1\!+\!\frac{x'(s_t,a_t)}{\hmu(s_t,a_t)}\Big)\!\left(\nrm{V^t\!-\!V^{t-\tau}}_{\infty}\!+\nrm{\lambda^t\!-\!\lambda^{t-\tau}}_1\right).
\end{equation}

As of $\Gamma_2^t$, by directly applying \cref{prop:markov-diff-grad} we have $\abs{\Gamma_2^t}\leqsim \frac{\hphi}{\varphi(1-\gamma)^2}\Eps(\tau)$.
Thus, to estimate $S_2$, it remains to bound the sum of quantities $\nrm{V^t-V^{t-\tau}}_{\infty}$, $\nrm{\lambda^t-\lambda^{t-\tau}}_1$ and $\frac{x'(s_t,a_t)}{\hmu(s_t,a_t)}$. For $\nrm{V^t-V^{t-\tau}}_{\infty}$ and $\nrm{\lambda^t-\lambda^{t-\tau}}_1$, as long as $\eta\leq \frac{\alpha_\lambda}{2M_\lambda}$, we have
\[\begin{aligned}
  	\nrm{V^{t+1}-V^{t}}_{\infty}&\leq\nrm{V^{t+1}-V^{t}}\leq \frac{\eta}{\alpha_V} \nrm{\hg_V(Z^t;\zeta_t)},\\
  	\nrm{\lambda^{t+1}-\lambda^{t}}_1&\leq \frac{\eta D_{\lambda,1}}{\alpha_\lambda} \nrm{\hg_\lambda(Z^t;\zeta_t)}_{\infty},
\end{aligned}\]
due to \cref{cor:mirror-seq-diff}.
Therefore, it holds that
\begin{equation*}
  \begin{aligned}
  &\frac{1}{T}\sum_{t=\tau+1}^T\left(1+\frac{x'(s_t,a_t)}{\hmu(s_t,a_t)}\right)\left(\nrm{V^t-V^{t-\tau}}_{\infty}+\nrm{\lambda^t-\lambda^{t-\tau}}_1\right)\\
  \leqsim 
  &\frac{\eta}{T}\sum_{t=\tau+1}^T \left(1+\frac{x'(s_t,a_t)}{\hmu(s_t,a_t)}\right)\left(\frac{\nrm{\hg_V(Z^t;\zeta_t)}}{\alpha_V} + \frac{D_{\lambda,1} \nrm{\hg_\lambda(Z^t;\zeta_t)}_{\infty}}{\alpha_\lambda} \right)\\
  \leqsim 
  &\frac{\eta}{T}\sum_{t=\tau+1}^T \left(\frac{\nrm{\hg_V(Z^t;\zeta_t)}^2}{\alpha_V} + \frac{D_{\lambda,1} \nrm{\hg_\lambda(Z^t;\zeta_t)}_{\infty}^2}{\alpha_\lambda} \right)
  +\frac{\eta}{T}\left(\frac{1}{\alpha_V}+\frac{D_{\lambda,1}}{\alpha_\lambda}\right)\sum_{t=\tau+1}^T\left(1+\frac{x'(s_t,a_t)}{\hmu(s_t,a_t)}\right)^2.
  \end{aligned}
\end{equation*}

Finally, we apply Bernstein's inequality to bound the sequence $\left(\frac{x'(s_t,a_t)^2}{\hmu(s_t,a_t)^2}\right)_{t}$ as follows. Due to
\[\begin{aligned}
  &\frac{x'(s,a)}{\hmu(s,a)}\leq \frac{\hphi}{1-\gamma}, \qquad\EE_{s,a\sim \mu}\left[\frac{x'(s,a)^2}{\hmu(s,a)^2}\right]=\sum_{s,a}\frac{\mu(s,a)}{\hmu(s,a)}\frac{x'(s,a)}{\hmu(s,a)}x'(s,a) \leq \frac{8\hphi}{(1-\gamma)^2},
\end{aligned}\]
and \cref{prop:bernstein-markov}, \highprobs{10}, it holds that
\[\begin{aligned}
  \sum_{t=\tau+1}^T \frac{x'(s_t,a_t)^2}{\hmu(s_t,a_t)^2}\leqsim T\frac{\hphi}{(1-\gamma)^2}+t_{\mix}\frac{\hphi^2}{(1-\gamma)^2}\log\frac{1}{\delta}\leqsim \frac{T\hphi}{(1-\gamma)^2}.
\end{aligned}\]
Combining all the estimations above completes the proof of \eqref{eqn:markov-S-2}.

\subsubsection{Derivation of inequality \eqref{prop:markov-part-2-1}}
  By definition, it holds that
  \[\begin{aligned}
    \sum_{t=1}^{\tau}\Gamma^t+\sum_{t=\tau+1}^T \Gamma_1^t
    =\sum_{t=T-\tau+1}^T \iprod{\Gr(Z^t)}{Z^t-Z'}-\sum_{t=1}^\tau \iprod{\hg(Z^t;\zeta_{t})}{Z^t-Z'}.
  \end{aligned}\]
  For a sample $\zeta=(s_0,s,a,s',r,\bu)$, we denote
  \[\begin{aligned}
  \hcL_{\zeta}(V,\lambda,x):=V(s_0)+\frac{x(s,a)}{\hmu(s,a)}\left(r -V(s)+\gamma V(s')+\iprod{\lambda}{\bu^{\kappa}}\right).
  \end{aligned}\]
  Then, it holds that
  \begin{equation}\label{eqn:markov-linear-obj-diff}
  \begin{aligned}
    \iprod{\hg(Z;\zeta)}{Z-Z'}=\hcL_{\zeta}(V,\lambda,x')-\hcL_{\zeta}(V',\lambda',x).
  \end{aligned}
  \end{equation}
  Hence we have
  \[\begin{aligned}
  \abs{\iprod{\hg(Z^t;\zeta_{t})}{Z^t-Z'}}\leq \abs{\hcL_{\zeta_t}(V^t,\lambda^t,x')}+\abs{\hcL_{\zeta_t}(V',\lambda',x^t)}
  \leq \frac{100\hphi}{\varphi(1-\gamma)^2}.
  \end{aligned}\]
  Similarly, it holds that
  \begin{align*}
    \abs{\iprod{\Gr(Z^t)}{Z^t-Z'}}\leq \abs{\cL_w(V^t,\lambda^t,x')}+\abs{\cL_w(V',\lambda',x^t)}
    \leq \frac{512}{\varphi(1-\gamma)^2},
  \end{align*}
  and we complete the proof by combining the estimations above.



\subsubsection{Derivation of inequality \eqref{prop:markov-martingale}}
As in \cref{appdx:S2}, we consider the sequences
\[\begin{aligned}
  \Delta^t_V&:=\hg_V(Z^t;\zeta_{t+\tau})-\cond{\hg_V(Z^t;\zeta_{t+\tau})}{\cF_t},\\
  \Delta^t_\lambda&:=\hg_\lambda(Z^t;\zeta_{t+\tau})-\cond{\hg_\lambda(Z^t;\zeta_{t+\tau})}{\cF_t},\\
  \Delta^t_x&:=\hg_x(Z^t;\zeta_{t+\tau})-\cond{\hg_x(Z^t;\zeta_{t+\tau})}{\cF_t}.
\end{aligned}\]
They are no longer martingale difference sequences, because $\cond{\Delta^t}{\cF_t}=0$ but $\Delta^t$ is $\cF_{t+\tau+1}$ measurable. Therefore, we invoke the following modified version of Bernstein's inequality.
\begin{lemma}[Modified Bernstein's Inequality]\label{lemma:bernstein-markov-var}
  Assume $\{x_i\}_{i=1}^n$ is a sequence of random vectors in $\RR^{d}$, such that $\cond{x_t}{\cF_t}=0$ and $x_t$ is $\cF_{t+\tau}$ measurable. Assume that $\cond{\|x_t\|^2}{\cF_{t}}\leq \sigma^2$ and $\|x_t\|\leq M$ a.s., then with probability at least $1-\delta$,
	\[
	\nrm{\sum_{i=1}^n x^i}\leq
	2\sigma\sqrt{n\tau\log\left(\frac{(d+1)\tau}{\delta}\right)} + 2M\tau\log\left(\frac{(d+1)\tau}{\delta}\right).
	\]
	When the $\ell_2$ norm is replaced by the $\ell_\infty$ norm, i.e., $\{x_i\}_{i=1}^n$  satisfies $\cond{\|x_t\|_{\infty}^2}{\cF_{t}}\leq \sigma^2$, we have
	\[
	\nrm{\sum_{i=1}^n x^i}_{\infty}\leq
	2\sigma\sqrt{n\tau\log\left(\frac{2d\tau}{\delta}\right)} + 2M\tau\log\left(\frac{2d\tau}{\delta}\right)
	\]
	with probability at least $1-\delta$.
\end{lemma}

We still decompose
\[\begin{aligned}
  \sum_{t=1}^{T-\tau} \Gamma_3^t=&\underbrace{
  \sum_{t=1}^{T-\tau} \left(\iprod{\Delta_V^t}{V'-V^1}+\iprod{\Delta_\lambda^t}{\lambda'-\lambda^1}\right)
}_{S_{c}}\\
&+\underbrace{
  \sum_{t=1}^{T-\tau} \left(\iprod{\Delta_V^t}{V^1-V^t}+\iprod{\Delta_\lambda^t}{\lambda^1-\lambda^t}+\iprod{-\Delta_x^t}{x'-x^t}\right)
}_{S_{m}}.
\end{aligned}\]
Most of the following analysis is similar to the one in \cref{appdx:S2}.

\paragraph{Correlated part} Rewrite
\[\begin{aligned}
  S_{c}
  &= \iprod{\sum_{t=1}^{T-\tau}\Delta_V^t}{V'-V^1}+\iprod{\sum_{t=1}^{T-\tau}\Delta_\lambda^t}{\lambda'-\lambda^1}\\
  &\leq \nrm{V'-V^1}\cdot\nrm{\sum_{t=1}^{T-\tau} \Delta^t_V}
  +\nrm{\lambda'-\lambda^1}_{1}\cdot\nrm{\sum_{t=1}^{T-\tau} \Delta^t_\lambda}_{\infty}.
\end{aligned}\]
For each $t$, by \cref{prop:markov-bounded-moment} (or \cref{prop:markov-p-q-basic}), we have
\begin{equation*}
\begin{aligned}
  &\cond{\sqr{\Delta^t_V}}{\cF_t}\leq \cond{\sqr{\hg_V(Z^t;\zeta_{t+\tau})}}{\cF_t}\leqsim \frac{C(\tau)\hphi}{(1-\gamma)^2},
  &&\nrm{\Delta^t_V}\leqsim \frac{\hphi}{1-\gamma},\\
  &\cond{\sqr{\Delta^t_\lambda}_{\infty}}{\cF_t}
  \leqsim \cond{\sqr{\hg_\lambda(Z^t;\zeta_{t+\tau})}_{\infty}}{\cF_t}
  \leqsim \frac{C(\tau)\hphi}{(1-\gamma)^2},
  &&\nrm{\Delta^t_\lambda}_{\infty}\leqsim \frac{\hphi}{1-\gamma}.
\end{aligned}
\end{equation*}
Thus, we can apply \cref{lemma:bernstein-markov-var} to derive that, \highprobs{20},
\[\begin{aligned}
  &\nrm{\sum_{t=1}^{T-\tau} \Delta^t_V}
  \leqsim \frac{1}{1-\gamma}\sqrt{T\tau C(\tau)\hphi\log(1/\delta)}+\frac{\hphi}{1-\gamma}\cdot\tau\log(1/\delta),\\
  &\nrm{\sum_{t=1}^{T-\tau} \Delta^t_\lambda}_{\infty}
  \leqsim \frac{1}{1-\gamma}\sqrt{T\tau C(\tau)\hphi\log(I/\delta)}+\frac{\hphi}{1-\gamma}\cdot\tau\log(I/\delta).
\end{aligned}\]
Therefore, it holds that \highprobs{20},
\begin{equation}\label{eqn:markov-s3-1}
\begin{aligned}
  S_{c}
  \leqsim \frac{1}{\varphi(1-\gamma)^2}\sqrt{T\tau C(\tau)\nS\hphi\clog}+\frac{\tau\hphi\clog}{1-\gamma}
  \leqsim \frac{1}{\varphi(1-\gamma)^2}\sqrt{T\tau C(\tau)\nS\hphi\clog}.
\end{aligned}
\end{equation}

\paragraph{Martingale part} In order to bound $S_{m}$, we have to consider $\odelta_V^t:=\iprod{\Delta_V^t}{V^1-V^t}$,$\odelta^t_\lambda:=\iprod{\Delta_\lambda^t}{\lambda^1-\lambda^t}$, $\odelta^t_x:=\iprod{\Delta_x^t}{x^t-x'}$. By \cref{prop:markov-bounded-moment}, it holds that
\begin{align*}
  &\abs{\odelta_V^t}\leqsim \frac{\hphi}{\varphi(1-\gamma)^2},
  &&\cond{\left(\odelta_V^t\right)^2}{\cF_t}\leq D_V^2\cond{\sqr{\Delta_V^t}}{\cF_t}\leqsim \frac{C(\tau)\hphi}{\varphi^2(1-\gamma)^4}, \\
  &\abs{\odelta_\lambda^t}\leq \frac{\hphi}{\varphi(1-\gamma)},
  &&\cond{\left(\odelta_\lambda^t\right)^2}{\cF_t}\leq D_{\lambda,1}^2\cond{\sqr{\Delta_\lambda^t}_{\infty}}{\cF_t}\leqsim \frac{C(\tau)\hphi}{\varphi^2(1-\gamma)^2},\\
  &\abs{\odelta_x^t}\leq \frac{\hphi}{\varphi(1-\gamma)^2},
  &&\cond{\left(\odelta_x^t\right)^2}{\cF_t}\leq \cond{\sqr{\frac{x'-x^t}{\sqrt{x'+x^t}}}\sqr{\Delta_x^t}_{x'+x^t}}{\cF_t}
  \leqsim \frac{C(\tau)\cN\hphi}{\varphi^2(1-\gamma)^4}.
\end{align*}
Thus, by applying \cref{lemma:bernstein-markov-var}, the following three estimations hold \highprobs{20}
\begin{align*}
    \sum_{t=1}^{T-\tau} \odelta^t_V
    &\leqsim \frac{1}{\varphi(1-\gamma)^2}\sqrt{T\tau C(\tau)\hphi\log(1/\delta)}+\frac{\hphi}{\varphi(1-\gamma)^2}\cdot\tau\log(1/\delta),\\
    \sum_{t=1}^{T-\tau} \odelta^t_\lambda
    &\leqsim \frac{1}{\varphi(1-\gamma)}\sqrt{T\tau C(\tau)\hphi\log(1/\delta)}+\frac{\hphi}{\varphi(1-\gamma)}\cdot\tau\log(1/\delta),\\
    \sum_{t=1}^{T-\tau} \odelta^t_x
    &\leqsim \frac{1}{\varphi(1-\gamma)^2}\sqrt{T\tau C(\tau)\cN\hphi\log(1/\delta)}+\frac{\hphi}{\varphi(1-\gamma)^2}\cdot\tau\log(1/\delta).
\end{align*}
Therefore,
\begin{equation}\label{eqn:markov-s3-2}
\begin{aligned}
  S_{m}\leqsim \frac{1}{\varphi(1-\gamma)^2}\sqrt{T\tau C(\tau)\cN\hphi\clog}+\frac{\tau\hphi\clog}{\varphi(1-\gamma)^2}\leqsim \frac{1}{\varphi(1-\gamma)^2}\sqrt{T\tau C(\tau)\cN\hphi\clog}.
\end{aligned}
\end{equation}
Combining \eqref{eqn:markov-s3-2} with \eqref{eqn:markov-s3-1} completes the proof.

\begin{proof}[Proof of \cref{lemma:bernstein-markov-var}]
  We reduce \cref{lemma:bernstein-markov-var} to the standard martingale Bernstein's inequality (\cref{lemma:concen-vec}). The set $[n]$ can be decomposed into
  \[\begin{aligned}
  [n]=\bigsqcup_{k=1}^\tau \cI_k,\qquad \cI_k:=\left\{j\in[n]: j\equiv k \mod{\tau}\right\}.
  \end{aligned}\]
  For each $k$, the sequence $\left(X_j\right)_{j\in\cI_k}$ is a martingale difference sequence w.r.t. the filtration $\left(\cF_j\right)_{j\in\cI_k}$. Hence by \cref{lemma:concen-vec}, \highprobs{\tau}, we have
  \[\begin{aligned}
    \nrm{\sum_{j\in\cI_k} x^j}\leq
    2\sigma\sqrt{|\cI_k|\log\left(\frac{(d+1)\tau}{\delta}\right)} + 2M\log\left(\frac{(d+1)\tau}{\delta}\right).
  \end{aligned}\]
  Summing over $k=1,\cdots,\tau$ yields that \highprobdemo
  \begin{align*}
    \nrm{\sum_{j=1}^n x^j}
    &\leq
    2\sigma\sqrt{\log\left(\frac{(d+1)\tau}{\delta}\right)}\sum_{k=1}^{\tau} \sqrt{|\cI_k|} + 2M\tau\log\left(\frac{(d+1)\tau}{\delta}\right)\\
    &\leq 2\sigma\sqrt{n\tau\log\left(\frac{(d+1)\tau}{\delta}\right)} + 2M\tau\log\left(\frac{(d+1)\tau}{\delta}\right),
  \end{align*}
  where the last inequality is due to the Cauchy inequality. 
  
  The analogous $\ell_{\infty}$ case can be done similarly.
\end{proof}

\subsubsection{Derivation of inequality \eqref{prop:markov-diff-obj}}

By \eqref{eqn:markov-linear-obj-diff}, it holds that
\begin{equation}\label{eqn:markov-linear-obj-2}
\begin{aligned}
  \Gamma_4^t
  =&\iprod{\hg(Z^{t-\tau};\zeta_{t})}{Z^{t-\tau}-Z'}-\iprod{\hg(Z^t;\zeta_{t})}{Z^t-Z'}\\
  =&\hcL_{\zeta_t}(V^{t-\tau},\lambda^{t-\tau},x')-\hcL_{\zeta_t}(V',\lambda',x^{t-\tau})
  +\hcL_{\zeta_t}(V',\lambda',x^t)-\hcL_{\zeta_t}(V^t,\lambda^{t},x').
\end{aligned}
\end{equation}
Then we have
\begin{align*}
  &\abs{\hcL_{\zeta_t}(V',\lambda',x^t)-\hcL_{\zeta_t}(V',\lambda',x^{t-\tau})}\\
  =&\frac{\abs{x^t-x^{t-\tau}}(s_t,a_t)}{\hmu(s_t,a_t)}\abs{r_t -V'(s_t)+\gamma V'(s_{t+1})+\iprod{\lambda'}{\bu_t^{\kappa}}}\\
  \leq & \frac{\abs{x^t-x^{t-\tau}}(s_t,a_t)}{\hmu(s_t,a_t)}\left(1+\frac{16}{1-\gamma}\left(1+\frac{2}{\varphi}\right)+\frac{8(1+\kappa)}{\varphi}\right)\\
  \leq & \frac{64}{\varphi(1-\gamma)}\frac{\abs{x^t-x^{t-\tau}}(s_t,a_t)}{\hmu(s_t,a_t)}.
\end{align*}
Similarly,
\begin{align*}
  &\abs{\hcL_{\zeta_t}(V^t,\lambda^{t},x')-\hcL_{\zeta_t}(V^{t-\tau},\lambda^{t-\tau},x')}\\
  \leq& \abs{V^t(s_{0,t})-V^{t-\tau}(s_{0,t})}\\
  &+\frac{x'(s_t,a_t)}{\hmu(s_t,a_t)}\left(\abs{V^t(s_{t})-V^{t-\tau}(s_{t})}+\gamma \abs{V^t(s_{t+1})-V^{t-\tau}(s_{t+1})}+\abs{\iprod{\lambda^t-\lambda^{t-\tau}}{\bu^{\kappa}_t}}\right)\\
  \leq& \nrm{V^t-V^{t-\tau}}_{\infty} \left(1+2\frac{x'(s_t,a_t)}{\hmu(s_t,a_t)}\right)+\nrm{\lambda^t-\lambda^{t-\tau}}_1 \cdot 128\frac{x'(s_t,a_t)}{\hmu(s_t,a_t)}.
\end{align*}
The proof is completed by combining \eqref{eqn:markov-linear-obj-2} with the estimations above. 

\subsection{Proof of Proposition \ref{prop:markov-bound-p-q-tau-demo}}\label{appdx:markov-bound-p-q-tau}

The proof of \cref{prop:markov-bound-p-q-tau-demo} is separated into two steps. 

\textbf{Step 1.} We derive bounds on $\sum p(x^t;\zeta_{t+\tau})^2$ and $\sum q(x^t;\zeta_{t+\tau})$ by directly applying Bernstein's inequality. 

\textbf{Step 2.} We leverage the idea demonstrate in \eqref{eqn:markov-decomp-demo} again to bound $\sum p(x^t;\zeta_t)^2$ and $\sum q(x^t;\zeta_t)$, by bounding their difference with $\sum p(x^t;\zeta_{t+\tau})^2$ and $\sum q(x^t;\zeta_{t+\tau})$ respectively. 

Then we finalize the proof by combining the results of Step 1 and Step 2.

\subsubsection{Step 1. Bounding the asynchronous sums}

First, let us present the following result for the ease of discussion. \\
\textbf{Corollary.}\label{lemma:bernstein-markov-var2}
	\emph{Assume $\{x_i\}_{i=1}^n$ is a sequence of random variables, such that $x_t$ is $\cF_{t+\tau}$ measurable, and $\cond{|x_t|}{\cF_{t}}\leq c$, $|x_t|\leq M$ a.s. Then with probability at least $1-\delta$,}
	\[
	\abs{\frac1n \sum_{i=1}^n x^i}\leq
	2c\tau + 3M\tau\frac{\log\left(2\tau/\delta\right)}{n}.
	\]

By \cref{prop:markov-p-q-basic}, we have
\begin{equation*}\label{eqn:markov-q-tau}
	\begin{aligned}
		& q(x^t;\zeta_{t+\tau})\leq\frac{\hphi}{(1-\gamma)\varsigma},
		& \cond{q(x^t;\zeta_{t+\tau})}{\cF_t}\leq C(\tau)\frac{\cN\hphi}{1-\gamma}.
	\end{aligned}
\end{equation*}
Applying the above corollary yields that \highprobs{20\tau_0},
\begin{equation*}\label{eqn:markov-q-sum-tau}
	\begin{aligned}
		\sum_{t=1}^{T-\tau} q(x^t;\zeta_{t+\tau})
		&\leqsim TC(\tau)\frac{\cN\hphi}{1-\gamma}+\frac{\tau\hphi}{(1-\gamma)\varsigma}\log\left(\frac{\tau_0}{\delta}\right)\\
		&\leqsim TC(\tau)\frac{\cN\hphi}{1-\gamma}+\frac{\tau_0\cN\hphi^2\clog}{\varphi(1-\gamma)^3\epsilon_e}\\
		&\leqsim TC(\tau)\frac{\cN\hphi}{1-\gamma},
	\end{aligned}
\end{equation*}
where the last inequality is due to $T\geqsim \frac{\tau_0^2 \cN\hphi\clog^3}{\varphi^2(1-\gamma)^4\epsilon_e^2}\geq \frac{\tau_0\hphi\clog}{\varphi(1-\gamma)^2\epsilon_e}$.

Similarly, we have
\begin{equation*}\label{eqn:markov-p-tau}
	\begin{aligned}
		& p(x^t;\zeta_{t+\tau})\leq \frac{\hphi}{1-\gamma},
		& \cond{p(x^t;\zeta_{t+\tau})^2}{\cF_t}\leq C(\tau)\frac{4\hphi}{(1-\gamma)^2}.
	\end{aligned}
\end{equation*}
Therefore, for each $1\leq \tau\leq \tau_0$, it holds \highprobs{20\tau_0}
\begin{equation*}\label{eqn:markov-p2-sum-tau}
	\begin{aligned}
		\sum_{t=1}^{T-\tau} p(x^t;\zeta_{t+\tau})^2
		\leqsim \frac{TC(\tau)\hphi}{(1-\gamma)^2} +\frac{\tau\hphi^2}{(1-\gamma)^2}\log\left(\frac{\tau_0}{\delta}\right)
		\leqsim \frac{TC(\tau)\hphi}{(1-\gamma)^2}.
	\end{aligned}
\end{equation*}

By taking the union bound for $1\leq\tau\leq \tau_0$, we conclude that \highprobs{10},
\begin{equation} 
	\label{eqn:markov-bound-p-q-tau}
		\sum_{t=1}^{T-\tau} p(x^t;\zeta_{t+\tau})^2
		\leqsim \frac{TC(\tau)\hphi}{(1-\gamma)^2},\qquad
		\sum_{t=1}^{T-\tau} q(x^t;\zeta_{t+\tau})
		\leqsim \frac{TC(\tau)\cN\hphi}{1-\gamma}, 
\end{equation}
hold simultaneously and uniformly for $1\leq\tau\leq \tau_0$.

\subsubsection{Step 2. Bounding the difference}
Utilizing the closeness between  $Z^t$ and $Z^{t+\tau}$, we bound the difference $q(x^t;\zeta_t)-q(x^t;\zeta_{t+\tau})$ as
\begin{equation}\label{eqn:markov-s1q-to-p}
	\begin{aligned}
		\sum_{t=1}^T q(x^t;\zeta_t)-\sum_{t=1}^{T-\tau} q(x^t;\zeta_{t+\tau})
		&\leq \sum_{t=1}^{\tau} q(x^t;\zeta_{t})+\sum_{t=1}^{T-\tau} q(\abs{x^t-x^{t+\tau}};\zeta_{t+\tau})\\
		&\leq \frac{\tau\hphi}{(1-\gamma)\varsigma} + \frac{1}{\varsigma} \sum_{t=1}^{T-\tau} p(\abs{x^t-x^{t+\tau}};\zeta_{t+\tau}).
	\end{aligned}
\end{equation}

We next deal with the quantity $p(\abs{x^t-x^{t+\tau}};\zeta_{t+\tau})$ carefully. For any $(s,a)\in\cS\times\cA$, it holds that
\begin{align*}
	p(\abs{x^t-x^{t+\tau}};s,a)
	&=\frac{\abs{x^{t}(s,a)-x^{t+\tau}(s,a)}}{\hmu(s,a)}\\
	&\leq \frac{1}{\hmu(s,a)}\sum_{t'=t}^{t+\tau-1} \abs{x^{t'}(s,a)-x^{t'+1}(s,a)}\\
	&\leq \frac{1}{\hmu(s,a)}\sum_{t'=t}^{t+\tau-1} \sqrt{x^{t'}(s,a)+x^{t'+1}(s,a)}\nrm{\frac{x^{t'}-x^{t'+1}}{\sqrt{x^{t'}+x^{t'+1}}}}\\
	&\stackrel{(a)}{\leqsim} \frac{\eta}{\alpha_x}\frac{1}{\hmu(s,a)}\sum_{t'=t}^{t+\tau-1} \sqrt{x^{t'}(s,a)+x^{t'+1}(s,a)}\nrm{\hg_x(Z^{t'};\zeta_{t'})}_{x^{t'}}\\
	&\stackrel{(b)}{\leqsim} \frac{\eta}{\alpha_x}\frac{1}{\hmu(s,a)}\sum_{t'=t}^{t+\tau-1} \sqrt{x^{t'}(s,a)+x^{t'+1}(s,a)}\cdot \frac{1}{\varphi(1-\gamma)}\sqrt{q(x^t;\zeta_t)}\\
	&=\frac{\eta}{\alpha_x}\cdot \frac{1}{\varphi(1-\gamma)} \sum_{t'=t}^{t+\tau-1} \sqrt{q(x^{t'};s,a)+q(x^{t'+1};s,a)}\sqrt{q(x^t;\zeta_t)}\\
	&\stackrel{(c)}{\leq} \frac{\eta}{\alpha_x}\cdot \frac{1}{\varphi(1-\gamma)}
	\sqrt{\sum_{t'=t}^{t+\tau-1} q(x^t;\zeta_t)}
	\sqrt{\sum_{t'=t}^{t+\tau} q(x^{t'};s,a)}.
\end{align*}
Here the inequality (a) is due to \cref{cor:mirror-seq-diff}, the inequality (b) is due to \cref{prop:markov-p-q-basic}, and the inequality (c) comes from Cauchy inequality. Hence, we have
	\begin{align}\label{eqn:markov-p-to-q}
		\sum_{t=1}^{T-\tau} p(\abs{x^t-x^{t+\tau}};\zeta_{t+\tau})
		\leqsim&  \frac{\eta}{\alpha_x}\cdot \frac{1}{\varphi(1-\gamma)}
		\sum_{t=1}^{T-\tau}  \sqrt{\sum_{t'=t}^{t+\tau-1} q(x^t;\zeta_t)}  \sqrt{\sum_{t'=t}^{t+\tau} q(x^{t'};\zeta_{t+\tau})}\notag\\
		\leq& \frac{\eta}{\alpha_x}\cdot \frac{1}{\varphi(1-\gamma)}
		\sqrt{\sum_{t=1}^{T-\tau}\sum_{t'=t}^{t+\tau-1} q(x^t;\zeta_t)}  \sqrt{\sum_{t=1}^{T-\tau}\sum_{t'=t}^{t+\tau} q(x^{t'};\zeta_{t+\tau})}\notag\\
		\leq& \frac{\eta}{\alpha_x}\cdot \frac{1}{\varphi(1-\gamma)}
		\sqrt{\tau \sum_{t=1}^{T}q(x^t;\zeta_t)}  \sqrt{\sum_{j=1}^\tau \sum_{t=1}^{T-j} q(x^{t};\zeta_{t+j})}.
	\end{align}
Combining \eqref{eqn:markov-p-to-q} with \eqref{eqn:markov-s1q-to-p} yields
\begin{equation}\label{eqn:markov-s1}
	\begin{aligned}
		&\sum_{t=1}^{T-\tau} q(x^t;\zeta_{t})- \sum_{t=1}^{T-\tau} q(x^t;\zeta_{t+\tau})\\
		&\leqsim \frac{\tau\hphi}{(1-\gamma)\varsigma}
		+ \frac{\eta}{\alpha_x}\frac{1}{\varphi(1-\gamma)\varsigma}
		\sqrt{\tau\sum_{t=1}^{T} q(x^t;\zeta_t)}
		\sqrt{\sum_{t=1}^{T-\tau}\sum_{t'=t}^{t+\tau} q(x^{t'};\zeta_{t+\tau})}
	\end{aligned}
\end{equation}
Similarly, it holds that for $0\leq j\leq \tau$,
\begin{equation}\label{eqn:markov-pq-2}
	\begin{aligned}
		&\sum_{t=1}^{T-j} q(x^t;\zeta_{t+j})- \sum_{t=1}^{T-\tau} q(x^t;\zeta_{t+\tau})\\
		&\leqsim \frac{\tau\hphi}{(1-\gamma)\varsigma}
		+ \frac{\eta}{\alpha_x}\frac{1}{\varphi(1-\gamma)\varsigma}
		\sqrt{\tau\sum_{t=1}^{T} q(x^t;\zeta_t)}
		\sqrt{\sum_{t=1}^{T-\tau}\sum_{t'=t}^{t+\tau} q(x^{t'};\zeta_{t+\tau})}.
	\end{aligned}
\end{equation}

\subsubsection{Combining Step 1 and Step 2}
Actually, \eqref{eqn:markov-pq-2} is already enough to bound $\sum_{t=1}^{T} q(x^t;\zeta_{t})$. For simplicity, we denote
\begin{align*}
	&Q_1:=\frac{c_0\tau\hphi}{(1-\gamma)\varsigma} + \sum_{t=1}^{T-\tau} q(x^t;\zeta_{t+\tau}), \qquad Q_2:=\sum_{t=1}^{T} q(x^t;\zeta_t),\qquad c:=c_0\frac{\eta}{\alpha_x}\frac{1}{\varphi(1-\gamma)\varsigma},\\
	&Q_3:=\frac{1}{\tau} \sum_{t=1}^{T-\tau}\sum_{t'=t}^{t+\tau} q(x^{t'};\zeta_{t+\tau})=\frac{1}{\tau}\sum_{j=1}^\tau \sum_{t=1}^{T-\tau+j} q(x^t;\zeta_{t+\tau-j}),
\end{align*}
where $c_0$ is a universal constant hidden by the $\leqsim$ in \eqref{eqn:markov-pq-2}. Now, \eqref{eqn:markov-pq-2} implies
\begin{equation}\label{eqn:markov-cycling}
	\begin{aligned}
		& Q_2\leq Q_1+c\tau\sqrt{ Q_2Q_3},\quad Q_3\leq Q_1+c\tau\sqrt{Q_2Q_3},\\
		\Rightarrow & Q_2+Q_3\leq Q_1+c\tau(Q_2+Q_3).
	\end{aligned}
\end{equation}
Thus, as long as $c\tau_0\leq\frac{1}{2}$, we have $Q_2+Q_3\leq 2Q_1$. The condition $c\tau_0\leq\frac12$ is equivalent to
\[\begin{aligned}
	\frac{1}{2c_0}\geq \tau_0\frac{\eta}{\alpha_x}\frac{1}{\varphi(1-\gamma)\varsigma}
	=\tau_0\cdot \sqrt{\frac{1}{T}}\cdot \left(\frac{1}{\varphi(1-\gamma)}\sqrt{\frac{\cN\hphi}{\log\hphi}}\right)^{-1}\frac{1}{\varphi(1-\gamma)\varsigma}
	=\sqrt{\frac{4\tau_0^2\cN\hphi\log\hphi}{\varphi^2(1-\gamma)^4\epsilon_e^2}\cdot\frac{1}{T}}.
\end{aligned}\]
Thus, $T\geq 16c_0^2\frac{\tau_0^2\cN\hphi\log\hphi}{\varphi^2(1-\gamma)^4\epsilon_e^2}$ is enough to ensure $Q_2\leq 2Q_1, Q_3\leq 2Q_1$ for any $\tau\leq \tau_0$. Here, according to \eqref{eqn:markov-bound-p-q-tau} we have
\[\begin{aligned}
	Q_1=\frac{c_0\tau\hphi}{(1-\gamma)\varsigma} + \sum_{t=1}^{T-\tau} q(x^t;\zeta_{t+\tau})
	\leqsim \frac{c_0\tau_0\hphi}{(1-\gamma)\varsigma}+\frac{TC(\tau)\cN\hphi}{1-\gamma}
	\leqsim \frac{TC(\tau)\cN\hphi}{1-\gamma}.
\end{aligned}\]

Consequently, we obtain
\begin{equation}\label{eqn:markov-q}
	\begin{aligned}
		& \sum_{t=1}^{T} q(x^t;\zeta_t)\leqsim \frac{TC(\tau)\cN\hphi}{1-\gamma},\\
		& \sum_{t=1}^{T-\tau}\sum_{t'=t}^{t+\tau} q(x^{t'};\zeta_{t+\tau})
		\leqsim \tau\frac{TC(\tau)\cN\hphi}{1-\gamma}.
	\end{aligned}
\end{equation}
Hence, by \eqref{eqn:markov-p-to-q},
\[\begin{aligned}
	\sum_{t=\tau+1}^{T} p(\abs{x^t-x^{t-\tau}};\zeta_{t})
	\leqsim \frac{\tau C(\tau)}{1-\gamma}\sqrt{T\cN\hphi\log\hphi}.
\end{aligned}\]

We can further establish the bound for $\sum_{t=1}^T p(x^t;\zeta_t)^2$ as
\begin{equation}\label{eqn:markov-p-2}
	\begin{aligned}
		\sum_{t=1}^T p(x^t;\zeta_t)^2
		&\leqsim \sum_{t=1}^{\tau} p(x^t;\zeta_{t})^2
		+ \sum_{t=1}^{T-\tau} \left[p(x^t;\zeta_{t+\tau})^2+p(\abs{x^t-x^{t+\tau}};\zeta_{t+\tau})^2\right]\\
		&\stackrel{(a)}{\leqsim} \tau\frac{\hphi^2}{(1-\gamma)^2} + \sum_{t=1}^{T-\tau} p(x^t;\zeta_{t+\tau})^2
		+ \frac{\hphi}{1-\gamma}\sum_{t=1}^{T-\tau} p(\abs{x^t-x^{t+\tau}};\zeta_{t+\tau})\\
		&\leqsim \frac{\tau\hphi^2}{(1-\gamma)^2} + \frac{T C(\tau)\hphi}{(1-\gamma)^2} + \frac{\tau C(\tau)}{1-\gamma}\sqrt{T\cN\hphi\log\hphi}\\
		&\stackrel{(b)}{\leqsim} \frac{T C(\tau)\hphi}{(1-\gamma)^2},
	\end{aligned}
\end{equation}
where the inequality (a) is due to $p(x^t;\zeta_{t})\leq \frac{\hphi}{1-\gamma}$, $p(\abs{x^t-x^{t+\tau}};\zeta_{t+\tau})\leq \frac{2\hphi}{1-\gamma}$, and the inequality (b) is due to our requirement $T\geqsim \frac{\tau_0^2\cN\hphi\clog}{\varphi^2(1-\gamma)^4\epsilon_e^2}$.

The proof is completed by taking $\tau=\tau_0$ in \eqref{eqn:markov-p-2} and \eqref{eqn:markov-q}.






\end{document}